\documentclass[12pt]{article}

\usepackage{graphicx}
\usepackage{float}

\usepackage[letterpaper,margin=1in]{geometry}
\usepackage{amsmath}
\usepackage{amssymb}
\usepackage{amsthm}

\usepackage{makecell}

\DeclareMathOperator{\argmax}{argmax} 

\usepackage{pifont}
\newcommand{\cmark}{\ding{51}}%
\newcommand{\xmark}{\ding{55}}%

\usepackage{algorithm}
\usepackage{algpseudocode}
\usepackage{tikz}
\usepackage{newtxtext,newtxmath}
\renewenvironment{abstract}
  {\par
   \noindent
   \begin{center}
   {\LARGE\bfseries Abstract}
   \end{center}
   \par\vspace{1em}
   \noindent}
  {\par}

\date{}

\makeatletter
\renewcommand{\fnum@figure}{\textbf{Figure \thefigure}}
\renewcommand{\fnum@table}{\textbf{Table \thetable}}
\makeatother

\usepackage{scicite}

\usepackage{url}

\usepackage{tcolorbox}

\newcommand{\1}{\mathbf{1}}

\newcommand{\E}{\mathbb{E}}
\newcommand{\R}{\mathbb{R}}

\newcommand{\vb}{\mathbf{v}}

\newcommand{\xb}{\mathbf{x}}
\newcommand{\yb}{\mathbf{y}}

\newcommand{\Db}{\mathbf{D}}

\newcommand{\Xb}{\mathbf{X}}
\newcommand{\Yb}{\mathbf{Y}}
\newcommand{\Zb}{\mathbf{Z}}

\newcommand{\mub}{\boldsymbol{\mu}}
\newcommand{\thetab}{\boldsymbol{\theta}}

\newcommand{\Fc}{\mathcal{F}}

\newcommand{\Oc}{\mathcal{O}}
\newcommand{\Pc}{\mathcal{P}}

\newcommand{\Xc}{\mathcal{X}}
\newcommand{\Yc}{\mathcal{Y}}

\newtheorem{theorem}{Theorem}

\newtheorem{assumption}{Assumption}

\newtheorem{corollary}{Corollary}
\newtheorem{definition}{Definition}

\newtheorem{lemma}{Lemma}

\def\scititle{
	Mutual Information Surprise: Rethinking Unexpectedness in Autonomous Systems
}
\title{\bfseries \boldmath \scititle}

\author{
	Yinsong~Wang,
        Quan~Zeng,
        Xiao~Liu,
        Yu~Ding\and
	H. Milton Stewart School of Industrial and Systems Engineering,\and
        Georgia Institute of Technology, Atlanta \& 30332, USA.\and
}

\begin{document} 

\maketitle

\begin{abstract}
A community of researchers appears to think that a machine can be surprised and have introduced various surprise measures, principally the Shannon Surprise and the Bayesian Surprise. The questions of what constitutes a surprise and how to react to one still elicit debates. In this work, we introduce \textit{Mutual Information Surprise} (MIS), a new framework that redefines surprise not as anomaly measure, but as a signal of epistemic growth. Furthermore, we develop a statistical test sequence that could trigger a surprise reaction and propose a MIS-based reaction policy that dynamically governs system behavior through sampling adjustment and process forking. Empirical evaluations—on both synthetic domains and a dynamic pollution map estimation task—show that a system governed by the MIS-based reaction policy significantly outperforms those under classical surprise-based approaches in stability, responsiveness, and predictive accuracy. The important implication of our new proposal is that MIS quantifies the impact of new observations on mutual information, shifts surprise from reactive to reflective, enables reflection on learning progression, and thus offers a path toward self-aware and adaptive autonomous systems. We expect the new surprise measure to play a critical role in further advancing autonomous systems on their ability to learn and adapt in a complex and dynamic environment.
\end{abstract}

\newpage
\tableofcontents
\newpage

\section{Introduction}\label{sec:intro}

In July 2020, \textit{Nature} published a cover story \cite{burger2020mobile} about an autonomous robotic chemist—locked in a lab for a week with no external communication—independently conducting experiments to search for improved photocatalysts for hydrogen production from water. In the years that followed, \textit{Nature} featured three more articles \cite{merchant2023scaling, szymanski2023autonomous, dai2024autonomous} highlighting the transformative role of autonomous systems in materials discovery, experimentation, and even manufacturing, each reporting orders-of-magnitude improvements in efficiency. These reports spotlighted the intensifying global race to advance autonomous technologies beyond the already well-established domain of self-driving cars \cite{levinson2011towards, macleod2020self, yurtsever2020survey, bogdoll2022anomaly}. \textit{Nature} was not alone; numerous other outlets have documented the surge in autonomous research and innovation \cite{park2012autonomous, leng2023towards, reis2021high}. This rapid expansion is a natural consequence of recent advances in robotics and artificial intelligence, which continue to push the boundaries of what autonomous systems can accomplish. 

The systems featured in the \textit{Nature} publications demonstrate highly capable bodies that can perform complex tasks. Recall that an autonomous system comprises two fundamental components: a brain and a body—colloquial terms for its control mechanism and its sensing-action capabilities, respectively. Unlike traditional automation systems, which follow predefined instructions to execute simple, repetitive tasks, true autonomy requires a higher level of cognitive capacity---an autonomous system is supposedly capable of making decisions with minimal human intervention. However, their brain function, while more sophisticated than rigid pre-programmed instructions, remains relatively limited. 

Surveying the literature over the past decade, we found that \cite{burger2020mobile}, \cite{nikolaev2016autonomy}, and \cite{chang2020efficient} rely on classical Bayesian optimization to guide system decisions—a technique that, although effective, does not constitute full autonomy, i.e., completely eliminating human involvement. More recent works in \textit{Nature} \cite{merchant2023scaling, szymanski2023autonomous} continue in a similar vein, adopting active learning frameworks akin to Bayesian optimization, without fundamentally enhancing the cognitive capabilities of these systems. The conceptual limitations of their decision-making mechanisms continue to impede progress toward genuine autonomy. \cite{ahmed2024toward} argue that a core deficiency of current autonomous systems is the absence of a ``surprise” mechanism—the capacity to detect and adapt to unforeseen situations. Without this capability, true autonomy remains out of reach.

 What is a ``surprise,'' and how does it differ from existing measures governing automation? Surprise is a fundamental psychological trigger that enables humans to react to unexpected events. Intuitively, it arises when observations deviate from expectations. Traditionally, unexpectedness has been loosely equated with anomalies—quantifying inconsistencies between new observations and historical data. Common approaches to anomaly detection include statistical methods such as z-scores \cite{zhou2016continuous} and hypothesis testing \cite{cohen2015active, kamenik2023null}; distance-based techniques \cite{weller2014survey}, including Euclidean \cite{montechiesi2016artificial} and Mahalanobis distances \cite{wang2013online, hou2020mahalanobis}; and machine learning-based models \cite{schlegl2019f, lian2022anomaly}, which learn patterns to identify and filter out anomalous data. However, researchers increasingly recognize that simply detecting and discarding unexpected events is insufficient for achieving higher levels of autonomy. In human cognition, unexpectedness is not inherently undesirable; in fact, surprise often signals opportunities for discovery rather than error. Although mathematically similar to anomaly measures, surprise is conceptually distinct: it is not merely a deviation to be rejected, but a valuable learning signal that can enhance adaptation and decision-making.

This shift in perspective aligns with formal definitions of surprise in information theory and computational psychology, such as Shannon surprise \cite{barto2013novelty}, Bayesian surprise \cite{itti2009bayesian}, Bayes Factor surprise \cite{liakoni2021learning}, and Confidence-Corrected surprise \cite{faraji2018balancing}. These surprise definitions quantify unexpectedness by modeling deviations from prior beliefs or probability distributions. Using some of the existing surprise definitions, \cite{ahmed2024toward} demonstrated that treating surprising events not as noise to be removed but as catalysts for learning can significantly enhance a system’s learning speed. Additional empirical evidence shows that incorporating surprise as a learning mechanism can improve autonomy in domains such as autonomous driving \cite{ccatal2020anomaly, zamiri2022bayesian, dinparastdjadid2023measuring} and manufacturing \cite{raihan2024augmented, jin2022autonomous}. In Section 2, we will delve deeper into these existing measures and evaluate whether they truly serve the intended role of identifying opportunities, as human surprise does, more than merely flagging anomalies.

In this paper, we consider a general class of input-output systems described by a functional mapping, $f(\cdot)$, such that
\begin{equation}\label{equ:system}
f: \xb \in \Xc \rightarrow \yb \in \Yc,
\end{equation}
where $\xb$ denotes the inputs and $\yb$ denotes the outputs to the system. The data pair, $\{\xb, \yb\}$, is drawn independently (not necessarily identical) from a potentially non-stationary joint distribution $P(\xb, \yb)$. 

Most existing anomaly and surprise measures implicitly assume that the underlying distribution $P(\xb,\yb)$ remains stationary over time. Such an assumption is restrictive: truly autonomous systems operate in evolving environments, where system dynamics, external conditions, and even objectives may shift. A surprise measure should not merely detect deviations under a fixed model, but rather actively capture potential changes in both inputs and the systems. 

Motivated by the need to address this limitation, we argue that it is essential to develop a new surprise measure that does not anchor itself to the stationarity of $P(\xb,\yb)$, but instead inherently fosters learning and deepens an autonomous system’s understanding of the underlying processes it encounters. To capture this dynamic capability, we introduce the \textbf{Mutual Information Surprise} (MIS)—a new framework that redefines how autonomous systems interpret and respond to unexpected events. MIS quantifies the degree of both \textbf{frustration} and \textbf{enlightenment} associated with new observations, measuring their impact on refining the system’s internal understanding of its environment. We contrast the outcomes made by applying mutual information surprise  with that by relying solely on classical surprise definitions to highlight MIS's potential to enhance learning and decision-making.

The paper is organized as follows. In Section \ref{sec:expectation}, we revisit the concept of surprise by presenting a taxonomy of existing surprise measures and introducing the intuition, mathematical formulation, and limitations of classical definitions. In Section \ref{sec:mis}, we formally define the Mutual Information Surprise (MIS) and derive a testing sequence for detecting multiple types of system changes in autonomous systems. We also design an MIS reaction policy (MISRP) that provides high-level guidance to complement existing exploration-exploitation active learning strategies. In Section \ref{sec:simu}, we compare MIS with classical surprise measures to illustrate its numerical stability and enhanced cognitive capability. We further demonstrate the effectiveness of the MIS reaction policy through a pollution map estimation simulation. In Section 5, we conclude the paper.

\section{Current Surprise Definitions and Their Limitations}\label{sec:expectation}

Classical definitions of surprise, such as Shannon and Bayesian Surprise, provide elegant mathematical frameworks for quantifying unexpectedness. Indeed, the notion of surprise plays a central role in the field of active inference \cite{friston2015active}, where it serves as a key driver of perception and action. However, these approaches often fall short of capturing the mechanisms that underlie adaptive behavior, namely continuous learning and flexible model updating. In this section, we revisit existing formulations of surprise, examining their conceptual foundations while highlighting both their strengths and their limitations.

Before proceeding with our discussion, we introduce the notation used throughout this paper. Scalars are denoted by lowercase letters (e.g., $x$), vectors by bold lowercase letters (e.g., $\xb$), and matrices by bold uppercase letters (e.g., $\Xb$). Distributions in the data space are represented by uppercase letters (e.g., $P$), probabilities by lowercase letters (e.g., $p$), and distributions in the parameter space by the symbol $\pi$. The $L_2$ norm is denoted by $\|\cdot\|_2$, and the absolute value or $L_1$ norm is denoted by $|\cdot|$. We use $\E[\cdot]$ to denote the expectation operator and $\text{sgn}(\cdot)$ for the sign operator. Estimators are denoted with a hat, as in $\hat{\cdot}$.

\subsection*{The Family of Shannon Surprises}

The family of Shannon Surprise metrics emphasizes the improbability of observed data, typically \textit{independent} of explicit model parameters. This class broadly aligns with ``observation'' and ``probabilistic-mismatch'' surprises as categorized in \cite{modirshanechi2022taxonomy}. The central question which the Shannon family of surprises tries to answer is: How unlikely is the observation? For that a \textit{Shannon Surprise} \cite{baldi2002computational} is commonly defined as:
\begin{equation}\label{eq:Shannon}
S_{\text{Shannon}}(\xb) = -\log p(\xb),
\end{equation}
interpreting surprise directly through event rarity. The above expression defines Shannon surprise in terms of the marginal input observation $\xb$.  It is not difficult to see that the same definition is applicable to outputs $\yb$. For a practical system with an input-output structure as expressed in Equation~(\ref{equ:system}), surprise is more naturally defined with respect to the predictive distribution of outputs given inputs, namely in a conditional perspective. This is to say, instead of evaluating $-\log p(\xb)$, one considers 
\begin{equation}\label{eq:Shannon-conditional}
S_{\text{Shannon}}(\xb,\yb) = -\log p(\yb \mid \xb)
\end{equation}
for Shannon Surprise. In this work, we adopt this conditional viewpoint when evaluating the Shannon family of surprise measures.

Although conceptually clear and mathematically elegant, the above definition has a significant limitation: encountering a Shannon Surprise does not inherently imply knowledge acquisition. Consider, for instance, a uniform dartboard—a stochastic yet entirely understood system. Each outcome has an equally low probability, thus appearing ``surprising'' under Shannon's definition, despite humans neither genuinely finding these outcomes surprising nor gaining any additional knowledge by observing them. In other words, the focus of Shannon Surprise is statistical rarity rather than genuine knowledge gain.

To address this limitation, particularly in highly stochastic scenarios, \textit{Residual Information Surprise} \cite{dinparastdjadid2023measuring} has been introduced, which measures surprise by quantifying the gap between the minimally achievable and observed Shannon Surprises: 
\[
S_{\text{Residual}}(\xb,\yb) = \large| \underset{\yb'}{\min}\{-\log p(\yb'|\xb)\} - (- \log p(\yb|\xb)) \large|= \underset{\yb'}{\max}\log p(\yb'|\xb) - \log p(\yb|\xb).
\]
In the dartboard example, Residual Information Surprise becomes zero for all outcomes, as $p(\yb'|\xb)$ remains constant for every outcome $\yb'$, accurately reflecting an absence of genuine surprise. However, this formulation introduces a conceptual challenge, as determining $\underset{\yb'}{\max}\log p(\yb'|\xb)$ implicitly presumes an omniscient oracle, an assumption typically infeasible in practice.

Interestingly, Shannon Surprise serves as a foundation for various anomaly measures. For example, under Gaussian assumptions, Shannon Surprise becomes proportional to squared error:
\[
S_{\text{Shannon}}(\xb, \yb) \propto \|\yb - \mub_{\yb \mid \xb}\|_2^2,
\]
thus linking surprise with deviation from the mean. Similarly, assuming a Laplace distribution, Shannon surprise recovers an absolute error interpretation, termed \textit{Absolute Error Surprise} in \cite{prat2021human}:
\[
S_{\text{Shannon}}(\xb, \yb) \propto |\yb - \mub_{\yb \mid \xb}|.
\]
Both Squared Error Surprise and Absolute Error Surprise are commonly utilized metrics in anomaly detection literature \cite{rousseeuw1993alternatives, aytekin2018clustering, nguyen2019anomaly}.

\subsection*{The Family of Bayesian Surprises}

\textit{Bayesian Surprises}, by contrast, explicitly model belief updates. Rather than assessing the improbability of an observation, the Bayesian family assumes that the data-generating process is governed by an underlying parameter $\thetab$, and defines surprise in terms of how a new observation changes the posterior belief over $\thetab$. These measures quantify the degree to which a new observation alters the internal model, shifting the focus from event rarity to epistemic impact. This concept parallels the ``belief-mismatch'' surprise in the taxonomy by \cite{modirshanechi2022taxonomy}.

The canonical formulation, introduced in \cite{baldi2002computational}, defines Bayesian Surprise as the Kullback-Leibler divergence between the prior and posterior distributions over parameters:
\begin{equation}\label{equ: Bayesian}
S_{\text{Bayes}}(\xb) = D_{\text{KL}}\left(\pi(\thetab \mid \xb) \,\|\, \pi(\thetab)\right).
\end{equation}
The above definition is presented in terms of $\xb$. Like in the case of Shannon surprise, the Bayesian family of surprise measures naturally extend to input-output systems as well. This extension follows from recognizing that the parameter $\thetab$ characterizes the functional relationship between the input $\xb$ and the output $\yb$. As such, observations arrive as paired data $(\xb, \yb)$, and belief updates occur over the parameters governing the conditional mapping $P(\yb \mid \xb, \thetab)$. Consequently, the Bayesian Surprise for input-output systems can be written as
\begin{equation}\label{equ: Bayesian-2}
S_{\text{Bayes}}(\xb,\yb) 
= D_{\text{KL}}\left(\pi(\thetab \mid \xb, \yb) \,\|\, \pi(\thetab)\right),
\end{equation}
where the posterior $\pi(\thetab \mid \xb, \yb)$ reflects the updated belief after observing the input-output pair.

The Bayesian Surprise measure offers a principled approach to belief revision and naturally aligns with learning mechanisms. In theory, it encourages agents to reduce surprise through model updates, providing a pathway toward adaptive autonomy.

However, Bayesian Surprise is not without limitations. As data accumulates, new observations exert diminishing influence on the posterior, rendering the agent increasingly ``stubborn.'' This behavior can result in Bayesian Surprise overlooking rare but meaningful anomalies. For example, consider the discovery by S.~S.~Ting of the $J$ particle, characterized by an unusually long lifespan compared to other particles in its class. Under standard Bayesian updating, scientists' beliefs about particle lifespans would barely shift due to this single observation. Consequently, Bayesian Surprise would classify such an event as merely an anomaly, potentially disregarding it. 

To mitigate this posterior overconfidence, \textit{Confidence-Corrected (CC) Surprise} \cite{faraji2018balancing} compares the current informed belief against that of a naïve learner with a flat prior:
\[
S_{\text{CC}}(\xb,\yb) = D_{\text{KL}}\left(\pi(\thetab) \,\|\, \pi'(\thetab \mid \xb, \yb)\right),
\]
where $\pi'(\thetab \mid \xb, \yb)$ represents the updated belief assuming a uniform prior. This confidence-corrected formulation remains sensitive to new data irrespective of prior history. In the $J$ particle example, employing Confidence-Corrected Surprise would trigger a genuine surprise, as the posterior remains responsive to the novel observation without the inertia introduced by extensive historical data.

A related idea emerges with \textit{Bayes Factor (BF) Surprise} \cite{liakoni2021learning}, which compares likelihoods under naïve and informed beliefs:
\[
S_{\text{BF}}(\xb,\yb) = \frac{p(\xb,\yb \mid \pi^0(\thetab))}{p(\xb,\yb \mid \pi^t(\thetab))},
\]
where $\pi^0(\thetab)$ represents the naïve (untrained) prior and $\pi^t(\thetab)$ the informed belief based on all prior observations up to time $t$ (before observing $(\xb,\yb)$). This ratio quantifies how strongly the current observation supports the naïve prior over the informed prior. In practice, the effectiveness of both Confidence-Corrected and Bayes Factor Surprises heavily depends on constructing appropriate priors—a task often challenging and subjective.

Another variant within the Bayesian Surprise family is \textit{Postdictive Surprise} \cite{kolossa2015computational}, which operates in the output space rather than parameter space as in the original Bayesian Surprise:
\begin{equation}\label{eq:postdictive}
S_{\text{Postdictive}}(\xb, \yb) = D_{\text{KL}}\left(P(\yb \mid \thetab', \xb) \,\|\, P(\yb \mid \thetab, \xb)\right),
\end{equation}
where $\thetab$ and $\thetab'$ denote the predictive model parameters before and after the update, respectively. \cite{kolossa2015computational} argue that computing KL divergence in the output space is more computationally tractable for variational models but potentially less expressive when output variance depends on the input (e.g., under heteroskedastic conditions).

\subsection*{Reflection}

We acknowledge the presence of alternative categorizations of surprise definitions, notably the taxonomy in \cite{modirshanechi2022taxonomy}, which classifies surprise measures into three groups: observation surprises, probabilistic-mismatch surprises, and belief-mismatch surprises. As discussed previously, the Shannon Surprise family aligns closely with the first two categories, whereas the Bayesian Surprise family corresponds to the last.

These categorizations are not strictly delineated. For instance, Residual Information Surprise incorporates a conceptual element common to the Bayesian Surprise family—providing a baseline against which the observed data is contrasted with. On the other hand, Bayes Factor Surprise, despite being explicitly Bayesian in its formulation, closely resembles a Shannon Surprise conditioned on alternative priors. Notwithstanding their philosophical distinctions, Bayesian and Shannon Surprises often behave similarly in practice; we provide further details on this observation in Section~\ref{sec:simu}.

It is understandable that researchers initially explored these two foundational surprise definitions, each possessing inherent limitations: Shannon Surprise conflates probability with knowledge gain, while Bayesian Surprise suffers from increasing posterior stubbornness. Subsequent refinements emerged to address these shortcomings, primarily through adjusting the choice of prior to create more meaningful contrasts. The Residual Information Surprise assumes an oracle-like prior, whereas Confidence-Corrected and Bayes Factor Surprises rely on a non-informative prior. Regardless of the priors chosen, defining a suitable prior remains a challenging and unresolved issue in the research community.

Both surprise families share other critical limitations: they are \textit{single-instance measures} by design and are inherently \textit{one-sided measures}. Being single instance means that they primarily focus on surprise based solely on the marginal impact of individual observations, without explicitly modeling cumulative learning dynamics over time, whereas being one-sided means that they have a decision threshold on one single side, offering limited expressiveness since human perceptions of surprise can range from positive to negative.

\section{Mutual Information Surprise}\label{sec:mis}

In this section, we introduce the concept of \emph{Mutual Information Surprise} (MIS). We first explore the intuition and motivation underlying this concept, followed by the development of a novel, theoretically grounded testing sequence. We then discuss the implications when this test sequence is violated and propose a reaction policy contingent on different types of violations. 

\subsection{An Alternative Surprise Definition}\label{subsec:expect}
Recall that we consider input-output systems described by a functional mapping $f(\cdot): \xb \rightarrow \yb$, with observations drawn from a joint distribution $P(\xb,\yb)$. Classical surprise measures are implicitly constructed under the assumption that $P(\xb,\yb)$ remains stationary over time. However, autonomous engineering systems often operate in evolving environments where such stationarity assumptions may not hold.  What this means is that not all changes in $P(\xb,\yb)$ are of interest or should trigger a surprise reaction.  This prompts us to look for a new tool to define surprise for the input-output system we are dealing with.

Our research leads us to the concept of \textit{mutual information} (MI), which was introduced by \cite{shannon1948mathematical} (although the specific term, ``mutual information,'' was coined by Robert Fano nearly ten year later) and defined as
\begin{equation}\label{eq:mi}
    I(\xb;\yb) = \E_{\xb, \yb} \left[ \log \frac{p(\yb \mid \xb)}{p(\yb)} \right] = H(\xb) + H(\yb) - H(\xb,\yb) = H(\yb) - H(\yb \mid \xb),
\end{equation}
where $H(\cdot)$ denotes entropy, which measures the uncertainty of a random variable. The mutual information, $I(\xb;\yb)$, which is always non-negative, quantifies the reduction in uncertainty about $\yb$ when $\xb$ is observed. In fact, mutual information can serve as a quantitative tool to benchmark ``system learning'' because Fano's inequality \cite{verdu1994generalizing} establishes that independently of prediction models, the best possible learning performance is governed by mutual information. Specifically, an increasing $I(\xb;\yb)$ indicates that the system is gaining in understanding of the underlying functional relationship $f(\cdot)$, whereas stagnation or decline in $I(\xb;\yb)$ suggests that the learning has stalled. Importantly, not all changes in $P(\xb,\yb)$ will cause a change in mutual information, a property that we look for to accommodate dynamic autonomous systems. 

Unlike the traditional change detection literature which usually defines the null hypothesis (the baseline) as a stationary in-control distribution, we define our null hypothesis as an input-output system with a stationary mutual information. We formalize this null hypothesis through the following notion of a well-regulated autonomous system:

\begin{definition}\label{def:regulated}
A well-regulated input-output autonomous system may exhibit non-stationary joint distributions $P(\xb,\yb)$ over time, yet its mutual information $I(\xb;\yb)$ remains unchanged.
\end{definition}

For an autonomous system to trigger a surprise, it is when the mutual information of the system changes, meaning that the system is out of the well-regulated region it has been operating under. In practice, the mutual information $I(\xb;\yb)$ is typically estimated using maximum likelihood estimation (MLE) \cite{paninski2003estimation}; some details regarding the estimator used in this work are provided in the Appendix. The estimation introduces uncertainty, so that statistical decision thresholds need to be established as a function of the prescribed significant level; we will provide the upper/lower statistical decision thresholds in Section~\ref{subsec:bound}.

The above arguments lead us to introduce a \textbf{Mutual Information Surprise (MIS)}, defined as the change in estimated mutual information after incorporating new observations:
\begin{equation}\label{MIS}
\text{MIS} \triangleq \hat{I}_{n+m} - \hat{I}_n,
\end{equation}
where $\hat{I}_n$ denotes the estimate of mutual information based on the first $n$ observations, and $\hat{I}_{n+m}$ denotes the estimate after observing an additional $m$ samples, with $\xb$ and $\yb$ omitted for simplicity. In this work, we adopt the standard estimation approach of histogram-based distribution representation \cite{paninski2003estimation} and a default of $10$ bins unless otherwise specified. 

A large positive MIS indicates \textit{enlightenment} (increase in $I(\xb;\yb)$), meaning that the new observations significantly improve the system understanding of the input-output relationship. In contrast, a near-zero or negative MIS indicates \textit{frustration} (decrease in $I(\xb;\yb)$), suggesting that new data contributes little to learning or even disrupts previously acquired structure. In this way, MIS provides a practical signal for assessing whether an autonomous system continues to acquire knowledge from its environment. We argue that MIS is a better metric for guiding decision making of autonomous systems and we will provide further support to sustain our claim in the subsequent sections.

\subsection{Bounding MIS}\label{subsec:bound}

Testing the change in $I(\xb;\yb)$ via MIS is challenging: mutual information estimation is nonlinear and exhibits complex variance. The standard method, though principled, is a computationally expensive permutation test \cite{franccois2006permutation, doquire2013mutual}, involving repeatedly shuffling $m + n$ observations into two groups, calculating MI differences, and evaluating rejection probabilities:
\begin{equation*}
    p = \frac{1}{B} \sum_{i=1}^B \1(|\Delta \hat{I}| > |\Delta \hat{I}|_i),
\end{equation*}
where $\Delta \hat{I} = \hat{I}_n - \hat{I}_m$ represents the actual differences between mutual information estimations, and $\Delta \hat{I}_i$ represents the $i$th permuted difference. $\mathbf{1}(\cdot)$ is the indicator function. In real-time streaming scenarios, however, permutation tests become impractical due to their computational load. Moreover, when $m \ll n$, permutation tests lose effectiveness, yielding noisy outcomes.

An alternative is standard deviation-based testing. For MLE mutual information estimator $\hat{I}_n$, its estimation standard deviation satisfies \cite{paninski2003estimation}:
\begin{equation}\label{eq:I_std}
    \sigma \lesssim \frac{\log n}{\sqrt{n}},
\end{equation}
where $\lesssim$ stands for less or equal to (in terms of order), which yields an analytical test on the mutual information change when omitting the bias term (brief derivation provided in the Appendix),
\begin{equation}\label{eq:vartest}
    \hat{I}_{m+n} - \hat{I}_n \in \pm \sqrt{\frac{\log^2 (m+n)}{m+n} + \frac{\log^2 n}{n}} \cdot z_\alpha \asymp \Oc\left(\frac{\log n}{\sqrt{n}}\right),
\end{equation}
where $z_{\alpha}$ represents the standard normal random variable at confidence level $\alpha$ and $\asymp$ represents equal in order. But this test too is unsatisfying, because the above bound is so loose that it rarely gets violated. The root cause is the loose upper bound shown in Eq. \eqref{eq:I_std}, where empirical evidence suggests the true estimation standard deviation is usually much smaller than the theoretical bound. We provide the empirical evidence in the Appendix.

So, we turn to a new path for bounding MIS as follows. First, we impose several mild assumptions on the observations and the physical process.
\begin{assumption}\label{assump:MIS}
We impose the following assumptions on the sampling process and physical system.

    \begin{enumerate}
        \item We assume that the existing observations are typical in the sense of the Asymptotic Equipartition Property \cite{cover1999elements}, meaning that empirical statistics computed from the data are representative of their corresponding expected values under the experimental design’s intended distribution, i.e., $\hat{I}_n \approx \E[\hat{I}_n]$. This is true when we regard the initial observations as true system information.
        \item The number of existing observations $n$ is much smaller than the cardinality of space $\Xc, \Yc$. $n \ll |\Xc|, |\Yc|$
        \item The number of new observations $m$ is much smaller than the number of existing observations. $m \ll n$.
    \end{enumerate}
\end{assumption}

\begin{theorem}\label{the:mis}
    Consider a well-regulated autonomous system defined in Definition \ref{def:regulated}, which satisfies the conditions in Assumption \ref{assump:MIS}. With probability at least $1 - \rho$, the change in MLE-based mutual information estimates satisfies:
    \[
        \hat{I}_{n+m} - \hat{I}_{n} \in \left(\log (m+n)-\log n \right) \pm \frac{\sqrt{2m \log \frac{2}{\rho}} \log(m + n)}{m + n} \triangleq MIS_{\pm}.
    \]
    $MIS_{\pm}$ denotes the upper and lower bound for the test sequence.
\end{theorem}

The proof of Theorem \ref{the:mis} is shown in the Appendix. These bounds are both tighter ($\Oc(\frac{\log n}{n})$ instead of $\Oc(\frac{\log n}{\sqrt{n}})$) and more efficient (analytical test sequence) than previous methods. The bounds offer theoretically grounded thresholds within which we expect MI to evolve. When these bounds $MIS_{\pm}$ are breached—either from below or from above—we then deem that the true mutual information of the system has changed.


\subsection{What Does MIS Actually Tell Us?}\label{subsec:implications}

When the quantity $\text{MIS} = \hat{I}_{n+m} - \hat{I}_n$ falls outside the established bounds $MIS_{\pm}$—either exceeding the upper bound or falling below the lower bound—the system is considered to be surprised, thereby triggering a Mutual Information Surprise (MIS). Essentially, Theorem \ref{the:mis} functions as a statistical hypothesis test: the null hypothesis posits that the underlying system remains well-regulated, implying $\Delta I = I_{n+m} - I_n = 0$, where $I_n$ denotes the true mutual information at the time of $n$ observations. Any violation indicates a significant shift, with negative deviations ($\Delta I < 0$) and positive deviations ($\Delta I > 0$) each carrying distinct implications on learning performance via the Fano Lemma \cite{verdu1994generalizing}.

Recall that mutual information can be expressed in terms of entropy, as shown in Eq.~\eqref{eq:mi}, so changes in $\Delta I$ may result from variations in $H(\xb)$, $H(\yb)$, and $H(\yb \mid \xb)$. In this subsection, we examine the implications of MIS under different driving forces.

\subsubsection*{Violation from Below: Learning Has Stalled or Regressed}

If
\[
\text{MIS} < \text{MIS}_-,
\]
this implies $\Delta I(\xb;\yb) < 0$, signifying a downward shift in mutual information. A negative surprise indicates diminished or stalled learning, potentially due to:

\begin{enumerate}
    \item \textbf{Stagnation in Exploration:} A downward shift driven by a decrease in input entropy $\Delta H(\xb)<0$ suggests the system repeatedly samples in a limited region, thus gathering redundant data with minimal new information.

    \item \textbf{Increased Noise or Process Drift:} A downward shift could also result from increased conditional entropy $\Delta H(\yb \mid \xb)>0$, indicating greater uncertainty in predicting $\yb$ given $\xb$. Practically, this often signifies increased external noise or a fundamental change in the underlying process.
\end{enumerate}

\subsubsection*{Violation from Above: Sudden Growth in Understanding}

If
\[
\text{MIS} > \text{MIS}_+,
\]
this implies $\Delta I(\xb;\yb) > 0$, indicating an upward shift in mutual information. This positive surprise can result from:

\begin{enumerate}
    \item \textbf{Aggressive Exploration:} If the increase is driven by higher input entropy $\Delta H(\xb)>0$, the system is likely exploring previously unvisited regions aggressively, potentially inflating knowledge gains without sufficient validation.

    \item \textbf{Reduction in Noise:} An increase due to reduced conditional entropy $\Delta H(\yb \mid \xb)<0$ signals a desirable decrease in uncertainty, thus generally representing a beneficial development.

    \item \textbf{Novel Discovery:} An increase in output entropy $\Delta H(\yb)>0$ suggests discovery of novel and previously rare outputs—particularly valuable in exploratory or scientific contexts.
\end{enumerate}


\begin{table}
    \centering
    \caption{Summary Table: MIS Violations and Their Potential Causes}
    \label{tab:summary}
    \begin{tabular}{|l|l|c|}
\hline
\textbf{Violation Type} & \textbf{Possible Causes} & \textbf{Trend in Mutual Information} \\
\hline
Violation from Below & Stagnation in exploration & $\downarrow H(\xb) \Rightarrow \downarrow I(\xb;\yb)$ \\
& Increased noise / process drift & $\uparrow H(\yb \mid \xb) \Rightarrow \downarrow I(\xb;\yb)$ \\
\hline
Violation from Above & Aggressive exploration & $\uparrow H(\xb) \Rightarrow \uparrow I(\xb;\yb)$ \\
& Noise reduction & $\downarrow H(\yb \mid \xb) \Rightarrow \uparrow I(\xb;\yb)$ \\
& Novel discovery & $\uparrow H(\yb) \Rightarrow \uparrow I(\xb;\yb)$ \\
\hline
\end{tabular}
\end{table}


\vspace{9 pt}

Table \ref{tab:summary} summarizes potential causes for MIS violations and their implications. These patterns help the system differentiate between meaningful learning and misleading deviations, expanding beyond the capacity of classical surprise measures and providing a road map to corrective or adaptive responses for higher level autonomy. We purposely omit the case where a decrease in $H(\yb)$ causes violation from below, as this scenario typically lacks independent significance.  Instead, its happening is generally caused by changes in sampling strategy or underlying processes, which we have already discussed.

\subsection{Reaction Policy: A Three-Pronged Approach}\label{subsec:react}

Following the identification of potential causes behind MIS triggers (Section~\ref{subsec:implications}), the next question is how the system should respond. Naturally, the system’s reaction should align with the dominant entropy component contributing to the change. In practice, we identify the dominant entropy change by computing and ranking the ratios
\[
\frac{\text{sgn}(\text{MIS}) \Delta \hat{H}(\xb)}{|\text{MIS}|}, \quad
\frac{\text{sgn}(\text{MIS}) \Delta \hat{H}(\yb)}{|\text{MIS}|}, \quad \text{and} \quad
\frac{\text{sgn}(\text{MIS}) \Delta \hat{H}(\yb \mid \xb)}{|\text{MIS}|},
\]
where $\Delta \hat{H}(\cdot) = \hat{H}_{m+n}(\cdot) - \hat{H}_n(\cdot)$ denotes the estimated entropy change. 

We do not prescribe a specific reaction when $\Delta \hat{H}(\yb)$ dominates the MIS, as an increase in $H(\yb)$ is typically a passive consequence of changes in $H(\xb)$ and $H(\yb \mid \xb)$. When both $H(\xb)$ and $H(\yb \mid \xb)$ remain relatively stable, a rise in $H(\yb)$ indicates that the current sampling strategy is effectively uncovering novel information; thus, no change in action is required.

For $\Delta \hat{H}(\xb)$ and $\Delta \hat{H}(\yb \mid \xb)$, situations may arise where their contributions are similar, i.e., no clear dominant entropy component exists and we need a resolution mechanism to break the tie. To address all these scenarios, we propose a three-pronged reaction policy that serves as a supervisory layer, compatible with existing exploration–exploitation sampling strategies:

\vspace{3mm}
\noindent
\textbf{1. Sampling Adjustment.}  
The first policy addresses variations in input entropy $H(\xb)$. If $\Delta \hat{H}(\xb) > 0$ dominates MIS, indicating overly aggressive exploration, the system should moderate exploration and emphasize exploitation to prevent fitting to noise. Conversely, if $\Delta \hat{H}(\xb) < 0$, suggesting redundant sampling, the system should enhance exploration to restore sample diversity.

\vspace{2mm}
\noindent
\textbf{2. Process Forking.}  
The second policy responds to variations in conditional entropy $H(\yb \mid \xb)$, i.e., changes in conditional distribution. Upon surprise triggered by $\Delta \hat{H}(\yb \mid \xb)$, the system forks into two subprocesses, each consisting of $n$ existing observations and $m$ new observations divided at the surprise moment (Theorem \ref{the:mis}). The two subprocesses represent the prior process (existing observations) and the likely altered process (new observations), and will continue their sampling separately. The subprocess first encountering a $\Delta \hat{H}(\yb \mid \xb)$-triggered surprise is discarded, and the remaining subprocess continues as the main process. In the extremely rare case when both subprocesses trigger a $\Delta \hat{H}(\yb \mid \xb)$ dominated MIS surprise at the same time, we discard the process with fewer observations, and continues with the subprocess with more observations.

\vspace{2mm}
\noindent
\textbf{3. Coin Toss Resolution.}  
There are occasions where changes in $\Delta \hat{H}(\xb)$ and $\Delta \hat{H}(\yb \mid \xb)$ are comparable, making selecting a reaction policy challenging. Instead of arbitrarily favoring the slightly larger change, we always use a biased coin toss approach, stochastically selecting which entropy to address based on the magnitude of changes:
\[
p_{\text{adjust}} = \frac{|\Delta \hat{H}(\xb)|}{|\Delta \hat{H}(\xb)| + |\Delta \hat{H}(\yb \mid \xb)|}, \quad p_{\text{fork}} = 1 - p_{\text{adjust}}.
\]
The decision variable $z$ is sampled as $z \sim \text{Bernoulli}(P_{\text{adjust}})$, with $z=1$ indicating sampling adjustment and $z=0$ indicating process forking. This mechanism ensures balanced reactions, robustness, and prevents overreactions to marginal signals.

The description above provides a brief summary of the MIS reaction policy. In the remaining portion of this subsection, we will present the specific MIS reaction policy in an algorithm. To do that, we first need to define a sampling process formally and then present the detailed algorithmic implementation of this reaction policy in Algorithm \ref{algo:mis}.
\begin{definition}
A sampling process $\Pc(\Xb, g(\cdot))$ consists of two components: existing observations $\Xb$ and a sampling function $g(\cdot)$, where the next sample location is determined by
\[ \xb_{\text{next}} \sim g(\Xb),\]
with $\xb_{\text{next}}$ drawn from the stochastic oracle $g(\Xb)$. If $g(\cdot)$ is deterministic, $\sim$ is replaced by equality ($=$). For clarity, a sampling process with $n$ existing observations is denoted $\Pc_n$.
\end{definition}

\begin{algorithm}[H]
\caption{Mutual Information Surprise Reaction Policy (MISRP)}
\label{algo:mis}
\begin{algorithmic}[1]
\Require A sampling process $\Pc(\Zb, g(\cdot))$, where $\Zb$ consists of $k$ pairs of input $\Xb$ and output $\Yb$; A maximum reflection threshold $T$; Reflection period $m=2$
\While {$m \leq \min(T,\frac{k}{2})$}
\State Set $n = k-m$; Compute $MIS = \hat{I}_{m+n} - \hat{I}_n$; Record $\Delta \hat{H}(\xb)$, $\Delta \hat{H}(\yb)$, and $\Delta \hat{H}(\yb \mid \xb)$
\If{$MIS \not\in MIS_{\pm}$ and $\frac{\text{sgn}(\text{MIS})\Delta \hat{H}(\yb)}{|\text{MIS}|} \neq \max \big\{\frac{\text{sgn}(\text{MIS})\Delta \hat{H}(\xb)}{|\text{MIS}|}, \frac{\text{sgn}(\text{MIS})\Delta \hat{H}(\yb)}{|\text{MIS}|}, \frac{\text{sgn}(\text{MIS})\Delta \hat{H}(\yb\mid\xb)}{|\text{MIS}|}\big\}$}
    \State Compute bias: $p \gets \frac{|\Delta \hat{H}(\xb)|}{|\Delta \hat{H}(\xb)| + |\Delta \hat{H}(\yb \mid \xb)|}$
    \State Sample $z \sim \text{Bernoulli}(p)$
    \If{$z = 1$} \Comment{Sampling Adjustment}
        \If{$MIS > MIS_+$}
            \State Modify $g$ to reduce exploration and increase exploitation
        \Else
            \State Modify $g$ to increase exploration and reduce redundancy
        \EndIf
        \State \textbf{break while}
    \Else \Comment{Process Forking}
        \If{$\Pc$ is forked and the other process is not requesting Process Forking}
        \State Delete $\Pc$; Merge the other process as the main process
        \State \textbf{break while}
        \EndIf
        
        \If{$\Pc$ is forked and the other process is requesting Process Forking}
        \State Delete the $\Pc$ with fewer data; Merge the other one as the main process
        \State \textbf{break while}
        \EndIf
        \State Fork process into two branches: $\mathcal{P}_n$ and $\mathcal{P}_m$
        \State Call $\text{MISRP}(\Pc_n, t)$ and $\text{MISRP}(\Pc_m, t)$
        \State \textbf{break while}
    \EndIf
\Else
    \State No action required (surprise within expected bounds)
\EndIf
\State $m = m+1$
\EndWhile
\end{algorithmic}
\end{algorithm}

We offer several remarks on the MIS reaction policy $\text{MISRP}(\Pc, t)$:
\begin{itemize}

    \item In the pseudocode, we introduce two additional notations: the maximum reflection threshold $T$ and the total number of observations $k$. In practice, MIS is computed retroactively, that is, given a sequence of $k$ observations, we partition them into $m$ recent observations and $n = k - m$ older observations to compute the MIS. We term the $m$ recent observation as the reflection period and we increment $m$ to iterate over different partition points. The reflection period $m$ is constrained to be no greater than $\min(T, \frac{k}{2})$. This constraint is motivated by the comparative behavior of test statistics derived from Theorem \ref{the:mis} and the variance-based test in Eq.~\eqref{eq:vartest}. Specifically, when $m = n$, both our proposed test and the variance-based test yield statistics of order $\Oc\left(\frac{\log n}{\sqrt{n}}\right)$. As discussed in Section \ref{subsec:bound}, such statistics are typically too loose to be violated in practice, thereby diminishing the sensitivity advantage of our method. Consequently, evaluating MIS beyond $m = \frac{k}{2}$ is unnecessary and computationally inefficient. The reflection threshold $T$ is introduced to ensure computational feasibility, and we recommend selecting $T$ as large as computational resources permit.

    \item Note that the reflection period $m$ starts at $2$. This implies that the reaction policy does not respond to a single-instance surprise. Mathematically, this is because the derivation of the bound in Theorem \ref{the:mis} is ill-defined for $m=1$. Intuitively, MIS measures the progression of learning in a sampling process, and it is impossible to determine whether a single observation is informative or erroneous without additional verification. Therefore, the MIS policy always take at least two additional samples to start to react. One may argue that this requirement for an extra sample imposes additional cost in conducting experiments.  That is true.  But recall one insight from the study in \cite{ahmed2024toward} is the benefit of ``\textit{the extra resources spent on deciding the nature of an observation}'' in the long run.
\end{itemize}

\begin{itemize}

    \item It is important to emphasize that bot the sampling adjustment and process forking approaches are rooted in the active learning literature and practice. Balancing exploration and exploitation, i.e., sampling adjustment, has long been a key topic in Bayesian optimization and active learning \cite{bondu2010exploration}, whereas discarding irrelevant observations, as we do in process forking, is a common practice in the dataset drift literature \cite{moreno2012unifying, sugiyama2007covariate, bickel2009discriminative, vzliobaite2016overview, zhang2023concept}. Our Mutual Information Surprise reaction framework provides a principled mechanism for autonomous systems to determine how to balance exploration versus exploitation and when or what to discard (i.e., forget).

\end{itemize}

\section{Numerical Analysis}\label{sec:simu}

In this section, we illustrate the merits of Mutual Information Surprise (MIS). Section~\ref{subsec:sythetic} demonstrates the strength of MIS compared to classical surprise measures. Section~\ref{subsec:real} showcases the advantages of the MIS reaction policy in the context of dynamically estimating a pollution map using data generated from a physics-based simulator.

\subsection{Putting Surprise to the Test}\label{subsec:sythetic}




To compare MIS with classical surprise measures—principally Shannon and Bayesian Surprises—we conduct a series of controlled simulations using a simple yet interpretable system, designed to reveal how each measure behaves under varying conditions. The system is governed by the mapping
\begin{equation}\label{eq:mod}
    y = x \mod 10,
\end{equation}
chosen for its simplicity, modifiability, and clarity of interpretation. The first four scenarios are fully deterministic, while the final two introduce noise and perturbations, enabling an assessment of whether each surprise measure responds meaningfully to new observations, structural changes, or stochastic disturbances. Each simulation begins with $100$ samples drawn uniformly from $x \in [0, 30]$ to establish the system’s initial knowledge. We then progressively introduce new data under different conditions, recording the response of each surprise measure. As the magnitudes of MIS, Shannon Surprise, and Bayesian Surprise differ in scale, our analysis focuses on \textit{behavioral trends}—how each measure changes, spikes, or saturates—rather than on their absolute values.

The surprise measures are computed as follows. Shannon Surprise is calculated using its classical definition in Eq.~\eqref{eq:Shannon-conditional}, namely as the negative log-likelihood of the true label under a Gaussian Process (GP) predictive model. Bayesian Surprise is computed using Postdictive Surprise, defined in Eq.~\eqref{eq:postdictive}, by evaluating the KL divergence between the prior and posterior predictive distributions of $\yb$ at each input $\xb$.

While the current definition of Shannon and Bayesian Surprises are single instance, they can be extended to multi-instance settings. Specifically, for a sequence of $m$ independently sampled observation pairs $(\Xb,\Yb) = \{(\xb_i,\yb_i)\}_{i=1}^m$, the cumulative Shannon Surprise (CSS) is defined as
\begin{equation}\label{eq:CSS}
    S_{\text{Shannon}}(\Xb,\Yb) = \sum_{i=1}^{m} S_{\text{Shannon}}(\xb_i,\yb_i \mid p_{i-1}),
\end{equation}
where $p_{i-1}$ denotes the predictive distribution $p(\yb \mid \xb)$ prior to observing $(\xb_i,\yb_i)$. This formulation is also known as the \textit{cumulative prequential log score} \cite{dawid1984present}, which has been widely used in sequential concept drift and anomaly detection literature \cite{lee2020repad, ayed2020anomaly, bayram2022concept}.

For Bayesian (Postdictive) Surprise, we define the multi-instance version as the divergence between predictive distributions evaluated on the most recent observation $(\xb_m,\yb_m)$, after updating parameters from $\thetab$ to $\thetab'$ using the dataset $(\Xb,\Yb)$. Formally,
\begin{equation*}
    S_{\text{Postdictive}}(\Xb, \Yb) = D_{\text{KL}}\left(P(\yb_m \mid \thetab', \xb_m) \,\|\, P(\yb_m \mid \thetab, \xb_m)\right).
\end{equation*}
We consider two practical multi-instance variants: a \textit{cumulative} version that aggregates surprise values from the beginning of the sequence, and a \textit{rolling-window} version that measures cumulative surprise within the most recent $W=20$ sampling steps.

All versions of Shannon and Bayesian Surprises use the same Gaussian Process predictive model, with a Matérn kernel ($\nu=2.5$) and a fixed noise level of $0.1$. Except for the multi-instance Bayesian Surprise, the model is retrained using all available observations after each surprise computation.

For MIS, we treat the initial $100$ observations as the initial sample size $n=100$, as defined in Section~\ref{subsec:expect}. As sampling continues, the number of new observations $m$ increases (represented in the ticks of the X-axis in the figures). The output space has cardinality $|\Yc| = 10$, corresponding to the ten possible outcomes of the modulus function, except in Scenario 6 where $|\Yc| = 20$. MIS is calculated as defined in Eq.~\eqref{MIS}. When the theoretical bound in Theorem~\ref{the:mis} is used, the probability level is set to $\rho=0.1$. The bias term is adjusted as discussed in Section~\ref{subsec:bound}, since $n \gg |\Yc|$ in this setting.

\begin{figure}[h]
    \centering
    \includegraphics[width=0.45\linewidth]{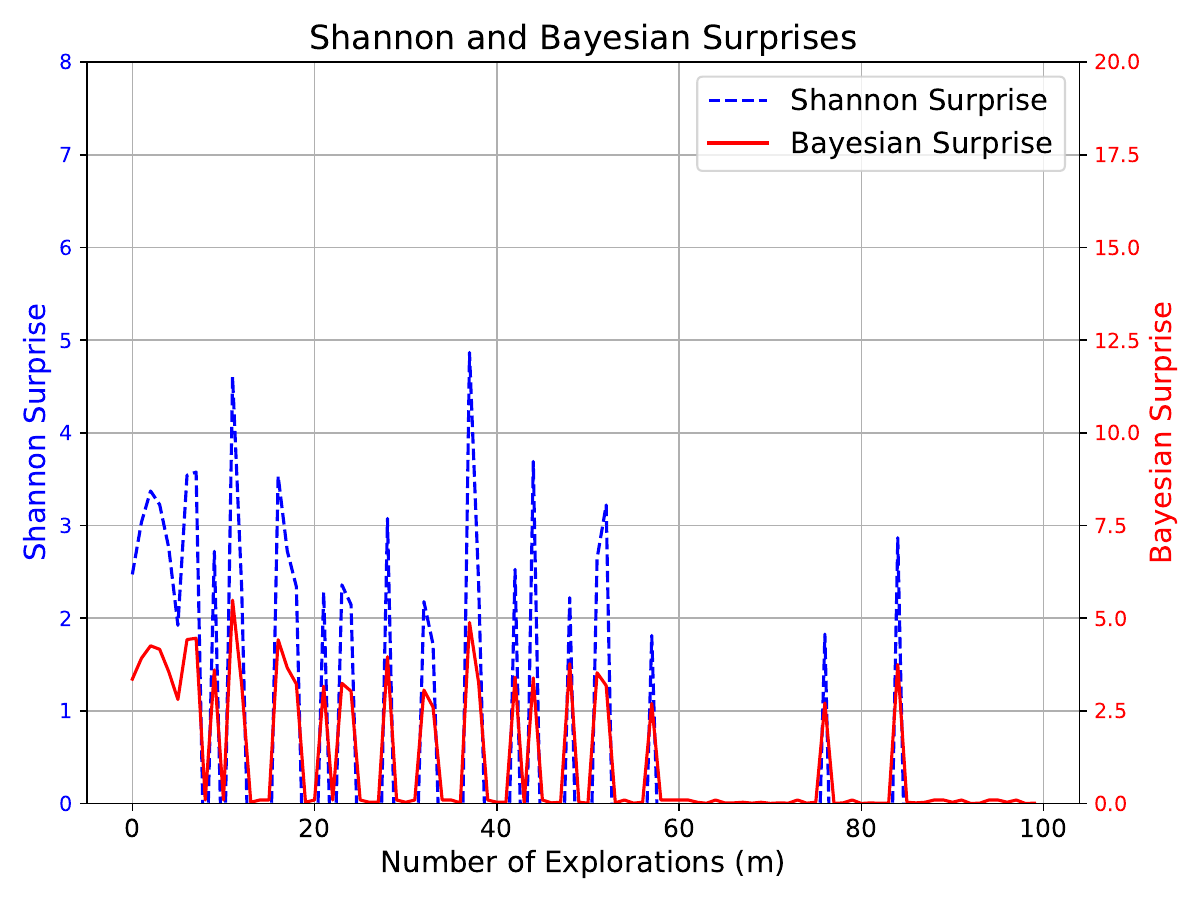}
    \quad
    \includegraphics[width=0.45\linewidth]{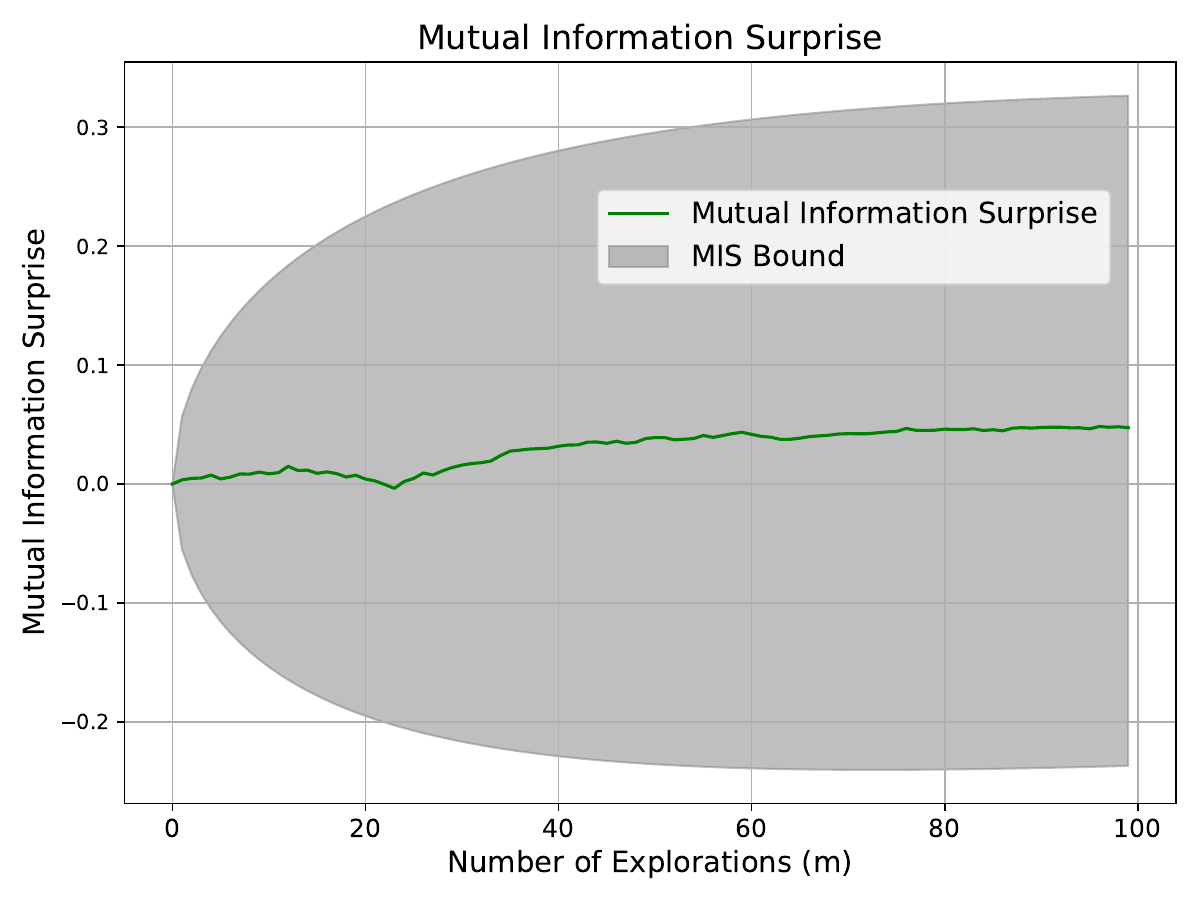}
    \includegraphics[width=0.45\linewidth]{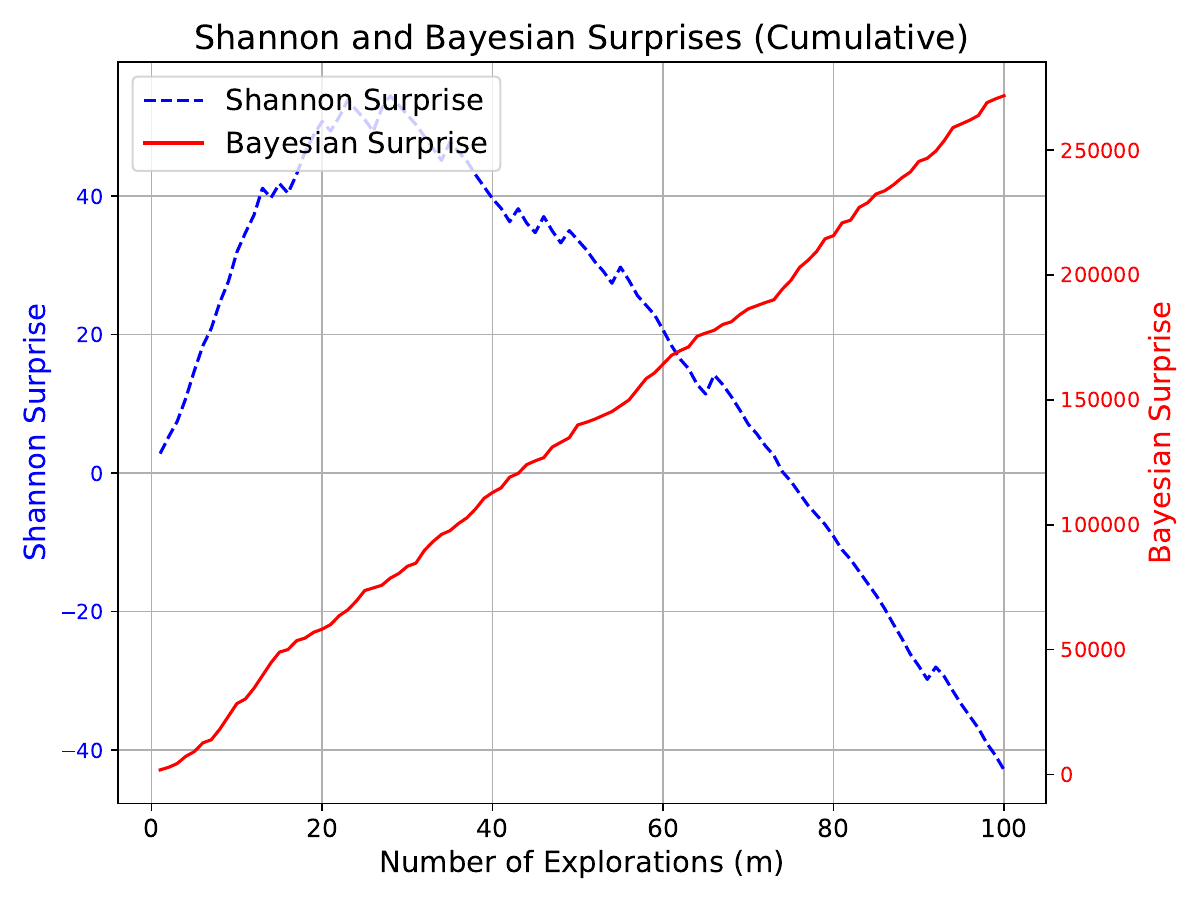}
    \quad
    \includegraphics[width=0.45\linewidth]{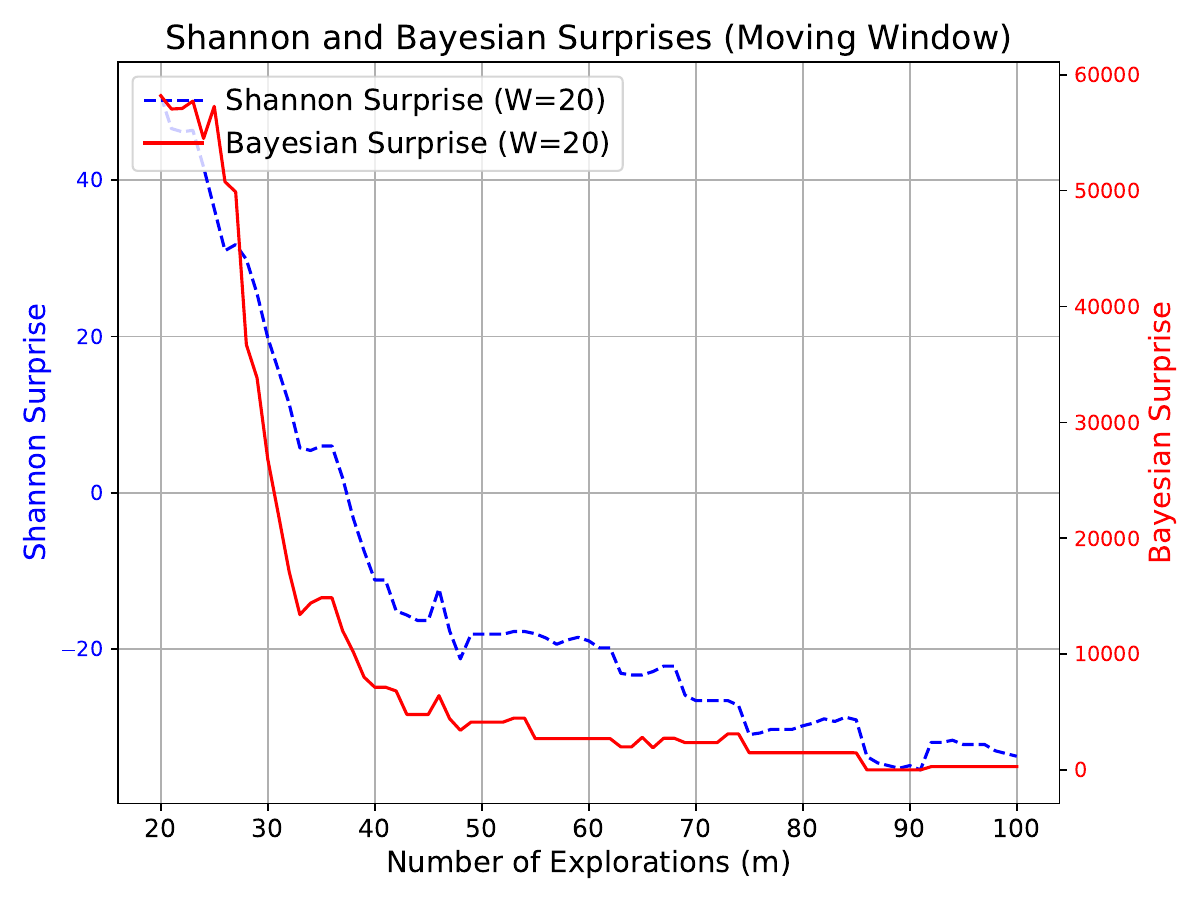}
    \caption{Surprise measures during standard exploration.}
    \label{fig:explore_std}
\end{figure}

\vspace{1em}
\noindent\textbf{Scenario 1: Standard Exploration.} \noindent New data is randomly sampled from $x \in [30, 100]$, expanding the domain while adhering to the same underlying response function in Eq.~\eqref{eq:mod}. As the sampling region broadens, the input distribution $p(\xb)$ evolves, which in turn alters the joint distribution $p(\xb,\yb)$ and the \textit{observed} functional relationship $p(\yb \mid \xb)$. Notably, this observed mapping may differ from the true system mapping in sparsely explored regions, where it is effectively inferred through extrapolation rather than direct observation.

This setting exemplifies a well-regulated autonomous system as defined in Definition~\ref{def:regulated}: although the joint distribution and the \textit{observed} mapping evolve over time, the underlying input-output relationship, and thus the mutual information $I(\xb;\yb)$, remains stable.

\textit{Expected behavior:} Since the system is well-regulated, a meaningful surprise measure should not react to exploration-driven distributional shifts. In particular, variations in $p(\xb)$ and the induced changes in $p(\yb \mid \xb)$ should not be interpreted as unexpected events. Accordingly, we do not expect MIS to be violated.

As shown in Figure~\ref{fig:explore_std}, MIS evolves steadily within its expected bounds, reflecting stable learning despite continuous changes in the observed data distribution. In contrast, the single-instance Shannon and Bayesian Surprises fluctuate erratically, frequently producing spikes that lack clear interpretability. Their cumulative and rolling variants, while smooth, tend to exhibit largely monotonic trends that offer little actionable insight for decision-making. Notably, the cumulative Shannon Surprise shows a pivot around $25$ sampling steps, which can be attributed to increased coverage of the sampling space: as the model becomes more familiar with the environment, familiar observations begin to dominate and outweigh unseen observations in new samples.

\begin{figure}[h]
    \centering
    \includegraphics[width=0.45\linewidth]{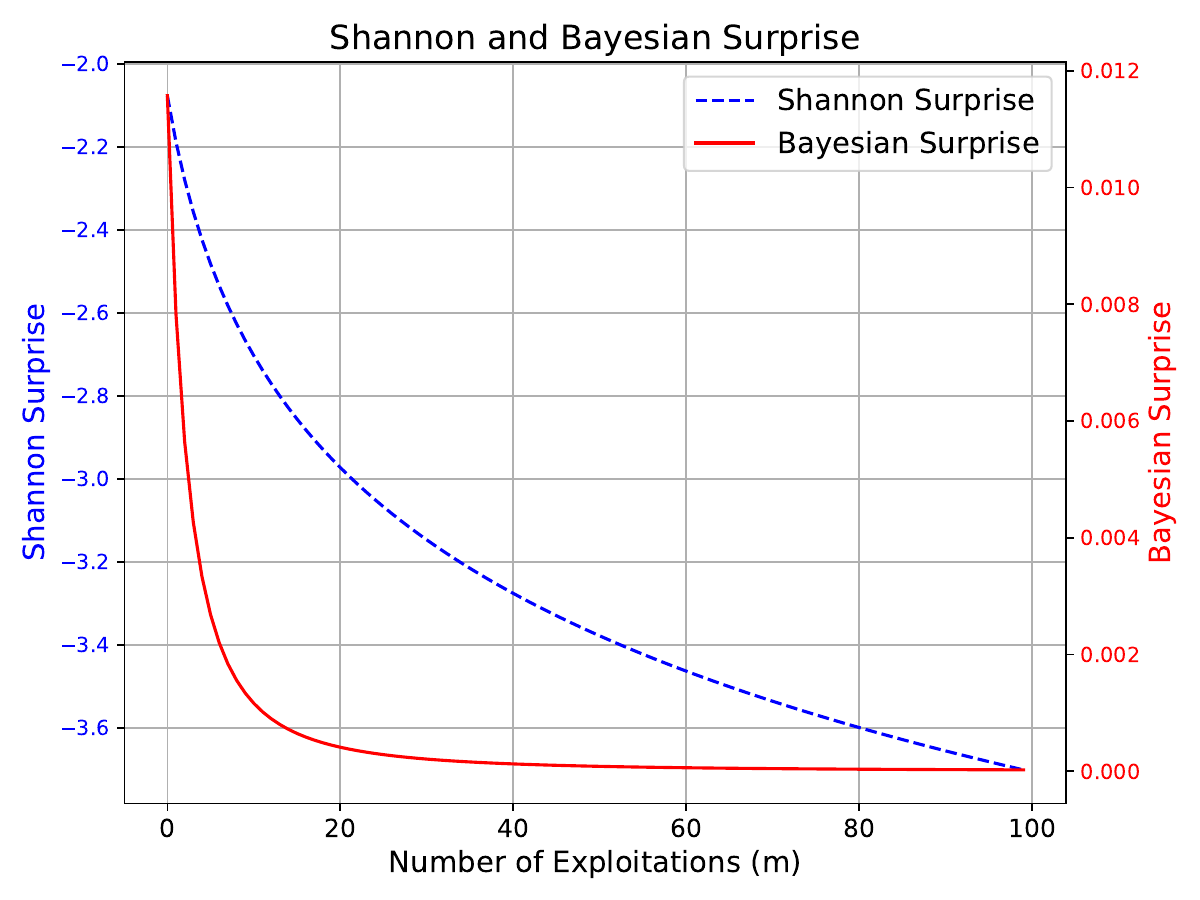}
    \quad
    \includegraphics[width=0.45\linewidth]{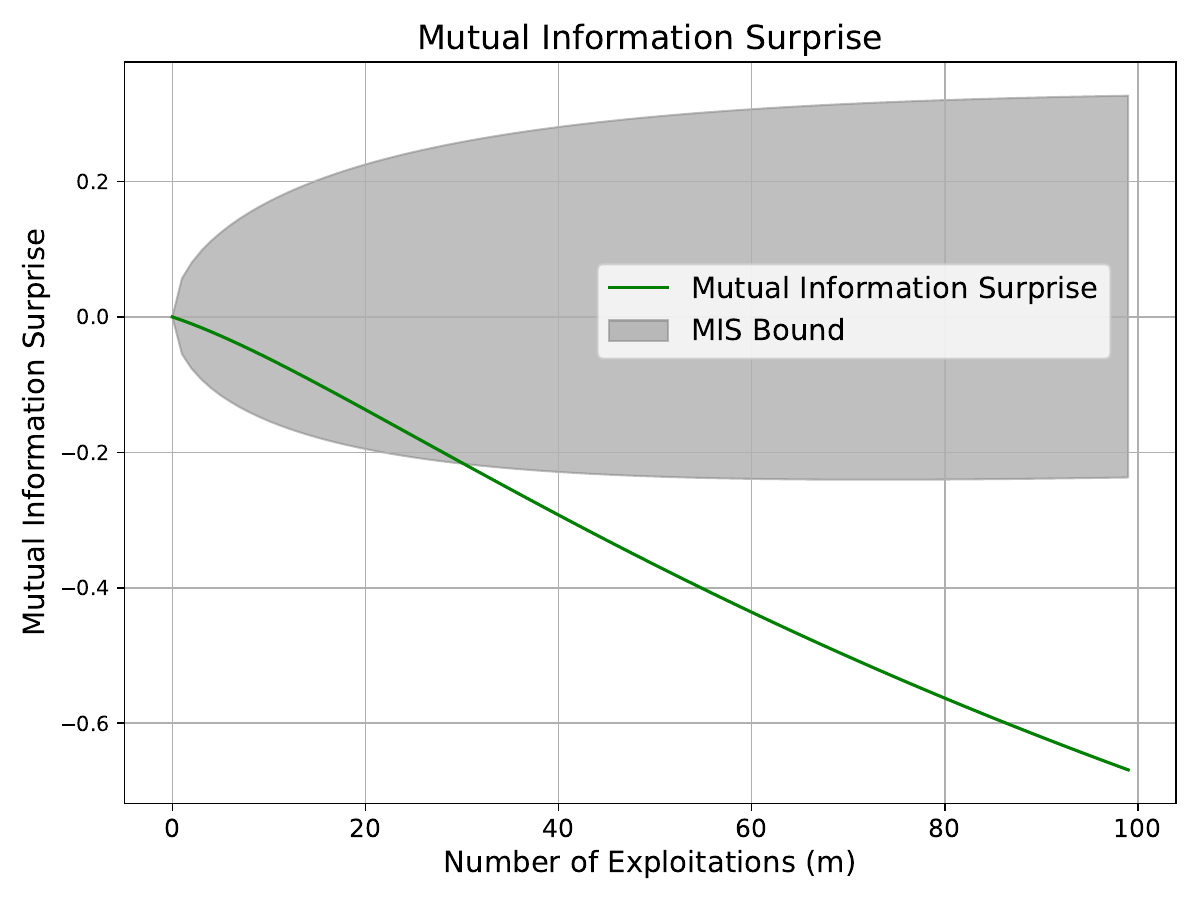}
    \includegraphics[width=0.45\linewidth]{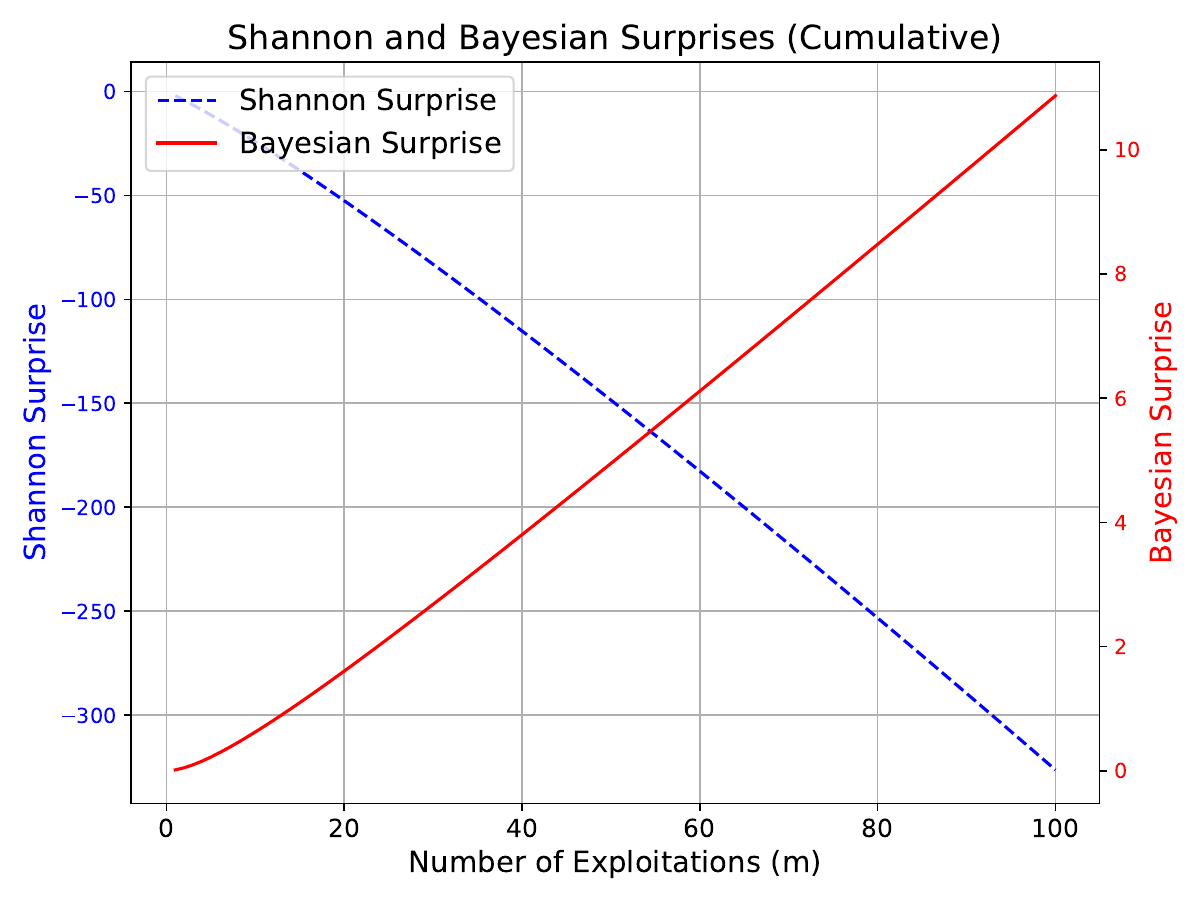}
    \quad
    \includegraphics[width=0.45\linewidth]{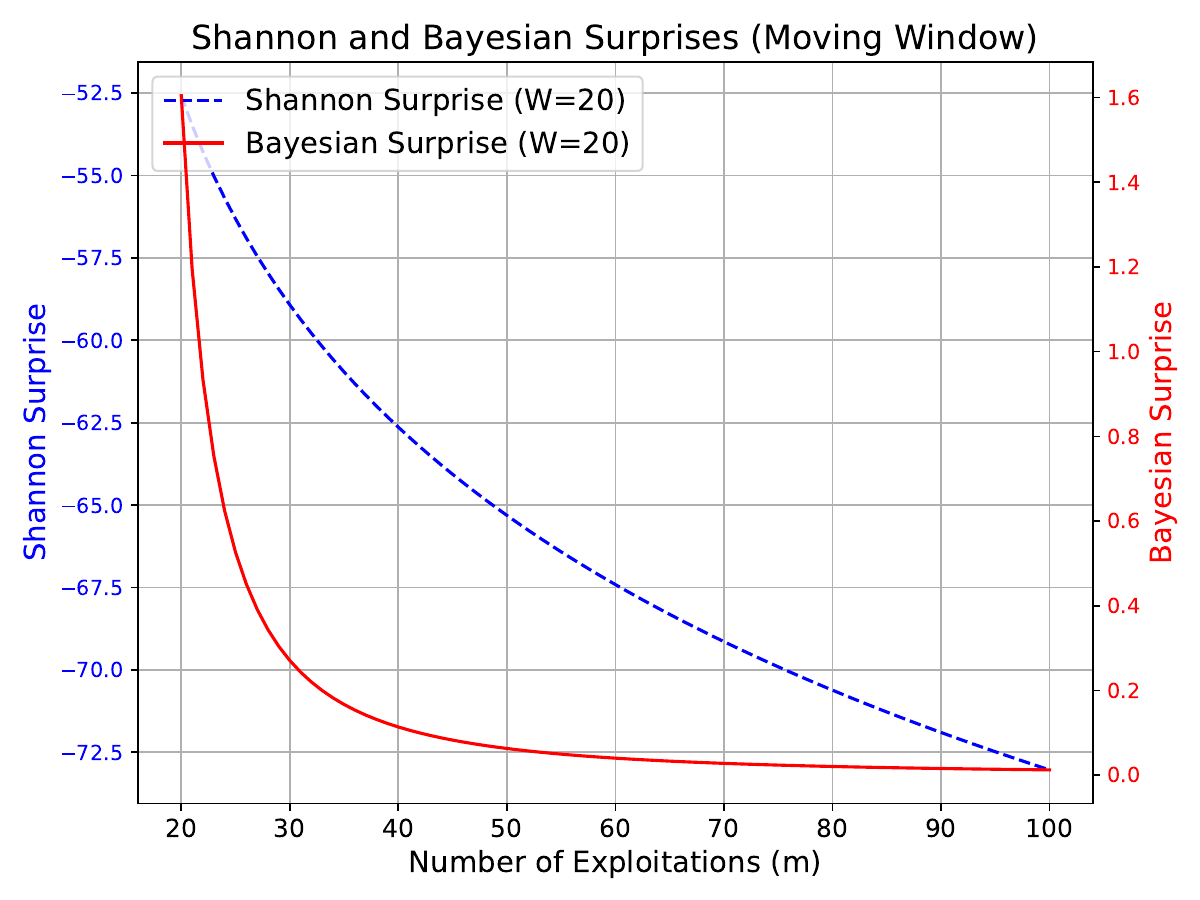}
    \caption{Surprise measures under over-exploitation.}
    \label{fig:exploitation}
\end{figure}

\vspace{1em}
\noindent\textbf{Scenario 2: Over-Exploitation.} \noindent In this scenario, the system repeatedly samples a previously seen point from $x \in [0,30]$, specifically observing the pair $(x, y) = (7,7)$ one hundred times. This simulates stagnation.

\textit{Expected behavior:} Surprise should diminish as no new information is gained. This mirrors the stagnation case in Section~\ref{subsec:implications}, and we expect MIS to violate its lower bound.

Figure~\ref{fig:exploitation} shows that MIS falls below its lower bound, clearly signaling a lack of knowledge gain. Although single-instance Shannon and Bayesian Surprises also exhibit a downward trend, they lack a principled lower threshold, limiting their ability to reliably detect such behavior. As noted in \cite{zamiri2022bayesian} and \cite{ahmed2024toward}, both Shannon and Bayesian Surprises are inherently one-sided measures, which further constrains their interpretability. Similar patterns are observed for their multi-instance variants, with one notable exception: cumulative Bayesian Surprise increasingly labels the exploitation behavior as more surprising over time, whereas rolling Bayesian measures deem it progressively less surprising. Nevertheless, due to its KL-divergence formulation, Bayesian Surprise lacks a clear probabilistic threshold, rendering this trend difficult to translate into actionable signals.

\begin{figure}[h]
    \centering
    \includegraphics[width=0.45\linewidth]{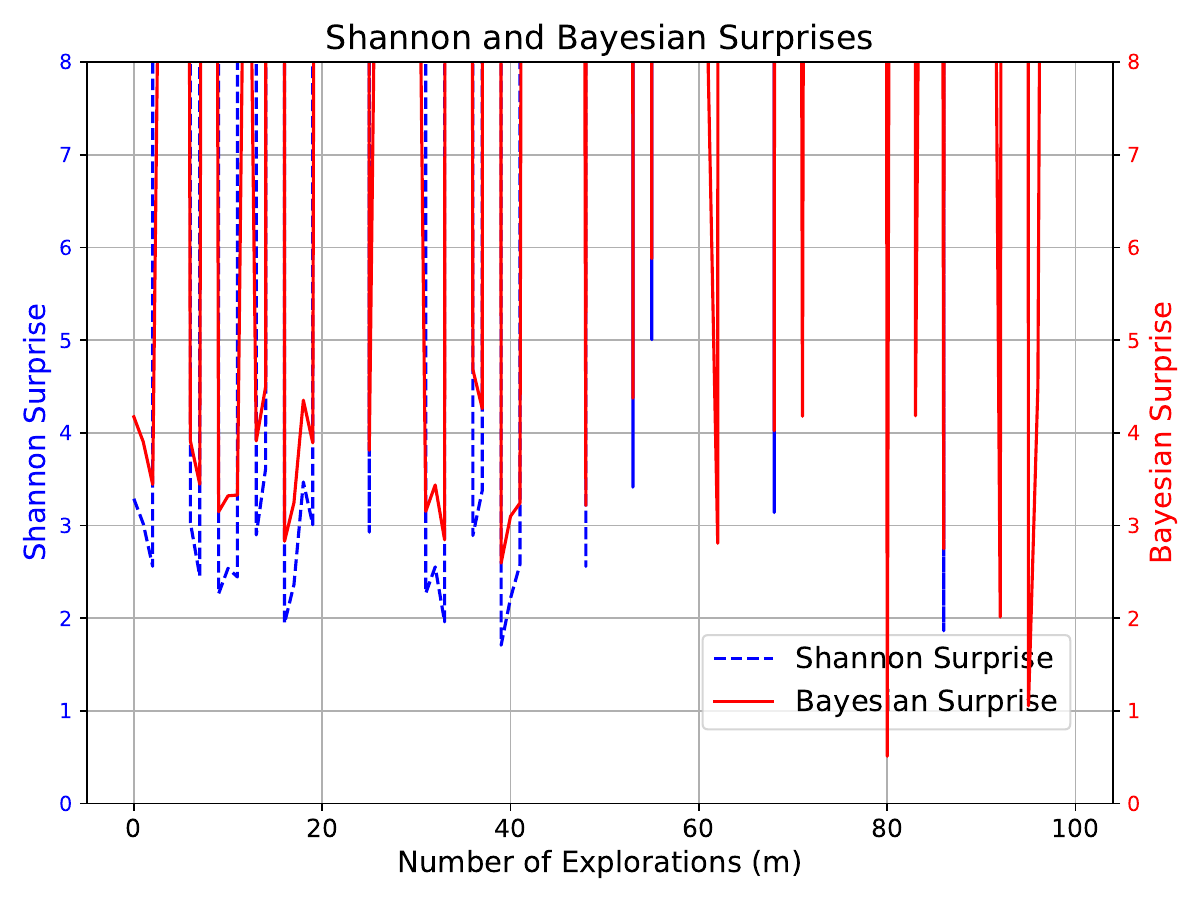}
    \quad
    \includegraphics[width=0.45\linewidth]{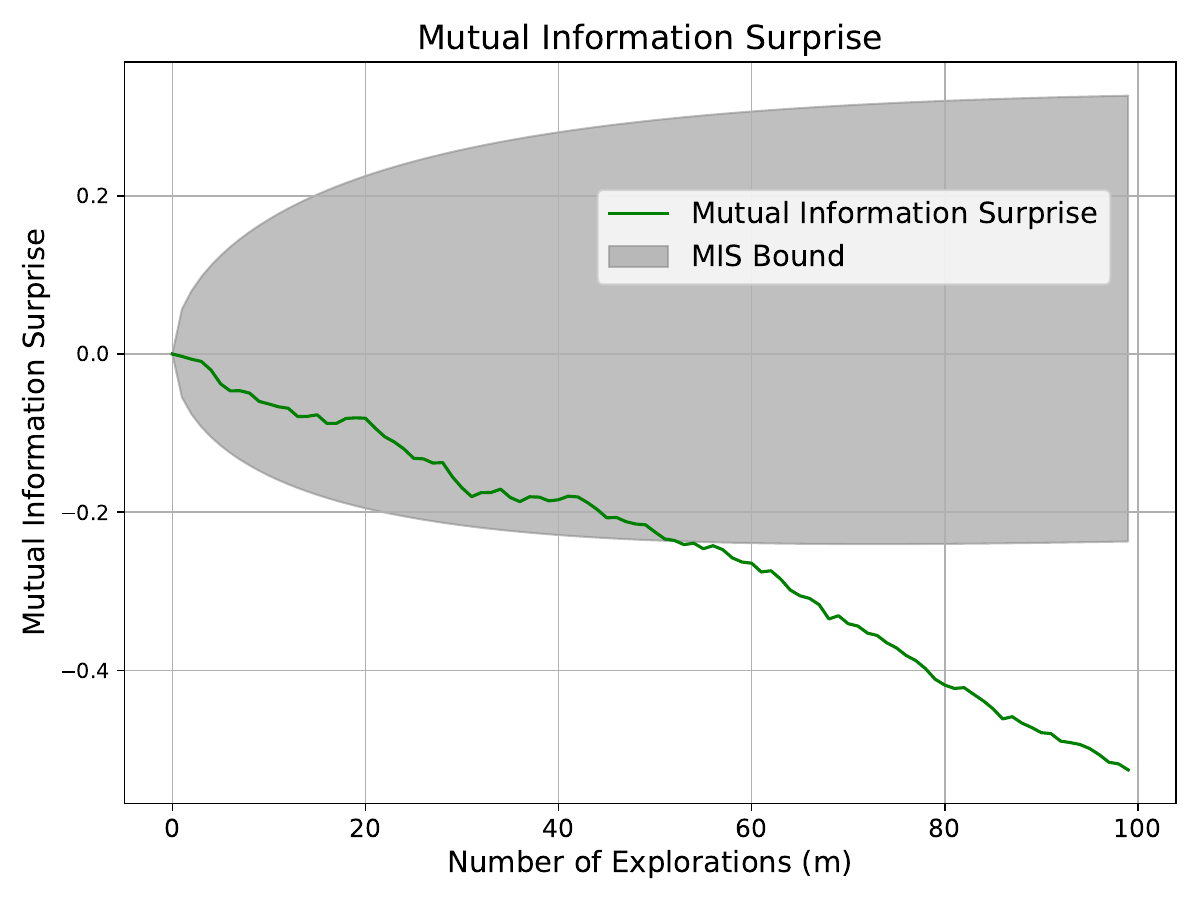}
    \includegraphics[width=0.45\linewidth]{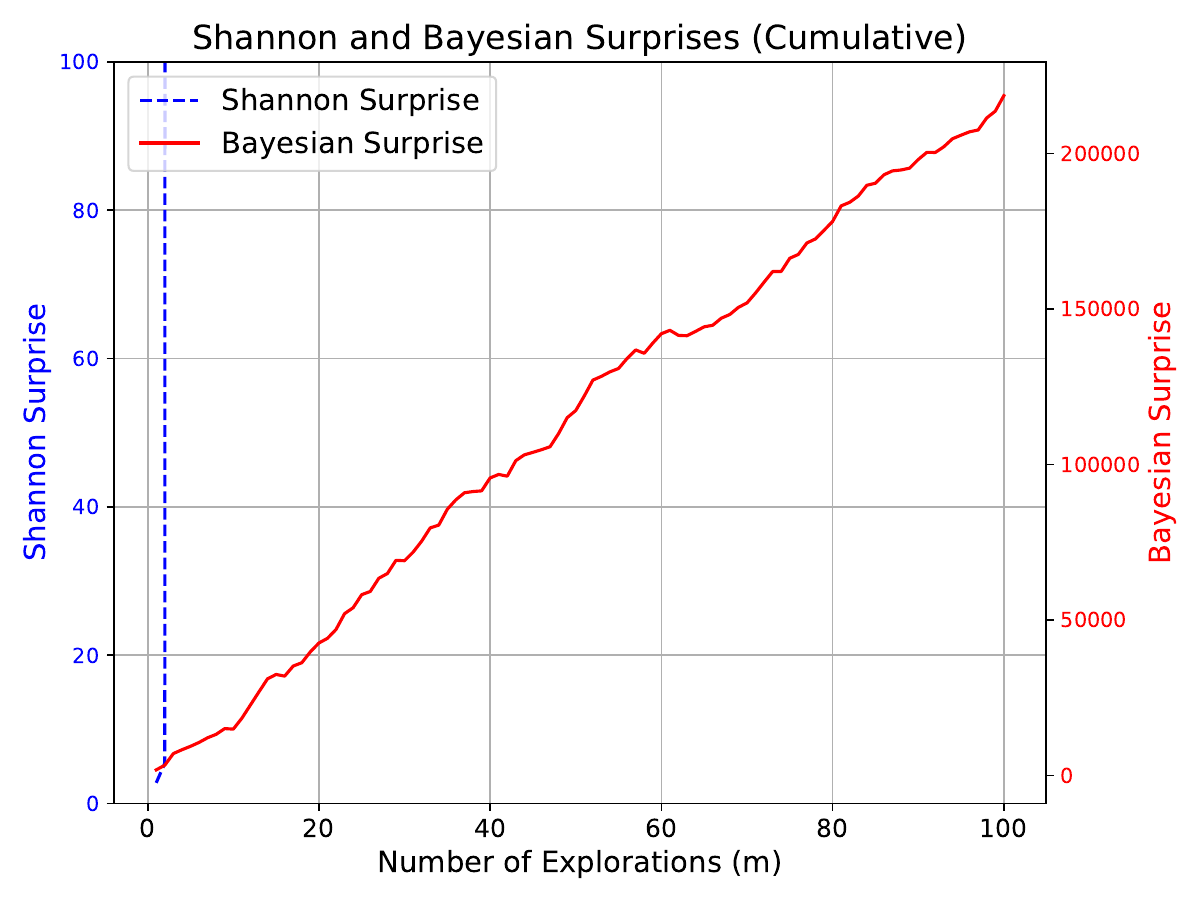}
    \quad
    \includegraphics[width=0.45\linewidth]{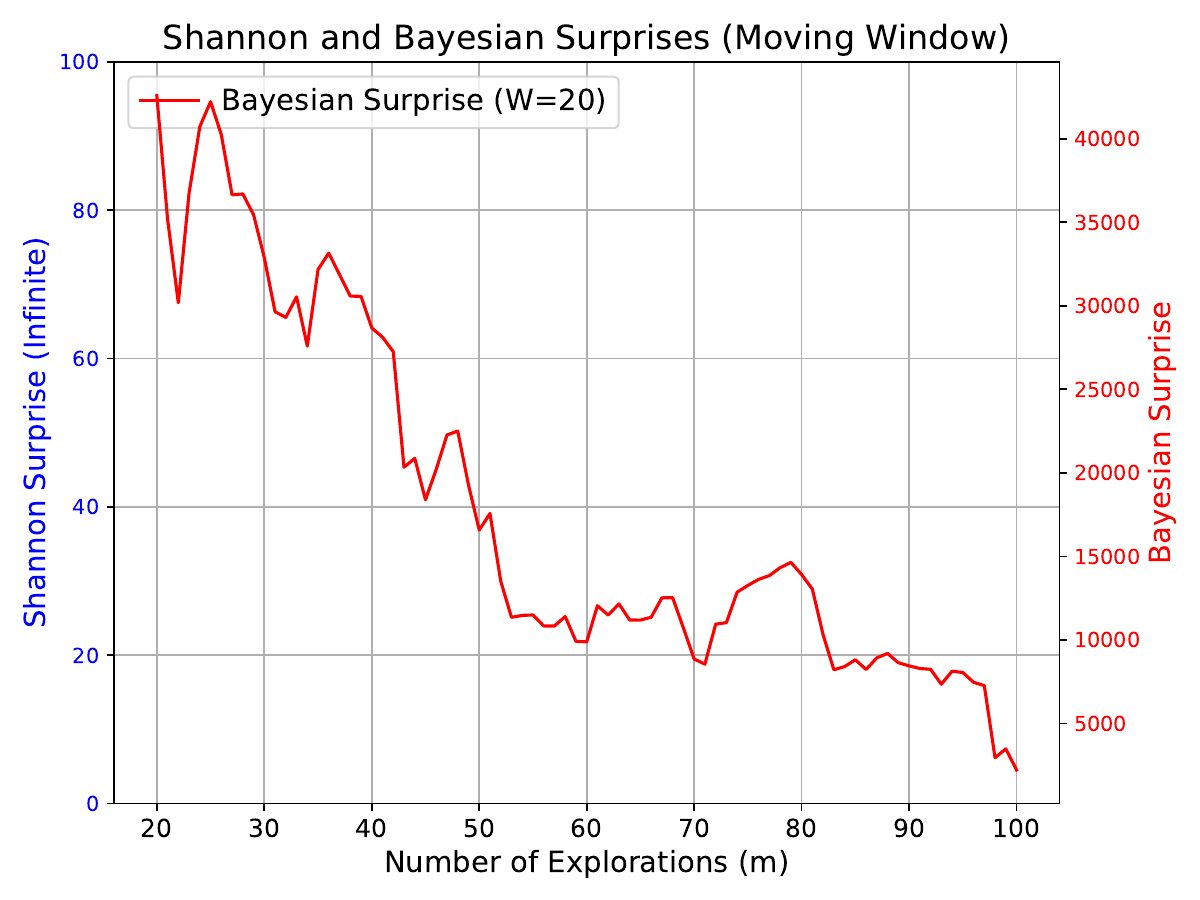}
    \caption{Surprise measures under noisy exploration.}
    \label{fig:explore_noise}
\end{figure}

\vspace{1em}
\noindent\textbf{Scenario 3: Noisy Exploration.} \noindent We perform standard exploration over $x \in [30,100]$ but apply random corruption to the outputs $\yb$, replacing each with a uniformly random digit between $0$ and $9$. This simulates exploration without informative feedback.

\textit{Expected behavior:} Despite novel inputs, the system should register confusion if understanding fails to improve. This mirrors the noise-increase case in Section~\ref{subsec:implications}, and we expect MIS to violate its lower bound.

Figure~\ref{fig:explore_noise} confirms this pattern: MIS drops below its expected range, accurately signaling knowledge loss. In contrast, Shannon and Bayesian Surprises again exhibit erratic behavior without consistent trends. Meanwhile, the multi-instance Shannon Surprise diverges to infinity, and the cumulative and rolling variants of Bayesian Surprise display contradictory monotonic patterns, further limiting their interpretability and practical usefulness.

\begin{figure}[h]
    \centering
    \includegraphics[width=0.45\linewidth]{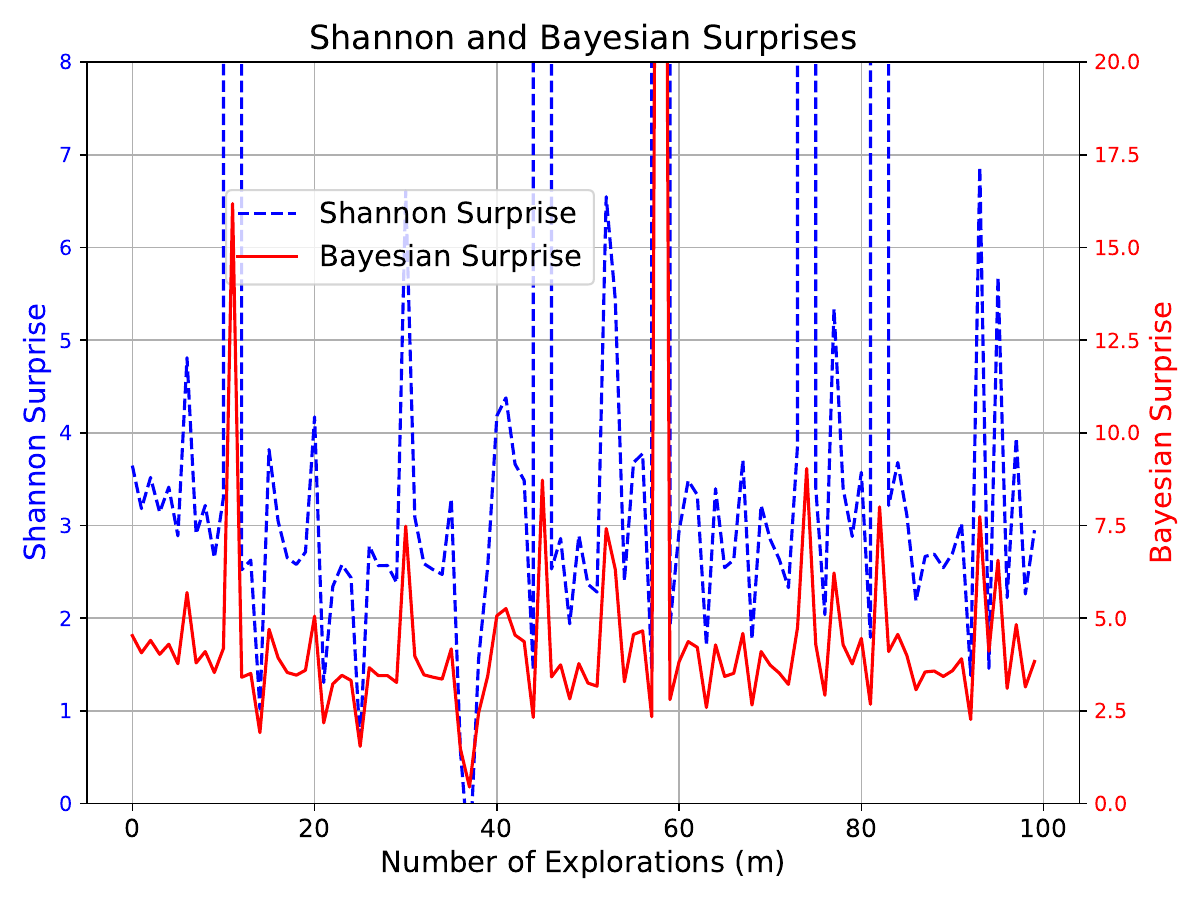}
    \quad
    \includegraphics[width=0.45\linewidth]{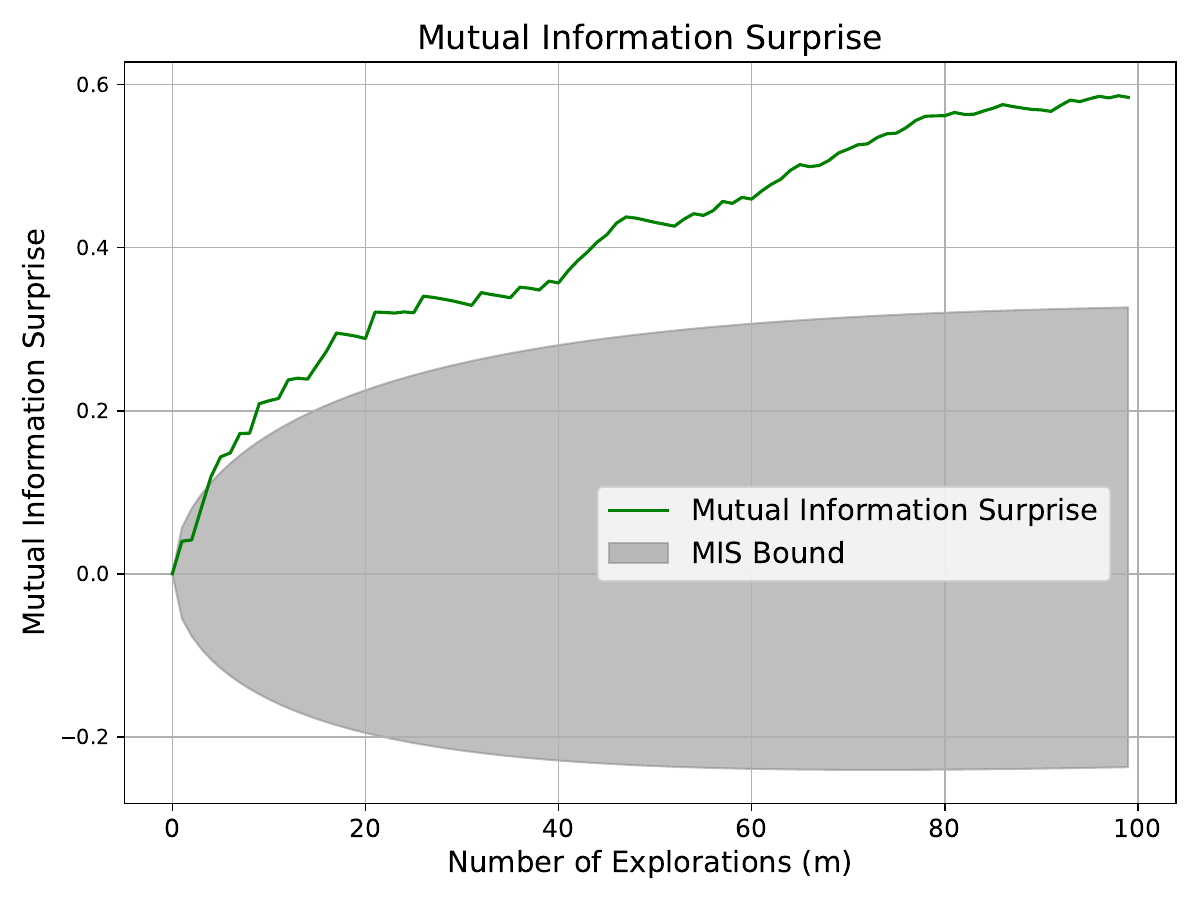}
    \includegraphics[width=0.45\linewidth]{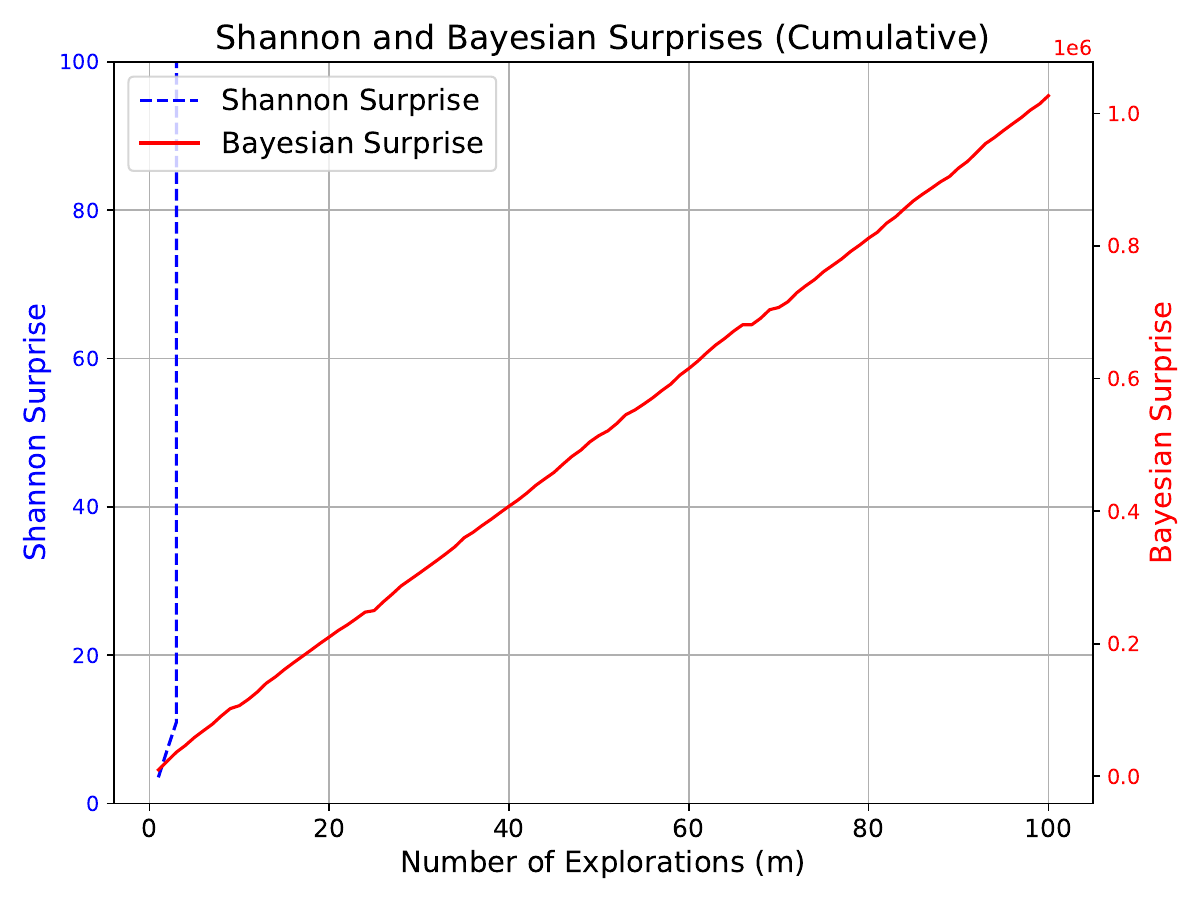}
    \quad
    \includegraphics[width=0.45\linewidth]{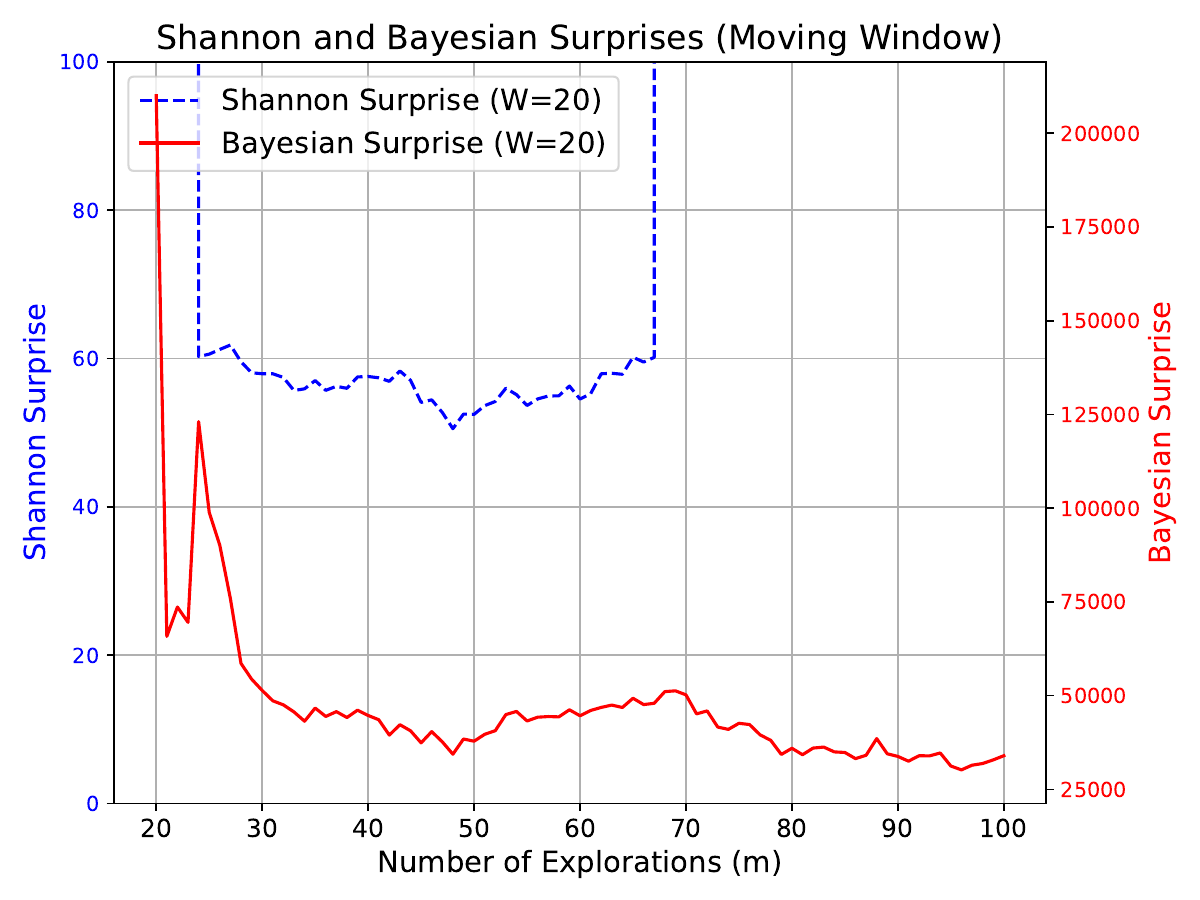}
    \caption{Surprise measures during aggressive exploration.}
    \label{fig:pure_explore}
\end{figure}

\vspace{1em}
\noindent\textbf{Scenario 4: Aggressive Exploration.} \noindent This scenario enforces strict exploration over $x \in [30,500]$, where each new sample is far from all observed points (i.e., outside the $\pm 1$ neighborhood range).

\textit{Expected behavior:} Aggressive exploration without verification can lead to overconfidence. This mirrors the aggressive exploration case in Section~\ref{subsec:implications}, and we expect MIS to exceed its upper bound.

Figure~\ref{fig:pure_explore} shows MIS exceeding its upper bound, consistent with the expected behavior under pure exploration. In contrast, the single-instance Shannon and Bayesian Surprises again fluctuate unpredictably. The multi-instance Shannon Surprise diverges to infinite values, while the cumulative and rolling versions of Bayesian Surprise once more exhibit contradictory monotonic trends, limiting their interpretability.

\begin{figure}[h]
    \centering
    \includegraphics[width=0.45\linewidth]{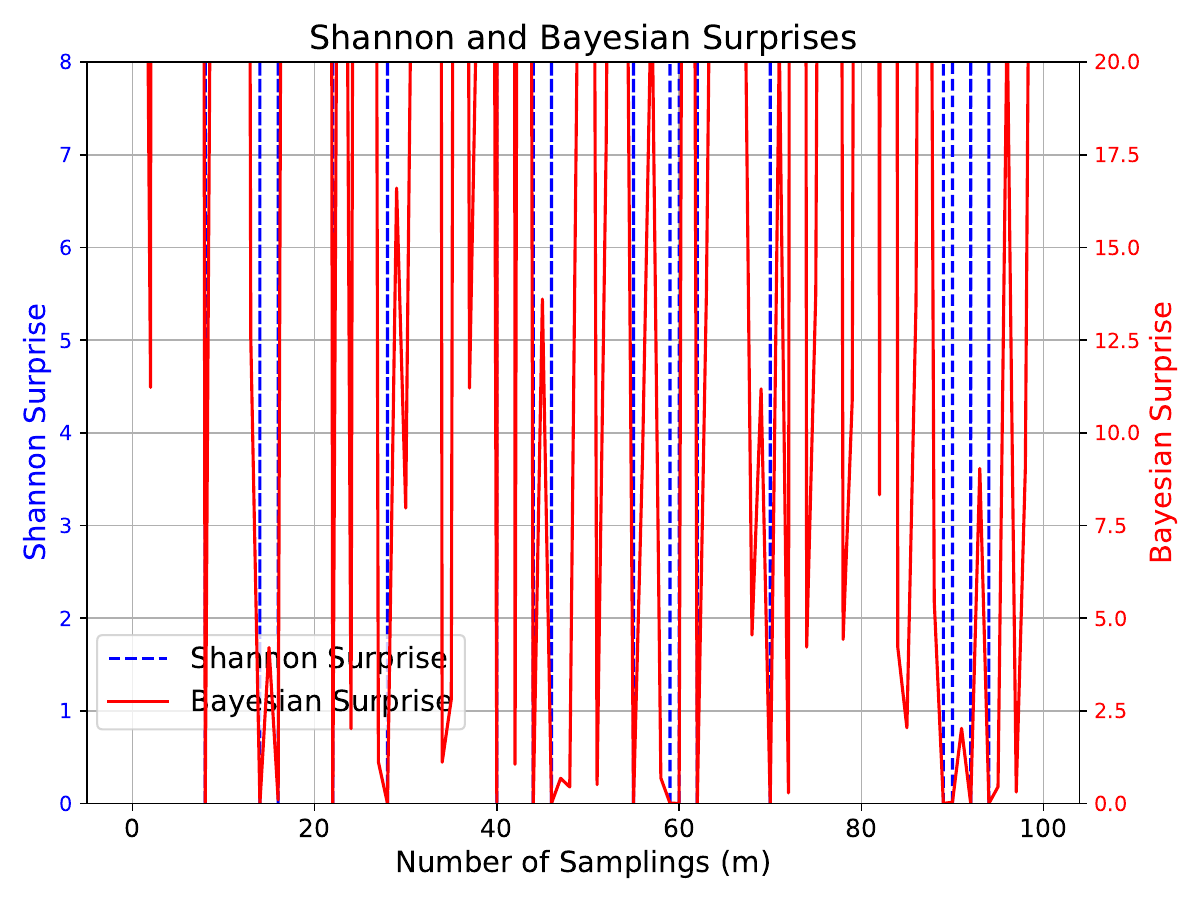}
    \quad
    \includegraphics[width=0.45\linewidth]{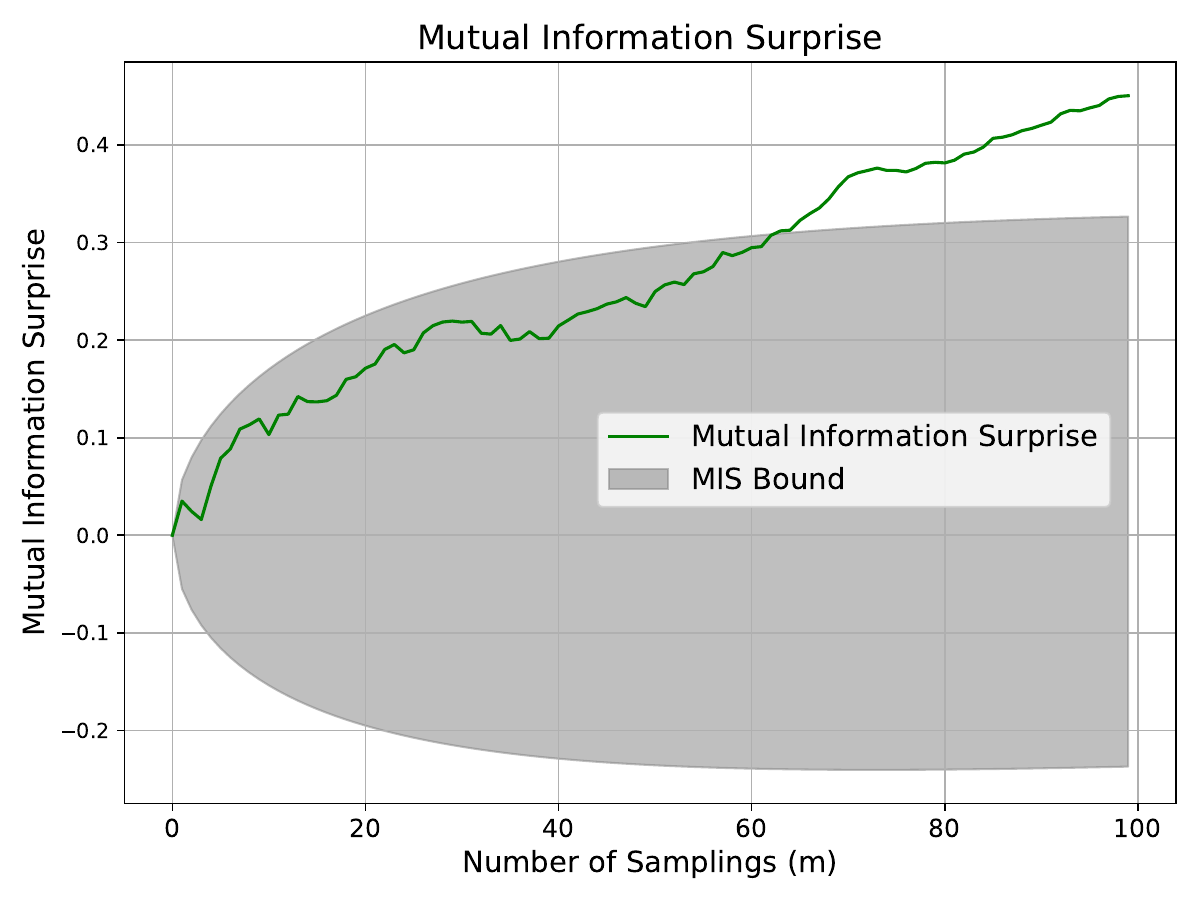}
    \includegraphics[width=0.45\linewidth]{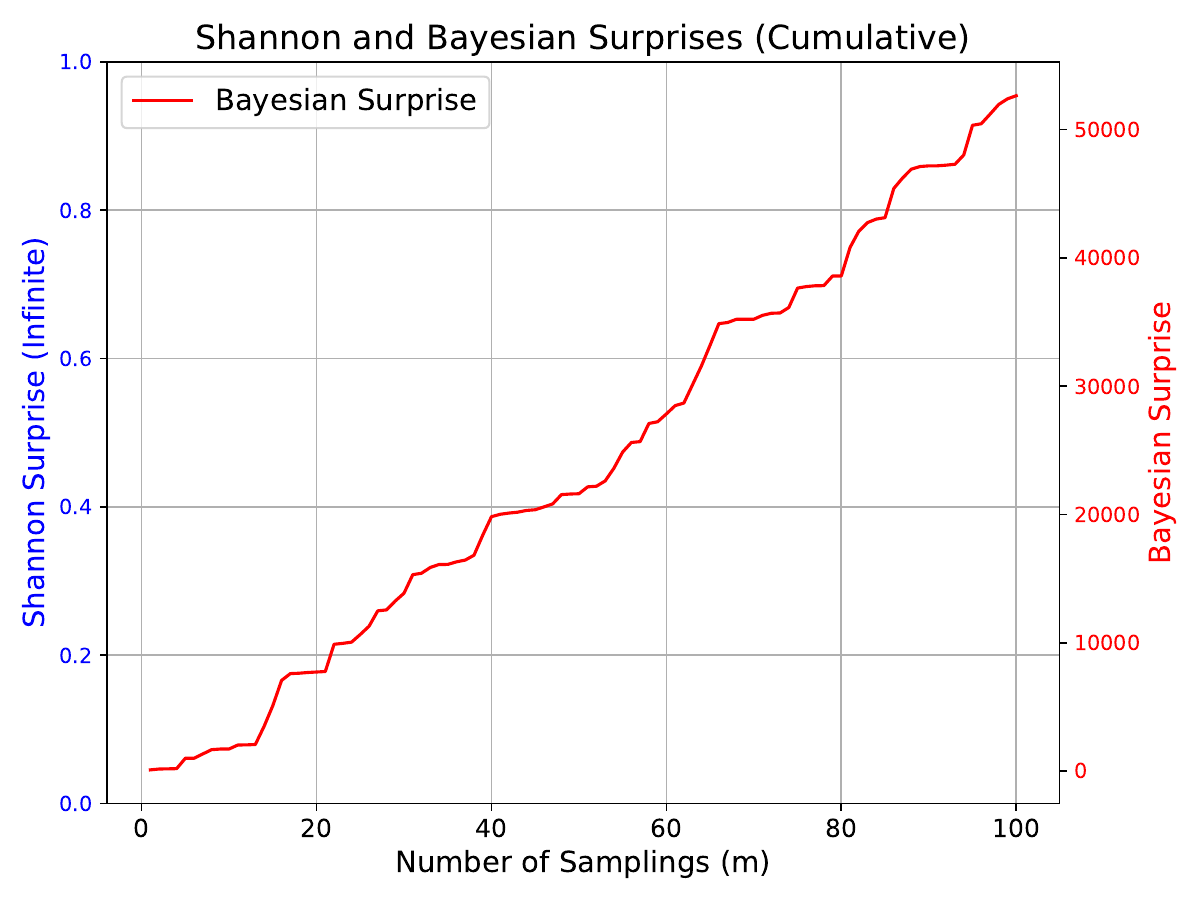}
    \quad
    \includegraphics[width=0.45\linewidth]{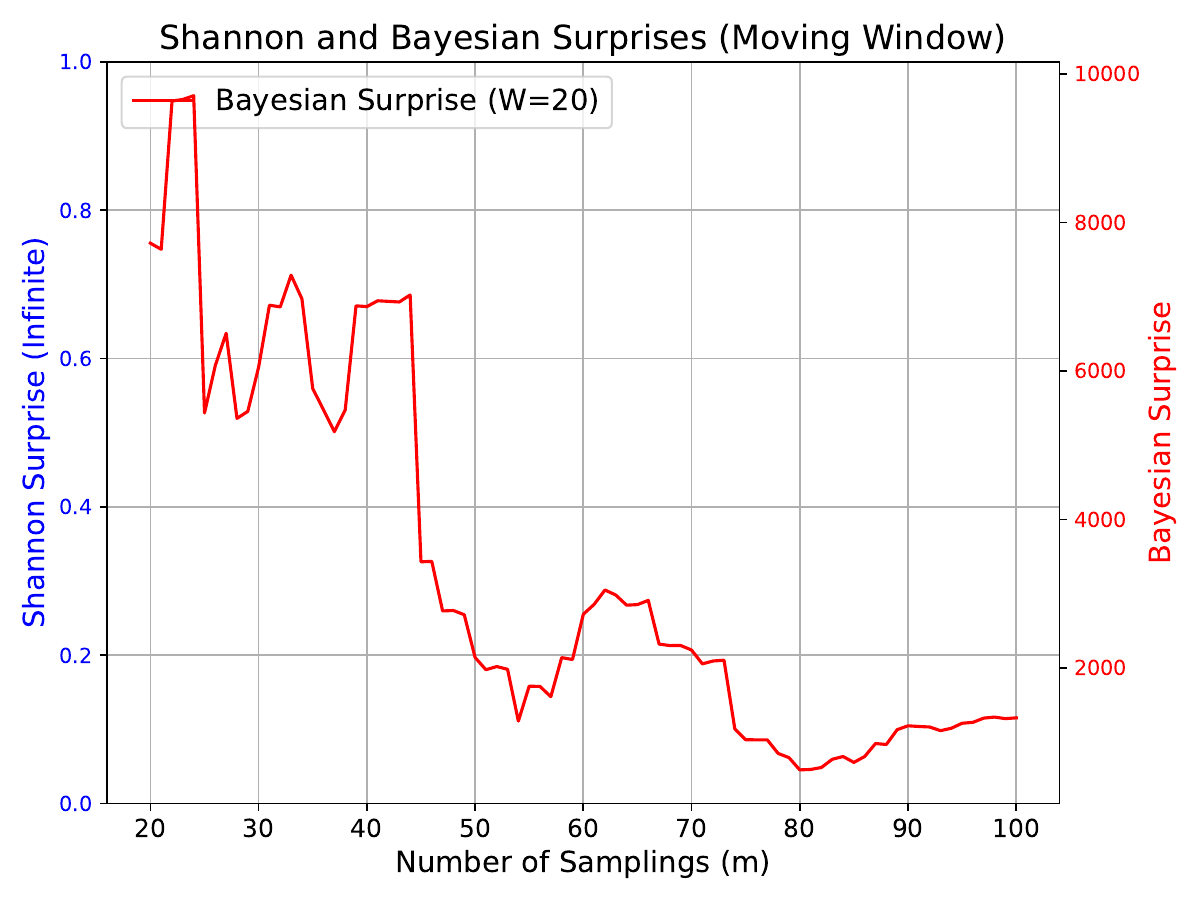}
    \caption{Surprise measures during noise decrease.}
    \label{fig:noise_decrease}
\end{figure}

\vspace{1em}
\noindent\textbf{Scenario 5: Noise Decrease.} \noindent To simulate noise reduction, we begin with $100$ initial observations from $x \in [0,30]$, paired with a randomly assigned output $y \in [0,9]$. New samples are drawn from the same $x$ range but the new $y$ is produced using the deterministic modulus function in Eq.~(\ref{eq:mod}).

\textit{Expected behavior:} Reduced noise implies stronger input-output dependency and we thus expect MIS to exceed its upper bound.

Figure~\ref{fig:noise_decrease} confirms this: MIS grows beyond its bound. Shannon and Bayesian Surprises shows similar behaviors to the prior scenarios.

\begin{figure}[h]
    \centering
    \includegraphics[width=0.45\linewidth]{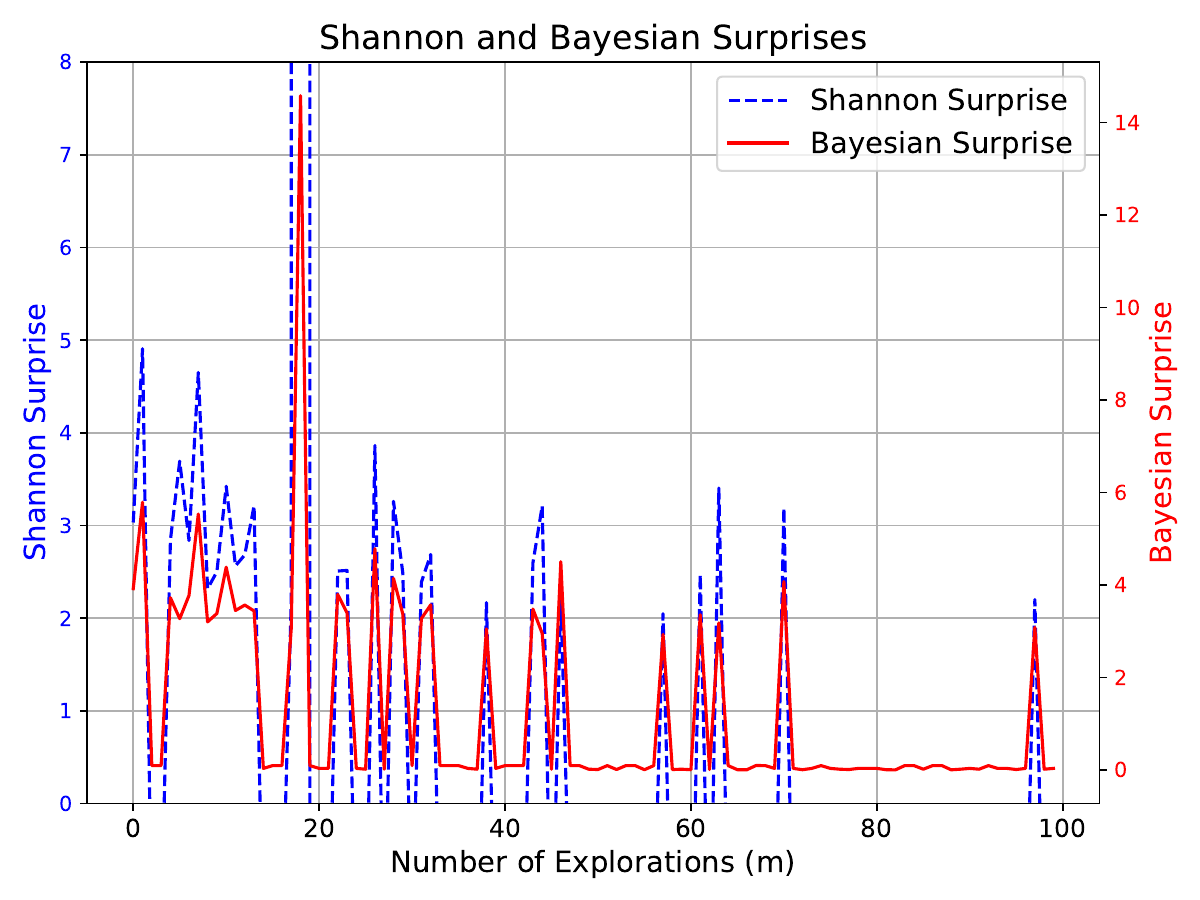}
    \quad
    \includegraphics[width=0.45\linewidth]{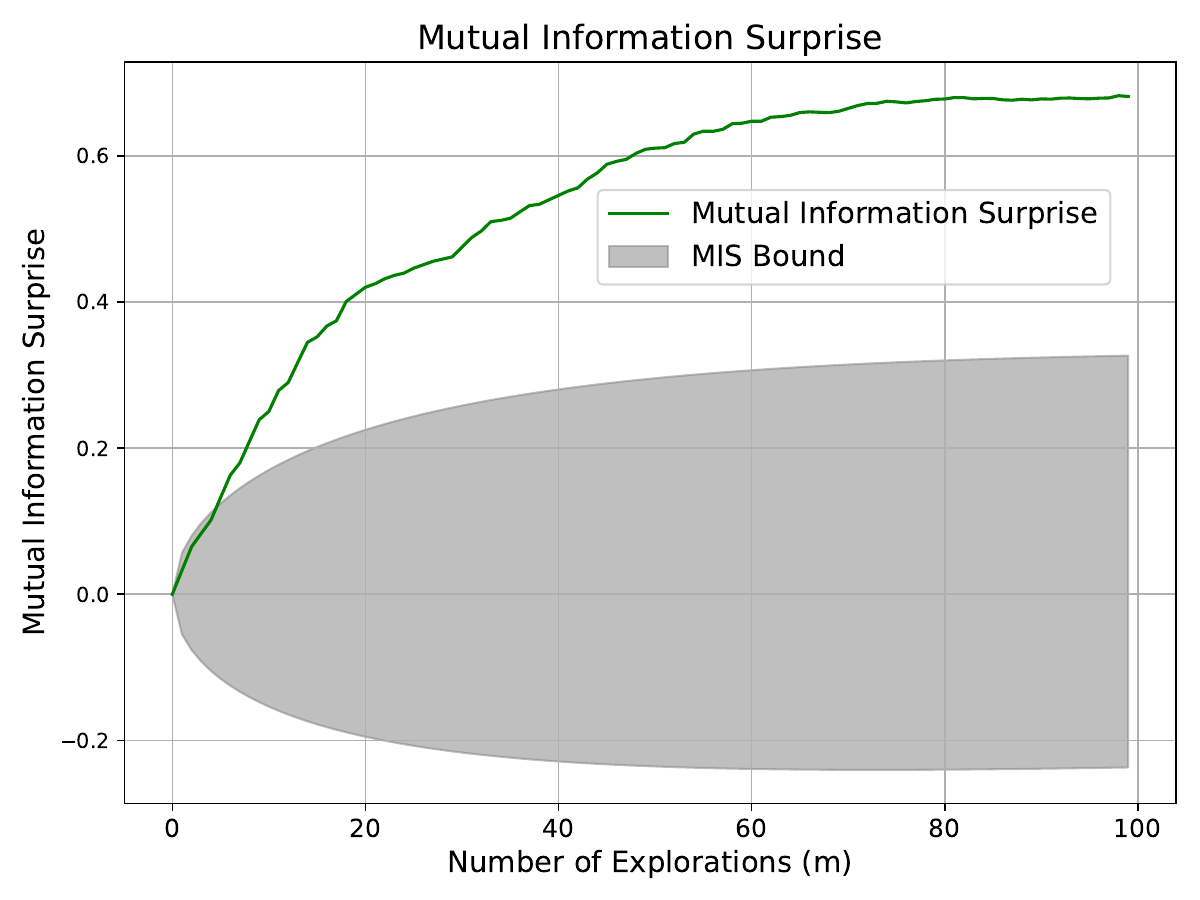}
    \includegraphics[width=0.45\linewidth]{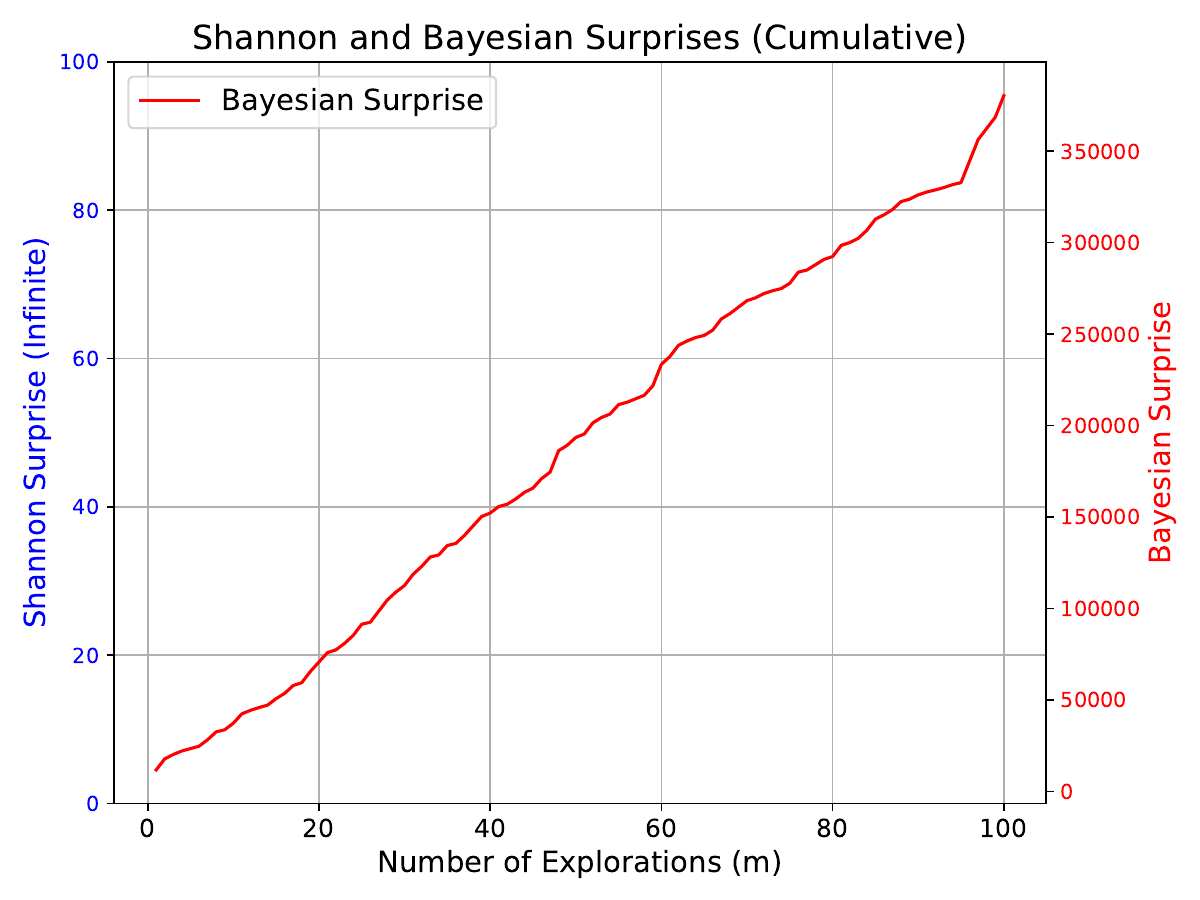}
    \quad
    \includegraphics[width=0.45\linewidth]{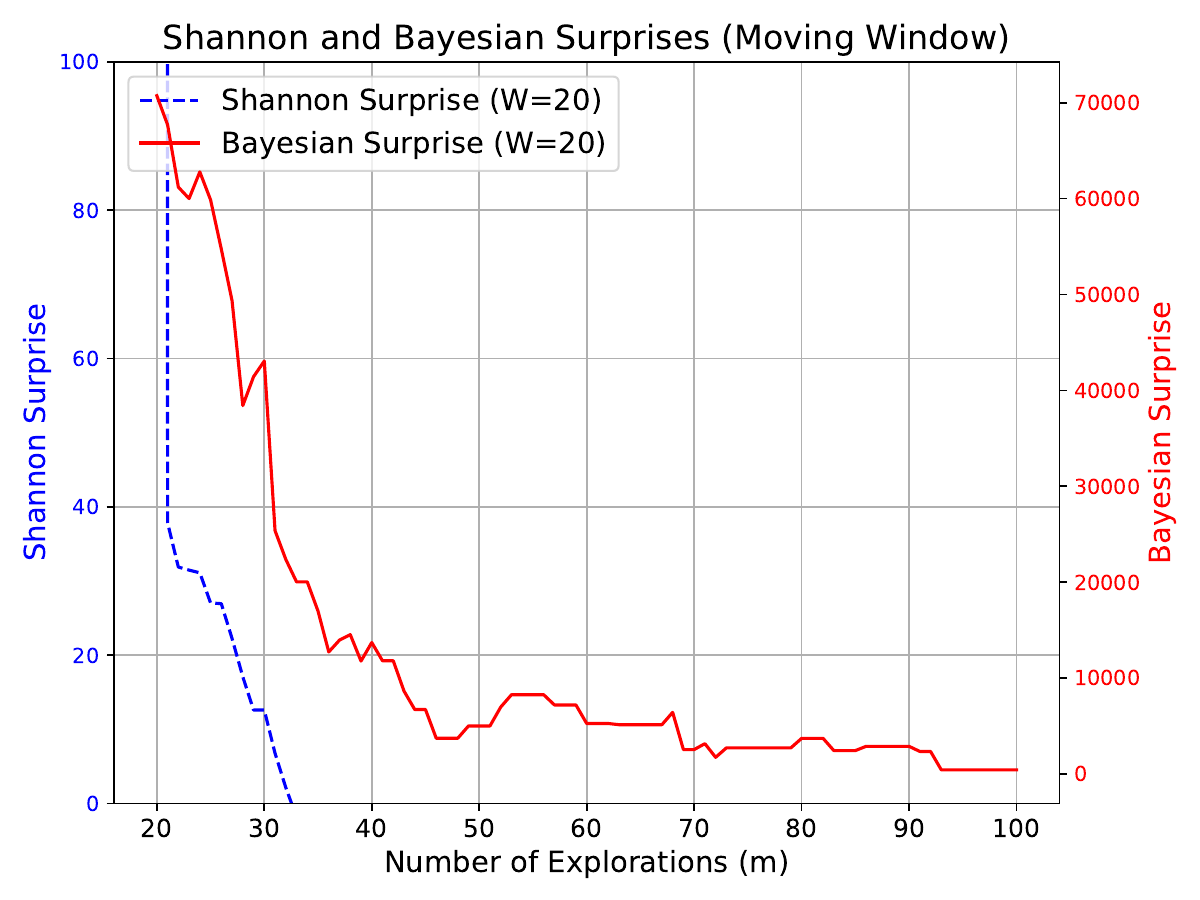}
    \caption{Surprise measures when exploring a new region with novel outputs.}
    \label{fig:explore_new}
\end{figure}

\vspace{1em}
\noindent\textbf{Scenario 6: Discovery of New Output Values.} \noindent We modify the function in the unexplored region ($x > 30$) to $y = x \mod 10 - 10$, introducing a different behavior while keeping the original function unchanged in $[0,30]$.

\textit{Expected behavior:} A competent surprise measure should register this new structure as a meaningful discovery. This mirrors the novel discovery case in Section~\ref{subsec:implications}, and we expect MIS to exceed its upper bound.

Figure~\ref{fig:explore_new} shows MIS sharply exceeding its expected trajectory, signaling successful identification of a structural shift. Shannon and Bayesian Surprises again fail to provide consistent or interpretable responses.

\begin{table}[h]
    \centering
            \caption{The perspective differences among Shannon family surprises, Bayesian family surprises, and Mutual Information Surprise.}    \label{tab:perspectives}
    \resizebox{\textwidth}{!}{%
    \begin{tabular}{|c|c|c|c|c|}
 
         \hline
        Surprise & Single Instance Focused & Capture Transient Changes & Aware of Learning Progression & Parametric Predictive Modeling \\
         \hline
        Shannon Family & \cmark & \xmark & \xmark & \xmark \\
         \hline
        Bayesian Family & \cmark & \xmark & \xmark & \cmark \\
         \hline
        MIS & \xmark & \cmark & \cmark & \xmark \\
         \hline
    \end{tabular}%
    }
\end{table}

\vspace{1em}
\noindent\textbf{Summary} 

\noindent Across all scenarios, MIS reliably indicates whether the system is genuinely learning, stagnating, or encountering degradation. It responds to the structure and value of observations, more than just novelty. In contrast, single-instance Shannon and Bayesian Surprises often react to superficial fluctuations and display numerical instability, and multi-instance Shannon and Bayesian Surprises often exhibits simple and even contradictory (cumulative versus rolling) monotone behaviors. Furthermore, the MIS progression bound remains consistent and interpretable across all scenarios, while Shannon and Bayesian Surprises lack a universal scale or threshold, as reflected by their inconsistent magnitudes across Figures~\ref{fig:explore_std} through~\ref{fig:explore_new}. This inconsistency limits their effectiveness as a reliable trigger. Overall, this simulation study demonstrates MIS not only as a novel metric for quantifying surprise, but also as a more trustworthy indicator of learning dynamics—making it a promising tool for autonomous system monitoring.

Table \ref{tab:perspectives} summarizes the differences between Mutual Information Surprise vs Shannon and Bayesian family of surprises.

\subsection{Pollution Estimation: A Case Study}\label{subsec:real}

To demonstrate the practical utility of our proposed MIS reaction policy, we apply it to a real-time pollution map estimation scenario. We evaluate the impact of integrating the MIS reaction policy on system performance in a dynamic, non-stationary environment. Specifically, we compare two approaches: a selection of baseline sampling strategies and the same strategies \textit{governed} by our MIS reaction policy. 

\subsection*{Dataset: Dynamic Pollution Maps}

We utilize a synthetic pollution simulation dataset comprising $450$ time frames, each representing a $50 \times 50$ pollution grid. Initially, the environment contains $3$ pollution sources, each emitting high pollution at a fixed level. The rest of the field exhibits moderate and random pollution values. Over time, the pollution levels across the entire field evolve due to natural diffusion, decay, and wind effects. Moreover, every $50$ frames, a new pollution source is added to the field at a random location. These new sources elevate the overall pollution levels and alter the input-output relationship between the spatial coordinates and the pollution intensity. Figure~\ref{fig:pollutionmaps} displays a snippet of the pollution map at two intermediate time points. The simulation details for the dynamic pollution map generation are provided in the Appendix.

\begin{figure}
    \centering
    \includegraphics[width=0.7\linewidth]{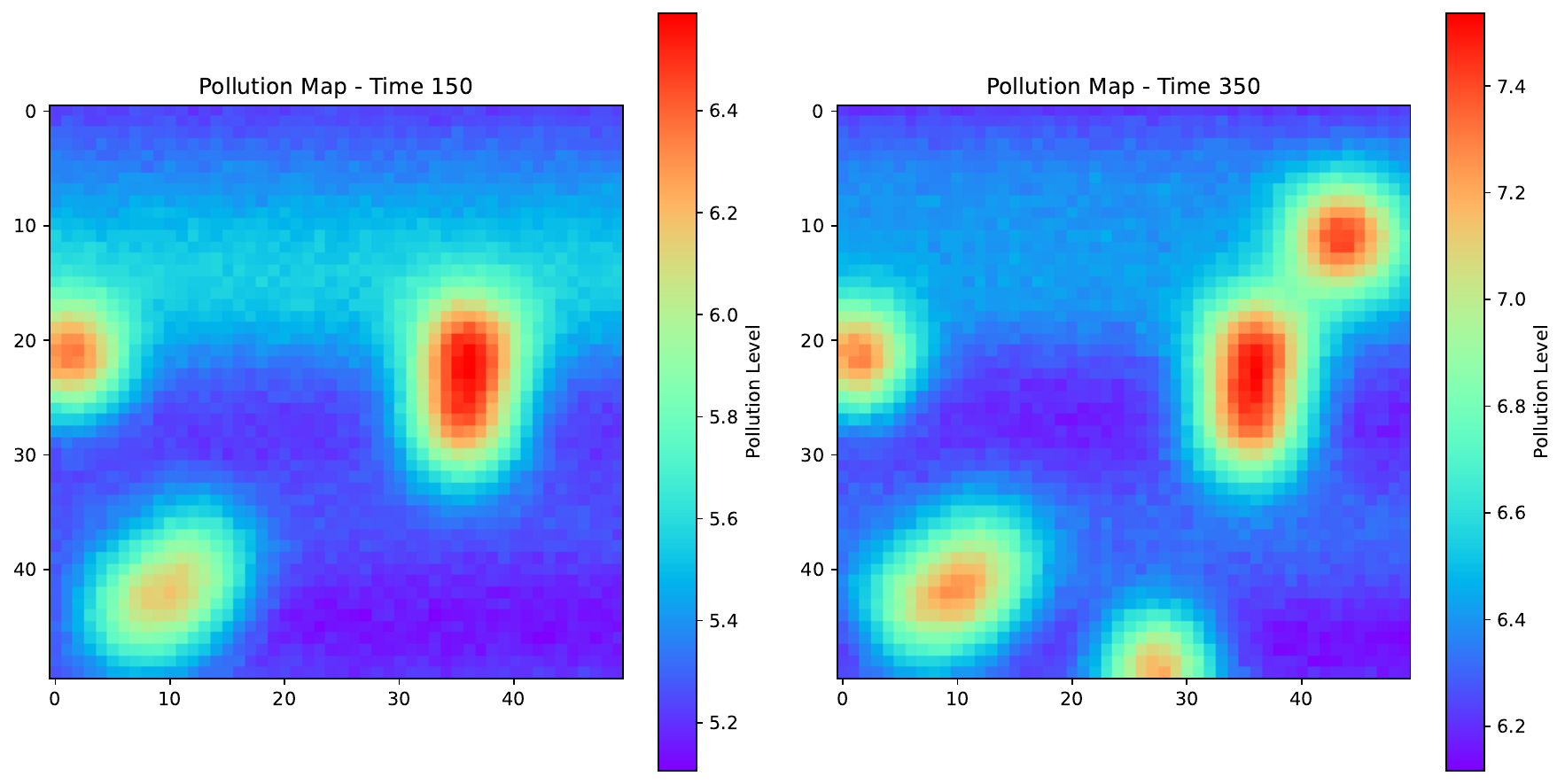}
    \caption{Pollution maps at time $150$ and time $350$.}
    \label{fig:pollutionmaps}
\end{figure}

\subsection*{Sampling Strategies}

As discussed in Section~\ref{subsec:react}, the MIS reaction policy (MISRP) is designed as a data-stream management mechanism that complements existing exploration-exploitation strategies. To evaluate its effectiveness, we conduct simulations using three well-established sampling strategies: the surprise-reactive (SR) sampling method proposed by \cite{ahmed2024toward}, implemented with either Shannon or Bayesian Surprises; the subtractive clustering/entropy (SC/E) active learning strategy introduced by \cite{cebron2009active}; and the greedy search/query-by-committee (GS/QBC) active learning strategy used in \cite{islam2025dynamic}. In addition, input-output data stream management has long been studied in the concept drift and anomaly detection literature. Accordingly, we compare MISRP against one of the most widely used cumulative drift detection metrics, the cumulative Shannon Surprise (CSS) defined in Eq.~\eqref{eq:CSS}. Following standard practices in concept drift detection, once CSS exceeds a predefined threshold ($\text{CSS}>2.3$), corresponding to the product likelihood falling below $0.1$, previously collected data are discarded. Note that we do not benchmark the cumulative Bayesian Surprise because unlike Shannon Surprise, the KL divergence nature of Bayesian Surprise makes its triggering threshold ill-defined, prohibiting the design of a reaction mechanism.

\begin{enumerate}
    \item {\bf SR}: The surprise-reactive sampling method \cite{ahmed2024toward} switches between exploration and exploitation modes based on observed Shannon or Bayesian Surprise. By default, SR operates in an exploration mode guided by the widely used space-filling principle \cite{joseph2016space}, selecting new sampling locations via the min-max objective:
    \[
        \xb^* = \underset{\xb}{\argmax} \: \underset{\xb_i \in \Xb}{\min} \: \|\xb - \xb_i\|_2,
    \]
    where $\Xb$ denotes the set of existing observations. Upon encountering a surprising event (in terms of either Shannon or Bayesian Surprise), SR switches to exploitation mode, performing localized verification sampling within the neighborhood of the surprise-triggering location. This continues either for a fixed number of steps defined by an exploitation limit $t$, or until an unsurprising event occurs. If exploitation confirms that the surprise is consistent (i.e., persistent surprise until reaching the exploitation threshold), all corresponding observations are accepted and incorporated into the pollution map estimation. Conversely, if an unsurprising event arises before the threshold is reached, the surprising observations are deemed anomalous and discarded. For Shannon Surprise, we set the triggering threshold at $1.3$, corresponding to a likelihood of $5\%$. For Bayesian Surprise, we use the Postdictive Surprise and adopt the threshold of $0.5$, following \cite{ahmed2024toward}.

    {\bf MISRP}: The MISRP modifies SR by dynamically adjusting the exploitation limit $t$. When increased exploitation is needed, $t$ is incremented by $1$. For increased exploration, $t$ is decremented by $1$, with a lower bound of $t = 1$.

    \item {\bf SC/E}: The subtractive clustering/entropy active learning strategy \cite{cebron2009active} selects the next sampling location by maximizing a custom acquisition function. For an unseen region $\Xc$ and a probabilistic predictive function $\hat{f}(\xb)$ trained on the observed data, the acquisition function is defined as:
    \begin{equation*}
        a(\xb) = (1-\eta)\E_{\xb' \in \Xc} [e^{-\|\xb - \xb'\|_2}] + \eta H(\hat{f}(\xb)),
    \end{equation*}
    where $\eta$ is the exploitation parameter, with a default value of $0.5$, and $H(\hat{f}(\xb))$ denotes the entropy of the predictive distribution at $\xb$. A larger value of $\eta$ emphasizes sampling at locations with high predictive uncertainty near previously seen points, promoting exploitation. A smaller value favors sampling at representative locations in the unseen region, promoting exploration \cite{cebron2009active}.

    {\bf MISRP}: The MISRP modifies SC/E by adjusting the exploitation parameter $\eta$. For increased exploitation, $\eta$ is increased by $0.1$, up to a maximum of $1$. For increased exploration, $\eta$ is decreased by $0.1$, with a minimum of $0$.

    \item {\bf GS/QBC}: The greedy search/query by committee active learning strategy \cite{islam2025dynamic} uses a different acquisition function. Given the set of seen observations $\{\Xb, \Yb\}$ and a model committee $\Fc$ composed of multiple predictive models trained on this data, the acquisition function is defined as:
    \begin{equation}\label{eq:qbc}
        a(\xb) = (1-\eta)\underset{\xb', \yb' \in \Xb, \yb}{\min} \|\xb - \xb'\|_2\|\hat{f}(\xb) - \yb'\|_2 + \eta \underset{\hat{f}(\cdot), \hat{f}'(\cdot) \in \Fc}{\max} \|\hat{f}(\xb) - \hat{f}'(\xb)\|_2,
    \end{equation}
    where the first term encourages exploration by selecting points that are distant from existing observations in both input and output space. The second term promotes exploitation by targeting locations with high disagreement among models in the committee.

    {\bf MISRP}: The MISRP regulates the balance between exploration and exploitation in GS/QBC in the same manner as in SC/E, by adjusting the parameter $\eta$.
\end{enumerate}

\subsection*{Experimental Setup}

The estimation process is initialized with $10$ observed locations uniformly sampled across the pollution field. Each time frame collects $10$ new samples according to the chosen sampling strategy, representing the operation of $10$ mobile pollution sensors. The pollution field is estimated using a Gaussian Process Regressor with a Matérn kernel ($\nu = 2.5$) and a noise prior of $10^{-2}$, consistently applied across all strategies. The model predicts pollution levels at specified spatial locations and is updated using both current and historical data, with a maximum of $200$ observations retained to reduce computational cost.

For the GS/QBC strategy, the model committee additionally includes regressors with a Matérn $\nu=1.5$ kernel and a Gaussian kernel with bandwidth $0.1$, both using a noise prior of $10^{-2}$. These two additional models are used solely for calculating disagreement in Eq.~\eqref{eq:qbc} and are not employed in pollution map estimation.

Shannon and Bayesian Surprise are computed following the procedure described in Section \ref{subsec:sythetic}. For MIS calculations, we discretize the range of pollution values observed in the data into $100$ bins to estimate entropy, and we set the triggering probability at $\rho=0.1$.

In process forking scenarios, two separate pollution map estimates, $\hat{f}_m$ and $\hat{f}_n$, are produced for subprocesses $\Pc_m$ and $\Pc_n$, respectively. The final pollution map estimate is formed as a weighted combination:
\[
    \hat{f} = \frac{\sqrt{m}}{\sqrt{m} + \sqrt{n}}\hat{f}_m + \frac{\sqrt{n}}{\sqrt{m} + \sqrt{n}}\hat{f}_n,
\]
accounting for generalization errors that scale as $\Oc(\frac{1}{\sqrt{m}})$ and $\Oc(\frac{1}{\sqrt{n}})$, respectively \cite{chai2009generalization}.

\subsection*{Simulation Results}

We assess performance using the mean squared error (MSE) between predicted and true pollution maps at each time step. Due to the dynamic nature of the pollution field, estimation errors exhibit substantial fluctuation. To smooth these variations, we compute a 20-frame moving average of the MSE for both vanilla and MISRP-governed strategies. The results are shown in Figure~\ref{fig:error_comaprison}.

\begin{figure}
    \centering
    \includegraphics[width=0.45\linewidth]{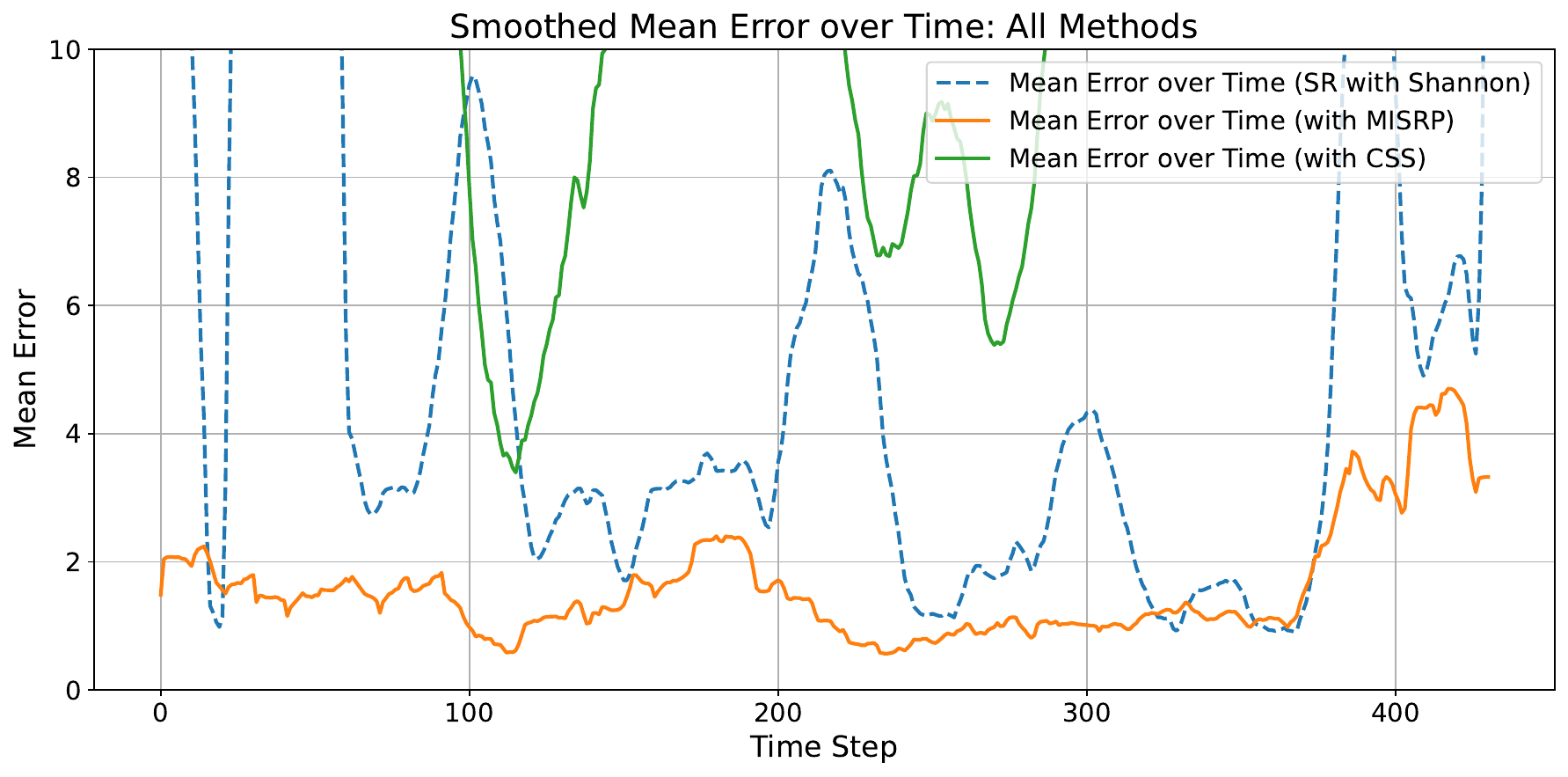}
    \includegraphics[width=0.45\linewidth]{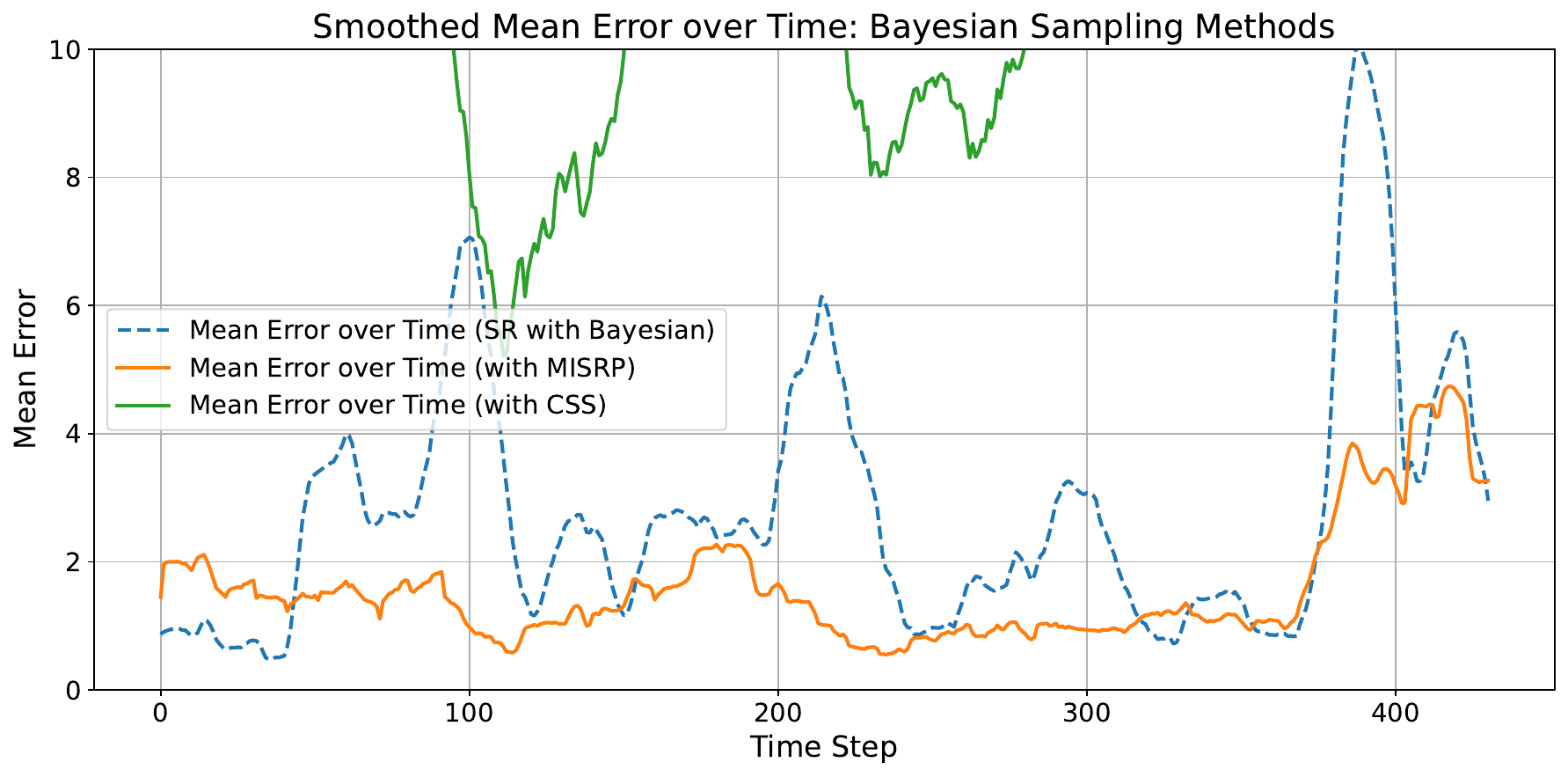}
    \includegraphics[width=0.45\linewidth]{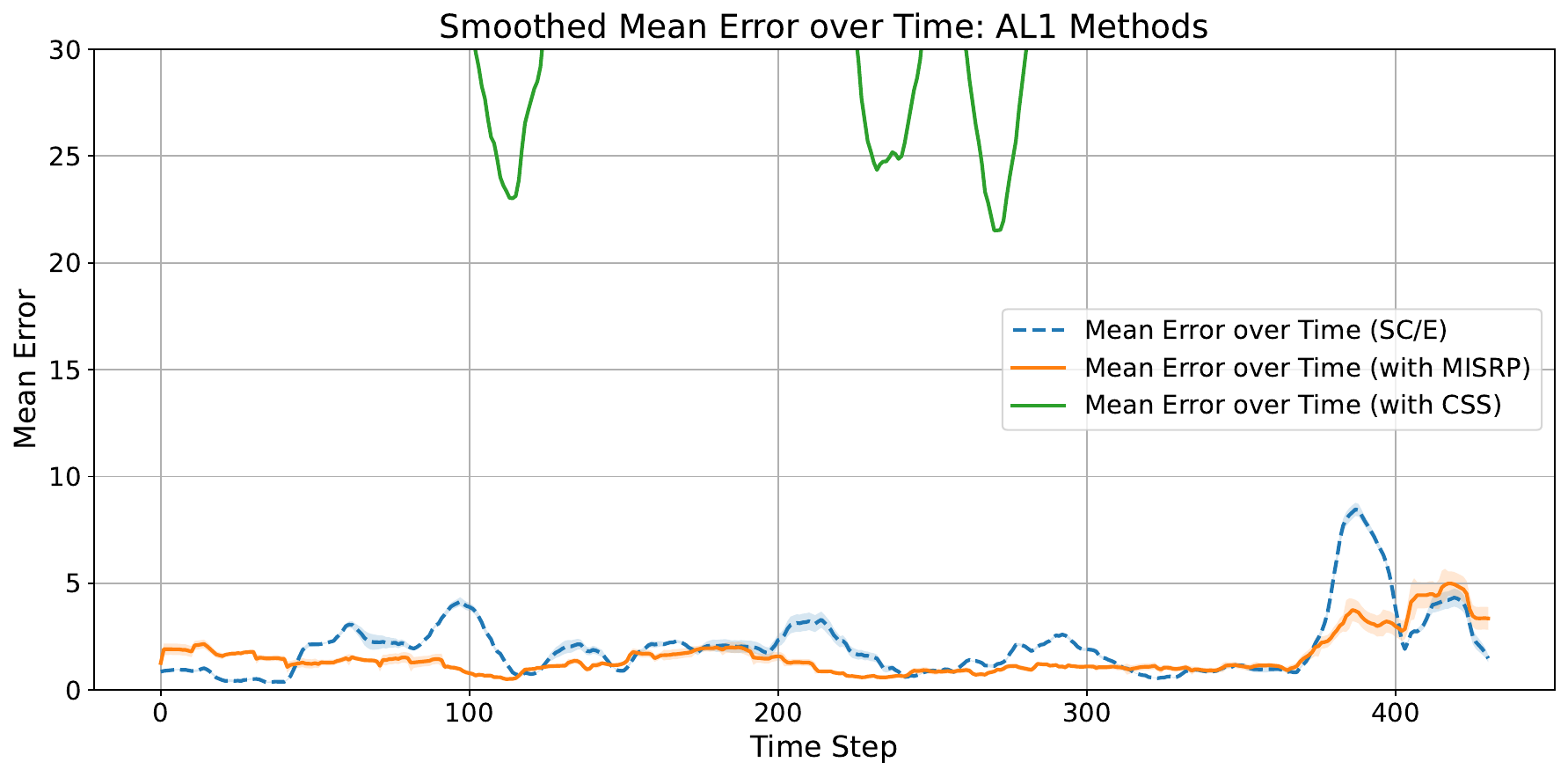}
    \includegraphics[width=0.45\linewidth]{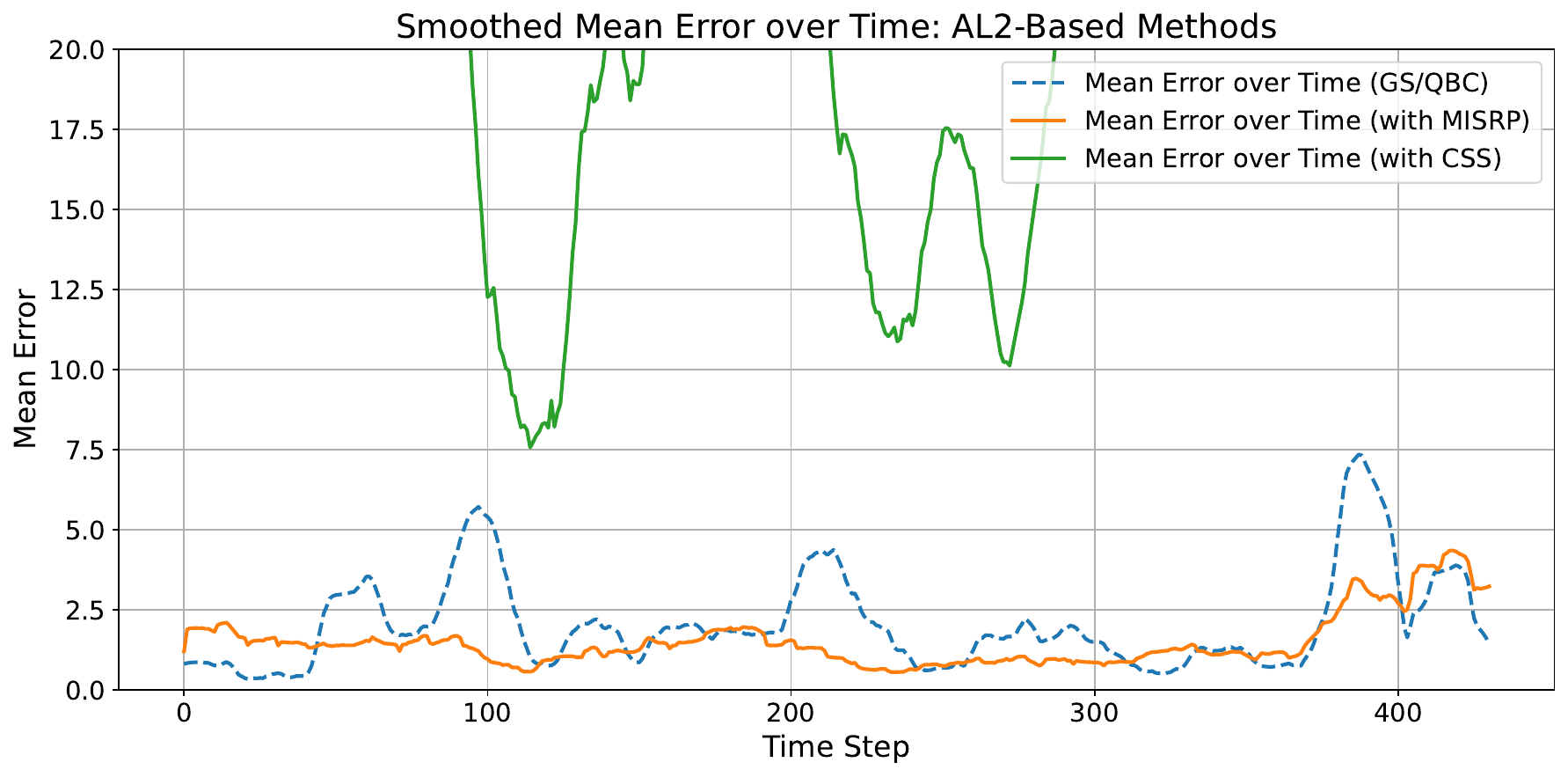}
    \caption{Moving average estimation error over time. Top-Left: SR with Shannon Surprise. Top-Right: SR with Bayesian Surprise. Bottom-Left: SC/E. Bottom-Right: GS/QBC.}
    \label{fig:error_comaprison}
\end{figure}

Across all comparisons, the baseline strategies display considerable volatility. In contrast, MISRP-governed counterparts produce smoother and consistently lower error curves, highlighting the stabilizing effect of MIS through its ability to facilitate adaptive responses in dynamic environments. CSS on the other hand, is constantly triggered, leaving very few observations for estimation, thus resulting in much higher estimation error and even higher volatility than the baseline strategies.

Table~\ref{tab:error} presents the average estimation errors and their corresponding standard errors. The standard error is measured across $10$ Monte Carlo simulations and $450$ frames. Across all sampling strategies, incorporating the MIS reaction policy yields a substantial reduction in both mean estimation error and variability. Improvements in estimation error (compared to baseline strategies) range from $24\%$ to $76\%$, while reductions in standard error range from $36\%$ to $90\%$. On the other hand, CSS shows considerable performance degradation even compared to baseline strategies, highlighting the potential negative impact of traditional concept drift detection policy in this inherently dynamic process. The quantitative results in Table~\ref{tab:error} highlight the impact of MISRP on learning performance, demonstrating substantial improvements in estimation accuracy and stability through data stream control.

To further illustrate the advantage of MISRP, we increase the per-frame sampling budget and the initial number of observed locations of the baseline strategies from $10$ to $25$, and expand the total memory buffer from $200$ to $500$, in order to assess whether baseline strategies can match the performance of MISRP-governed approaches. Table~\ref{tab:time} compares the estimation error of MISRP-governed strategies (maintaining the original sampling budget of $10$) against the enhanced baseline strategies. Even with a $2.5\times$ increase in sampling budget, the baseline strategies remain significantly outperformed by their MISRP-governed counterparts.

\begin{table}[h]
    \centering
    \caption{Comparison of pollution map estimation errors: baseline versus MISRP versus CSS.}     
    \resizebox{\textwidth}{!}{
   \begin{tabular}{|c|c|c|c|c|c|}
        \hline
        Sampling Strategy & \makecell{Estimation Error\\(Baseline)} & \makecell{Estimation Error\\(CSS)} & \makecell{Estimation Error\\(MISRP)} & \makecell{Mean Improvement\\(Over Baseline)} & \makecell{Std. Error Improvement\\(Over Baseline)} \\
        \hline
        SR with Shannon  & $6.64 \pm 0.436$ & $19.98 \pm 0.645$ & $\mathbf{1.60 \pm 0.043}$ & $76\%$ & $90\%$ \\
        SR with Bayesian & $2.79 \pm 0.096$ & $14.78 \pm 0.374$& $\mathbf{0.87 \pm 0.016}$ & $69\%$ & $83\%$ \\
        SC/E & $2.02 \pm 0.071$ & $75.63 \pm 2.371$& $\mathbf{1.53 \pm 0.045}$ & $24\%$ & $36\%$ \\
        GS/QBC & $2.07 \pm 0.071$ & $35.67 \pm 1.205$& $\mathbf{1.49 \pm 0.039}$ & $28\%$ & $45\%$ \\
        \hline
    \end{tabular} }\label{tab:error}
    \end{table}

\begin{table}[h]
    \centering
    \caption{Error Comparison under Extended Sampling for Baseline Strategies.}
        \resizebox{\textwidth}{!}{
    \begin{tabular}{|c|c|c|}
        \hline
        Sampling Strategy & Estimation Error (MISRP-Governed, Budget 10) & Estimation Error (Baseline, Budget $25$) \\
        \hline
        SR with Shannon  & $\mathbf{1.60}$ & $6.23$ \\
        SR with Bayesian & $\mathbf{0.87}$ & $2.72$ \\
        SC/E & $\mathbf{1.53}$ & $1.89$ \\
        GS/QBC & $\mathbf{1.49}$ & $2.00$ \\
        \hline
    \end{tabular} \label{tab:time}
    }
\end{table}

So far we demonstrated that governing basic sampling strategies with MISRP can substantially enhance learning performance in dynamic environments. To provide a clearer view of how MISRP operates over time, we conduct an additional simulation examining its actions throughout the process.

In this experiment, we simulate a two-phase pollution map evolution governed by the same PDE used in earlier simulation. During the first phase (time $0$–$250$), three pollution sources emit high levels of pollutants, and the map evolves under diffusion, decay, and wind effects. At time step $250$, the emission sources are removed, and the decay factor is reduced to one-twentieth of its original value. The system then continues evolving for an additional $50$ steps. 

When the pollution sources exists and are emitting (the dynamic phase, time $0$–$250$), the underlying process is a non-stationary process in which we expect frequent MIS triggering. When the pollution sources are gone (the stationary phase, time $251$-$300$), the pollutants in the area will eventually diffuse to a stationary existence, during which time MIS is expected to stop being triggered. 

\begin{figure}
    \centering
    \includegraphics[width=0.8\linewidth]{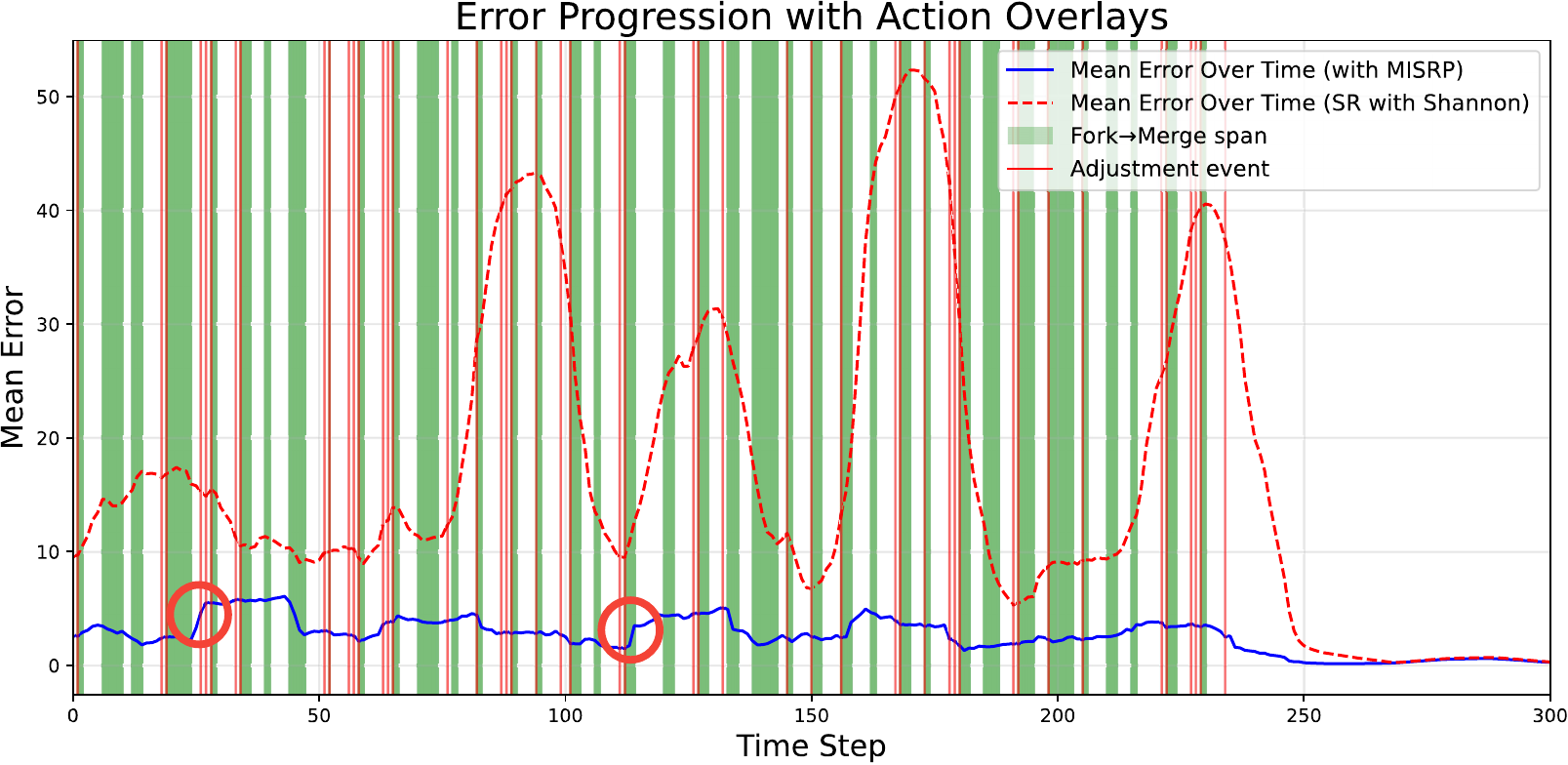}
    \caption{A visualization of estimation error progression with MISRP action overlays. Vertical lines represent sampling adjustments, and vertical shaded regions represent the span between process forking and subsequent process merging.}
    \label{fig:action}
\end{figure}

Figure~\ref{fig:action} shows the estimation error progression with action overlays under surprise-reactive sampling based on Shannon surprise.  Recall from Section~\ref{subsec:react} that there are two actions employed in MISRP governance: sampling adjustments and process forking. These two actions are marked as vertical lines and shaded regions in the plot, respectively. For clarity, we present the $20$-frame moving average of estimation error, whereas the non-smooth version is provided in the Appendix. Actions are displayed $20$ steps in advance, corresponding to their first observable effect on the smoothed error trajectory.

Several key observations emerge from the figure. First, both sampling adjustments and process forking occur frequently during the dynamic phase as expected, highlighting the effectiveness of MISRP’s action design in maintaining low estimation error. Second, sudden spikes in estimation error (circled) under MISRP governance are almost always followed by corrective actions that prevent further error growth, resulting in non-smooth error progressions after intervention. By contrast, the baseline sampling strategy allows estimation error to rise unchecked. Then, once the system enters the stationary phase, MISRP ceases intervention, aligning with the intuition that a balanced sampling strategy in a \textit{well-regulated} system should not trigger Mutual Information Surprise. 

\section{Conclusion}

We started the paper by presenting a vivid picture of the recent race towards autonomous systems. We argued that while the bodies of various autonomous system have advanced visibly, even to the general public, over the past decade, the brains of these autonomous systems still miss a critical element, which is how to define and react to surprises. Unlike classical surprise measures that merely characterize statistical irregularities, we re-imagine the concept of surprise as a mechanism for fostering understanding. We further define a new surprise metric based on \textit{mutual information} and re-frame surprise as a reflection of learning progression grounded in mutual information growth.

We developed a formal test sequence to monitor deviations in the estimated mutual information, and introduced a reaction policy, MISRP, that transforms surprise into actionable system behavior. Through a synthetic case study and a pollution map estimation task, we demonstrated that MIS governance offers clear advantages over conventional sampling strategies. Our results show improved stability, better responsiveness to environmental drift, and significant reductions in estimation error. These findings affirm MIS as a robust and adaptive supervisory signal for autonomous systems.

Looking forward, this work opens several promising directions for future research. A natural next step is the development of a \textit{continuous} space formulation of mutual information surprise, enabling its application in large complex systems. Another direction involves designing a \textit{specialized reaction policy}—one that incorporates a sampling strategy tailored directly to the structure and signals of MIS, rather than relying on existing sampling strategies. This could enhance efficiency and responsiveness in highly dynamic or resource-constrained systems. Moreover, pairing MIS with physical probing capability for specific physical systems could unlock the true potential of MIS, as MIS provides new perspectives in system characterization compared to traditional measures. 

\newpage

\section*{Appendix}

The appendix is organized as follows. In the first section, we present empirical evidence supporting our claim in Section \ref{subsec:bound} that standard deviation-based tests are overly permissive. In the second section, we provide the derivation of the standard deviation-based test for mutual information. In the third section, we provide the proof of Theorem \ref{the:mis}. The fourth section details the simulation setup for dynamic pollution map generation. In the fifth section, we provide the pseudocode for the surprise-reactive (SR) sampling strategy \cite{ahmed2024toward} to facilitate reproducibility.

\section*{MLE Mutual Information Estimator Standard Deviation}

In Section \ref{subsec:bound}, we discussed the limitations of standard deviation-based tests. Specifically, the current distribution agnostic tightest bound for the standard deviation of a maximum likelihood estimator (MLE) for mutual information with $n$ observations is given by \cite{paninski2003estimation}
\begin{equation*}
    \sigma \lesssim \frac{\log n}{\sqrt{n}}.
\end{equation*}
Despite the best result, this bound is still too loose.

To empirically verify this statement, we perform a simple simulation as follows. We construct variable pairs $(x, y)$ where $y = x \; \text{mod} \; 10$, in the same manner as the simulation in Section \ref{subsec:sythetic}. The variable $x$ is generated as random integers sampled from randomly generated probability mass functions over the domain $[0, 100]$. We generate $100$ such probability mass functions. For each probability mass function, we generate $3,000$ pairs of $(x, y)$, repeat the process using $10$ Monte Carlo simulations, and compute the standard deviation of the MLE mutual information estimates over the $10$ simulations for varying numbers of $(x, y)$ pairs $n$. We then plot the average standard deviation across the $100$ different probability mass functions as a function of $n$ versus the estimation bound shown in Eq.~(\ref{eq:I_std}). The results are shown in Figure \ref{fig:Istd}.

\begin{figure}[h]
    \centering
    \includegraphics[width=0.5\linewidth]{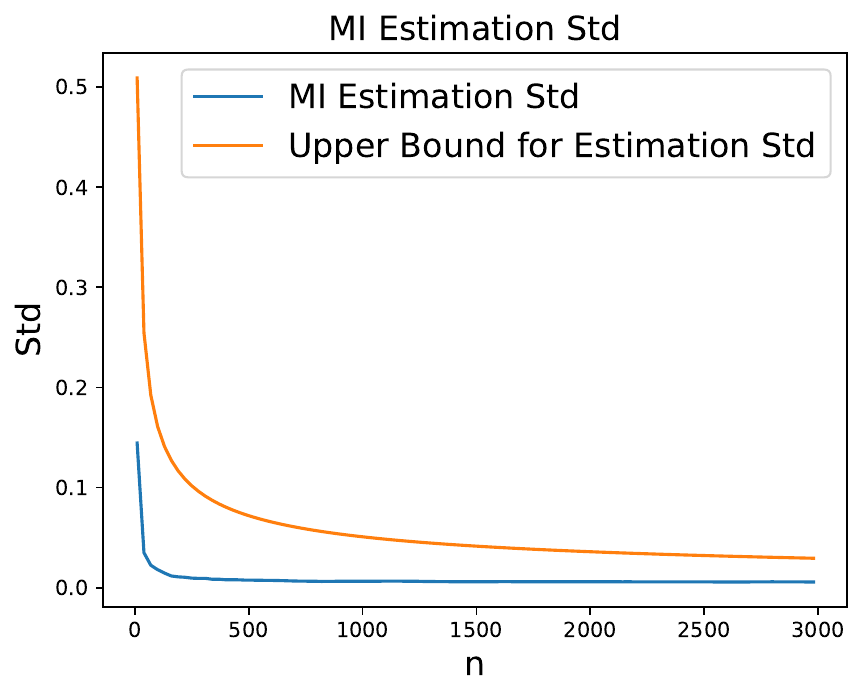}
    \caption{Empirical standard deviation of MLE mutual information estimates vs. the current tightest bound.}
    \label{fig:Istd}
\end{figure}

We observe that the current bound for the standard deviation of the mutual information estimate, computed using Eq.~\eqref{eq:I_std}, is significantly larger than the empirical average standard deviation. This empirical observation supports our claim in Section \ref{subsec:bound} that the test in Eq.~\eqref{eq:vartest} is rarely violated in practice.

\section*{Standard Deviation Test Derivation}

First, recall that the estimation standard deviation satisfies
\begin{equation*}
    \sigma \lesssim \frac{\log n}{\sqrt{n}}.
\end{equation*}
Therefore, we treat this worst case scenario as the baseline when deriving the test of difference between the two maximum likelihood estimators (MLE) of mutual information.

Let:
\begin{itemize}
    \item \( \hat{I}_n \) be the MLE estimate from a sample of size \(n\),
    \item \( \hat{I}_{m+n} \) be the MLE estimate from a larger sample of size \(m+n\),
\end{itemize}

Assume the standard deviation of the MLE estimator is approximately:
\[
\sigma_n = \frac{\log n}{\sqrt{n}}, \quad \sigma_{m+n} = \frac{\log(m+n)}{\sqrt{m+n}}
\]

We want to test the hypothesis:
\[
H_0: \mathbb{E}[\hat{I}_n] = \mathbb{E}[\hat{I}_{m+n}]
\quad \text{vs.} \quad H_1: \mathbb{E}[\hat{I}_n] \neq \mathbb{E}[\hat{I}_{m+n}]
\]
Note that we are omitting the estimation bias of MLE mutual information estimators for simplicity.

Under the null hypothesis and assuming the two estimates are independent, the test statistic is:

\[
z_{\alpha} = \frac{\hat{I}_n - \hat{I}_{m+n}}{\sqrt{\sigma_n^2 + \sigma_{m+n}^2}}
= \frac{\hat{I}_n - \hat{I}_{m+n}}{\sqrt{\left( \frac{\log n}{\sqrt{n}} \right)^2 + \left( \frac{\log(m+n)}{\sqrt{m+n}} \right)^2}}
\]

Moving the denominator to the left hand side will yield the form presented in Eq.~\eqref{eq:vartest}.

\section*{Proof of Theorem 1}

First, we formally introduce the maximum likelihood entropy estimator $\hat{H}$ \cite{strong1998entropy} for random variable $\xb \in \Xc$ as follows
\begin{equation*}
    \hat{H}(\xb) = \sum_{i=1}^{|\Xc|} \hat{p}_i \log \hat{p}_i,
\end{equation*}
where $\hat{p}_i$ is the empirical probability mass of random variable $\xb$ at category $i$. The MLE mutual information estimator is then defined based on the MLE entropy estimator
\begin{equation*}
    \hat{I}(\xb,\yb) = \hat{H}(\xb) + \hat{H}(\yb) - \hat{H}(\xb, \yb).
\end{equation*}

\noindent{\bf MIS test bound (Expectation):}

Here, we derive the first part of the MIS test bound, representing the expectation of the MIS statistics, i.e., $\E[\text{MIS}]$. The derivation involves two cases, $n \ll |\Xc|,|\Yc|$ and $n \gg |\Xc|,|\Yc|$.

When $n \ll |\Xc|,|\Yc|$, an MLE entropy estimator $\hat{H}$ with $n$ observations will behave simply as $\log n$ \cite{paninski2003estimation}, conditioning on the $n$ observations are selected using some kind of space filling designs, which is common for design the initial set of experimentation locations in design of experiments literature \cite{joseph2016space}. We have $\E[\hat{H}_n(\xb)] = \log n$. Hence, the mutual information estimator with $n$ observations admits
\begin{equation*}
    \E[\hat{I}_n(\xb,\yb)] = \E[\hat{H}_n(\xb) + \hat{H}_n(\yb) - \hat{H}_n(\xb,\yb)] = \log n.
\end{equation*}
Then for MIS, we have
\begin{equation*}
    \E[\text{MIS}] = \E[\hat{I}_{m+n}] - \E[\hat{I}_{n}] = \log(m+n) - \log n.
\end{equation*}

When $n \gg |\Xc|,|\Yc|$, we are facing an oversampled scenario where the samples have most likely exhausted the input and output space. In this case, we first introduce the following lemma.

\begin{lemma}\label{the:bias}
    \cite{paninski2003estimation} For a random variable $\xb \in \Xc$, the bias of an oversampled ($n \gg |\Xc|$) MLE entropy estimator $\hat{H}_n(\xb)$ is
    \begin{equation}
        \E[\hat{H}_n(\xb)] - H(\xb) = -\frac{|\Xc|-1}{n} + o(\frac{1}{n}).
    \end{equation}
\end{lemma}
With the above lemma, we can derive the following Corollary.
\begin{corollary}
    For random variable $\xb \in \Xc$ and $\yb \in \Yc$, when the $\yb = f(\xb)$ mapping is noise free, the MLE mutual information estimator $\hat{I}_n$ asymptotically satisfies
    \begin{equation*}
        \E[\hat{I}_n] = I - \frac{|\Yc|-1}{n}.
    \end{equation*}
\end{corollary}
The proof of the above Corollary immediately follows observing the fact of $|\Xc| = |\Xc,\Yc|$ for noise free mapping and invoking Lemma \ref{the:bias}.

Therefore, for MIS under the case of oversampling, we have
\begin{equation*}
\begin{aligned}
    \E[\text{MIS}] &= \E[\hat{I}_{m+n}] - \E[\hat{I}_{n}]\\
    &= I - \frac{|\Yc|-1}{m+n} - I + \frac{|\Yc|-1}{n}\\
    &= \frac{|\Yc|-1}{n} - \frac{|\Yc|-1}{m+n}.
\end{aligned}
\end{equation*}

\noindent{\bf MIS test bound (Variation)}:

In this part, we derive the second term of the MIS test bound, accounting for the variation of the MIS statistics. We first investigate the maximum change in mutual information estimation $\hat{I}$ when changing one observation. Here, we derive the following Lemma.

\begin{lemma}\label{lemma:diff}
Let $\mathcal S=\{(x_i,y_i)\}_{i=1}^{n}$ be an i.i.d.\ sample from an
unknown joint distribution on finite alphabets and denote by
\[
\hat I_n(\xb,\yb)\;=\;\hat H_n(\xb)+\hat H_n(\yb)-\hat H_n(\xb,\yb)
\]
the MLE estimator, where  
$\hat H_n$ is the empirical Shannon entropy
(in nats).  
If $\mathcal S'$ differs from $\mathcal S$ in \emph{exactly one} observation,
then with a mild abuse of notation (denoting mutual information estimator on sample set $\mathcal S$ with $\hat I_n(\mathcal S)$),
\[
\bigl|\hat I_n(\mathcal S)-\hat I_n(\mathcal S')\bigr|
  \;\le\;
  \frac{2\,\log n}{n}.
\]
\end{lemma}

\begin{proof}{Proof For Lemma \ref{lemma:diff}}

We omit ~$\hat{\cdot}$~ for estimators during this proof for simplicity. Write $H=-\sum_{i}p_i\log p_i$ for Shannon entropy estimator with natural
logarithms. Replacing a single observation does two things:

\begin{enumerate}
  \item\label{itm:marginals} in \emph{one} $X$-category and \emph{one}
        $Y$-category the counts change by $\pm1$ (all other marginal
        counts are unchanged);
  \item\label{itm:joint} in \emph{one} joint cell the count changes by $-1$
        and in another joint cell the count changes by $+1$.
\end{enumerate}

\paragraph{Step 1.  How much can \emph{one} empirical Shannon entropy change?}

Assume a single observation is moved from category~$A$ to
category~$B$.  Let the counts \emph{before} the move be
$A=a$ (with $a\ge1$) and $B=b$ (with $b\ge0$).
After the move the counts become $a-1$ and $b+1$.
Only these two probabilities change; every other probability is fixed.

The change in entropy is therefore
\[
\begin{aligned}
\Delta H
 &= \big(\frac{a}{n}\log\frac{a}{n}-\frac{a-1}{n}\log\frac{a-1}{n}\big)
    -\big(\frac{b+1}{n}\log\frac{b+1}{n}
    -\frac{b}{n}\log\frac{b}{n}\big).
\end{aligned}
\]

We can see that the maximum difference is
\emph{largest} when $a=n$ and $b=0$, i.e.\ when all $n$ observations
initially occupy a single category and we create a brand-new one.
In that worst case
\begin{equation}\label{eq:one-entropy}
    \begin{aligned}
        \Delta H &= \frac{n-1}{n} \log \frac{n-1}{n} + \frac{1}{n} \log n\\
        & \leq \frac{n-1}{n} \log \frac{n}{n} + \frac{1}{n} \log n  = \frac{\log n}{n}.
    \end{aligned}
\end{equation}
Conversly, one could see that $-\frac{\log n}{n} \leq \Delta H$ also holds. Therefore, the maximum absolute differences of entropy estimation under the shift of one observations is upper bounded by $\frac{\log n}{n}$.

\paragraph{Step 2.  Sign coupling between the three entropies.}
Assume the moved observation leaves joint cell $(i,j)$ and enters cell
$(k,\ell)$. Because \((i,j)\) lies in row $i$ and column $j$ only, we have the key
fact (denoting sign operator with $\text{sgn}(\cdot)$):

\[
\text{sgn}\bigl(\Delta H(\xb,\yb)\bigr)
   \in\bigl\{\text{sgn}\bigl(\Delta H(\xb)\bigr),
               \text{sgn}\bigl(\Delta H(\yb)\bigr)\bigr\}.
\]

Hence $-\text{sgn}\bigl(\Delta H(\xb,\yb)\bigr) =\text{sgn}\bigl(\Delta H(\xb)\bigr) = \text{sgn}\bigl(\Delta H(\yb)\bigr)$ is impossible.

Then, with $\Delta I=\Delta H(\xb)+\Delta H(\yb)-\Delta H(\xb,\yb)$, we can see the following fact

\[
|\Delta I|
   =\bigl|\Delta H(\xb)+\Delta H(\yb)-
           \Delta H(\xb,\yb)|
   \le 2 \max \{ |\Delta H(\xb)|, |\Delta H(\yb)|, |\Delta H(\xb,\yb)| \}.
\]

Applying the one-entropy bound \eqref{eq:one-entropy} to the two marginals,

\[
|\Delta I|
   \le \frac{2\log n}{n},
\]
which is the desired inequality.
\end{proof}

Establishing Lemma \ref{lemma:diff} allows us to apply the McDiarmid's Inequality \cite{mcdiarmid1989method}, a concentration inequality for functions with bounded difference.

\begin{lemma}[McDiarmid's Inequality]\label{lemma:mc}
    If $\{\xb_i \in \Xc_i \}_{i=1}^n$ are independent random variables (not necessarily identical), and a function $f: \Xc_1 \times \Xc_2 \ldots \Xc_n \rightarrow \R$ satisfies coordinate wise bounded condition
    \begin{equation*}
        \underset{\xb'_j \in \Xc_j}{sup} |f(\xb_1,\xb_2,\ldots, \xb_j,\ldots, \xb_n) - f(\xb_1,\xb_2,\ldots, \xb'_j,\ldots, \xb_n)| < c_j,
    \end{equation*}
    for $1 \leq j \leq n$,
    then for any $\epsilon \geq 0$,
    \begin{equation}\label{eq:mc}
        P(|f(\xb_1,\ldots,\xb_n) - \E[f]| > \epsilon) \leq 2e^{-2\epsilon^2/\sum c_j^2}.
    \end{equation}
\end{lemma}

To apply the McDiarmid's Inequality, we can view the mutual information estimator with $n$ old observations and $m$ new observations, denoted with $\hat{I}_{m+n}$, as a function of the new $m$ observations $\{\xb_i \in \Xc \}_{i=1}^m$. Moreover, we have already bounded the maximum differences of the mutual information estimator through Lemma \ref{lemma:diff}, meaning
\begin{equation*}
    \underset{\xb'_j \in \Xc}{sup} |\hat{I}_{m+n}(\xb_1,\xb_2,\ldots, \xb_j,\ldots, \xb_m) - \hat{I}_{m+n}(\xb_1,\xb_2,\ldots, \xb'_j,\ldots, \xb_m)| < \frac{2\log (m+n)}{m+n}.
\end{equation*}
Then, plug the upper bound into Eq. \eqref{eq:mc}, we have
\begin{equation*}
    P(|\hat{I}_{m+n} - \E[\hat{I}_{m+n}]| > \epsilon) \leq 2e^{-2\epsilon^2/\sum (\frac{2\log (m+n)}{m+n})^2} = 2e^{-(m+n)^2\epsilon^2/2m\log^2(m+n)}.
\end{equation*}
 By setting the RHS of the above equation to $\rho$, we can get the following statement with probability at least $1-\rho$,
\begin{equation}
    |\hat{I}_{m+n} - \E[\hat{I}_{m+n}]| \leq \frac{\sqrt{2m \log 2/\rho} \log (m+n)}{m+n}.
\end{equation}

Finally, combining the derivation in the two parts, when $n \ll |\Xc|,|\Yc|$, we have the following with probability at least $1-\rho$
\begin{equation*}
    \begin{aligned}
        MIS &= \hat{I}_{m+n} - \hat{I}_n\\
        &= \hat{I}_{m+n} - \E[\hat{I}_n]\\
        &\in \E[\hat{I}_{m+n}] \pm \frac{\sqrt{2m \log 2/\rho} \log (m+n)}{m+n}- \E[\hat{I}_n]\\
        &= \left(\log (m+n)-\log n \right) \pm \frac{\sqrt{2m \log 2/\rho} \log(m + n)}{m + n}.
    \end{aligned}
\end{equation*}
The second equation follows the typical sample assumption in Assumption \ref{assump:MIS}. The proof of Theorem \ref{the:mis} is now complete.

\section*{Pollution Map Dataset}

The dynamic pollution map is modeled as $u(\xb,t)$, a function of spatial location $\xb = (x_1,x_2) \in [0,1]^2$. The governing partial differential equation (PDE) for the pollution map is
\begin{equation}\label{eq:pde}
    \frac{\partial u}{\partial t} = -\vb \cdot \nabla u + \nabla(\Db\nabla u) - \zeta u + S(\xb),
\end{equation}
where $\vb = [1, 0]$ is the advection velocity, representing wind that transports pollution horizontally to the right. The matrix $\Db = \text{diag}(0.01, 2)$ is the diagonal diffusion matrix, indicating that pollution diffuses much more rapidly in the $x_2$ direction than in the $x_1$ direction. The parameter $\zeta = 2$ represents the exponential decay factor, modeling the natural decay of pollution levels over time. The term $S(\xb)$ models the spatially dependent but temporally constant pollution source at location $\xb$. Additionally, a base level of random pollution with mean $2$ and standard deviation $0.25$ is added to the pollution field. The evolution of the pollution map is computed in the Fourier domain by applying a discretized Fourier transformation to the PDE in Eq.~\eqref{eq:pde}.

In the last simulation experiment with the pollution map, we use the same PDE with modified parameters. Specifically, the pollution sources $S(\xb)$ is removed, and the decay parameter $\zeta$ is reduced to $0.1$ in the second phase.

\section*{Surprise Reactive Sampling Strategy Pseudo Code}

In this section, we present the pseudocode for the SR sampling strategy in \cite{ahmed2024toward} for reproducibility purpose in Algorithm \ref{algo:sr}.

\begin{algorithm}[h]
\caption{Surprise Reactive (SR) Sampling Strategy}
\label{algo:sr}
\begin{algorithmic}[1]
\Require Observation set $\Xb: \{\xb_i \in \Xc \}_{i=1}^n$; Total sampling budget $k$; Exploitation limit $t$; A surprise measure $S(\cdot)$; A surprise triggering threshold $s$; Exploration mode indicator $\xi = \text{True}$; Surprising location $\xb_s = \text{None}$; Surprising location set $\Xb_s = \text{None}$; Neighborhood radius $\epsilon$. 
\While {$i < k$ ($i$ starts from $0$)}
    \If{$\xi$}
        \State Sample $\xb^*$ as
        \[
            \xb^* = \underset{\xb}{\argmax} \: \underset{\xb_i \in \Xb}{\min} \: \|\xb - \xb_i\|_2.
        \]
        \State $i = i + 1$
        \State Compute $S(\xb^*)$
        \If{$S(\xb^*) \leq s$}
        \State $\Xb = [\Xb,\xb^*]$
        \Else
        \State $\xi = False$, $\xb_s = \xb^*$, $\Xb_s = [\xb^*]$
        \EndIf
    \Else
        \While{$j \leq t$ ($j$ starts from $0$)}
        \State Sample $\xb^*$ randomly in the $\epsilon$ ball centered at $\xb_s$.
        \State $j = j+1$, $i = i+1$
        \State Compute $S(\xb^*)$
        \If{$S(\xb^*) \leq s$}
        \State $\Xb = [\Xb, \xb^*]$, $\xi=\text{True}$, $\Xb_s = \text{None}$
        \State {\bf Break While}
        \algstore{myalg}
        \end{algorithmic}
        \end{algorithm}
        
        \begin{algorithm}                     
        \begin{algorithmic} [1]
        \algrestore{myalg}
        \Else
        
        \State $\Xb_s = [\Xb_s, \xb^*]$
        
        \EndIf
        
        \If{$i \geq k$}
            \State {\bf Break While}
        \EndIf
        \EndWhile
        \If{$\Xb_s$ is not None}
            \State $\Xb = [\Xb, \Xb_s]$, $\xi = \text{True}$
        \EndIf
    \EndIf
\EndWhile
\end{algorithmic}
\end{algorithm}

\section*{Non-smoothed Error Progression with Action Overlays}
Here we present the non-smoothed estimation error progression figure with action overlays.
\begin{figure}[h]
    \centering
    \includegraphics[width=1\linewidth]{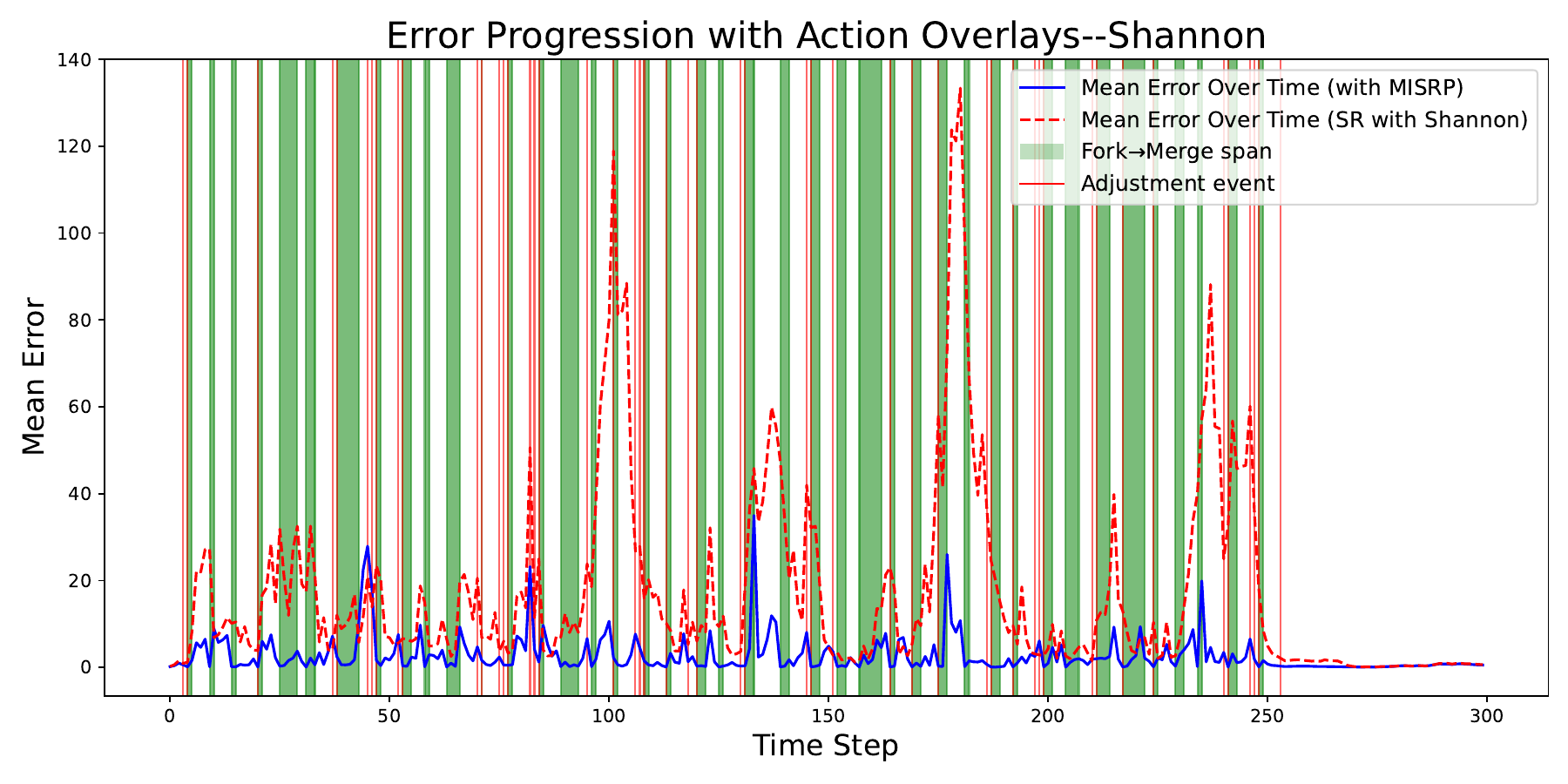}
    \caption{A non-smoothed visualization of estimation error progression with MISRP action overlays.}
    \label{fig:action_non_smooth}
\end{figure}


\clearpage 

\bibliography{science_template} 

@article{nikolaev2016autonomy,
  title={Autonomy in materials research: A case study in carbon nanotube growth},
  author={Nikolaev, Pavel and Hooper, Daylond and Webber, Frederick and Rao, Rahul and Decker, Kevin and Krein, Michael and Poleski, Jason and Barto, Rick and Maruyama, Benji},
  journal={NPJ Computational Materials},
  volume={2},
  pages={16031},
  year={2016},
  publisher={Nature Publishing Group}
}

@article{chang2020efficient,
  title={Efficient closed-loop maximization of carbon nanotube growth rate using {B}ayesian optimization},
  author={Chang, Jorge and Nikolaev, Pavel and Carpena-N{\'u}{\~n}ez, Jennifer and Rao, Rahul and Decker, Kevin and Islam, Ahmad E and Kim, Jiseob and Pitt, Mark A and Myung, Jay I and Maruyama, Benji},
  journal={Scientific Reports},
  volume={10},
  pages={9040},
  year={2020},
  publisher={Nature Publishing Group UK London}
}

@article{macleod2020self,
  title={Self-driving laboratory for accelerated discovery of thin-film materials},
  author={MacLeod, Benjamin P and Parlane, Fraser GL and Morrissey, Thomas D and H{\"a}se, Florian and Roch, Lo{\"\i}c M and Dettelbach, Kevan E and Moreira, Raphaell and Yunker, Lars PE and Rooney, Michael B and Deeth, Joseph R},
  journal={Science Advances},
  volume={6},
  number={20},
  pages={eaaz8867},
  year={2020},
  publisher={American Association for the Advancement of Science}
}

@article{yurtsever2020survey,
  title={A survey of autonomous driving: Common practices and emerging technologies},
  author={Yurtsever, Ekim and Lambert, Jacob and Carballo, Alexander and Takeda, Kazuya},
  journal={IEEE Access},
  volume={8},
  number={},
  pages={58443--58469},
  year={2020},
  publisher={IEEE}
}

@inproceedings{levinson2011towards,
  title={Towards fully autonomous driving: Systems and algorithms},
  author={Levinson, Jesse and Askeland, Jake and Becker, Jan and Dolson, Jennifer and Held, David and Kammel, Soeren and Kolter, J Zico and Langer, Dirk and Pink, Oliver and Pratt, Vaughan and Sokolsky, Michael and Stanek, Ganymed and Stavens, David and Teichman, Alex and Werling, Moritz and Thrun, Sebastian},
  booktitle={Proceedings of the 2011 IEEE Intelligent Vehicles Symposium},
  pages={June, 2011},
  year={2011},
  month={June},
  address = {Baden-Baden, Germany}
}

@article{leng2023towards,
  title={Towards resilience in Industry 5.0: A decentralized autonomous manufacturing paradigm},
  author={Leng, Jiewu and Zhong, Yuanwei and Lin, Zisheng and Xu, Kailin and Mourtzis, Dimitris and Zhou, Xueliang and Zheng, Pai and Liu, Qiang and Zhao, J Leon and Shen, Weiming},
  journal={Journal of Manufacturing Systems},
  volume={71},
  pages={95--114},
  year={2023},
  publisher={Elsevier}
}

@article{park2012autonomous,
  title={An autonomous manufacturing system based on swarm of cognitive agents},
  author={Park, Hong-Seok and Tran, Ngoc-Hien},
  journal={Journal of Manufacturing Systems},
  volume={31},
  number={3},
  pages={337--348},
  year={2012},
  publisher={Elsevier}
}

@article{reis2021high,
  title={High-tech defense industries: Developing autonomous intelligent systems},
  author={Reis, Jo{\~a}o and Cohen, Yuval and Mel{\~a}o, Nuno and Costa, Joana and Jorge, Diana},
  journal={Applied Sciences},
  volume={11},
  number={11},
  pages={4920},
  year={2021},
  publisher={MDPI}
}

@article{dai2024autonomous,
  title={Autonomous mobile robots for exploratory synthetic chemistry},
  author={Dai, Tianwei and Vijayakrishnan, Sriram and Szczypi{\'n}ski, Filip T and Ayme, Jean-Fran{\c{c}}ois and Simaei, Ehsan and Fellowes, Thomas and Clowes, Rob and Kotopanov, Lyubomir and Shields, Caitlin E and Zhou, Zhengxue and Ward, John W and Cooper, Andrew I},
  journal={Nature},
  volume={635},
  pages={890--897},
  year={2024},
  publisher={Nature Publishing Group UK London}
}

@article{burger2020mobile,
  title={A mobile robotic chemist},
  author={Burger, Benjamin and Maffettone, Phillip M and Gusev, Vladimir V and Aitchison, Catherine M and Bai, Yang and Wang, Xiaoyan and Li, Xiaobo and Alston, Ben M and Li, Buyi and Clowes, Rob and Rankin, Nicola and Harris, Brandon and Sprick, Reiner S and Cooper, Andrew I},
  journal={Nature},
  volume={583},
  number={},
  pages={237--241},
  year={2020},
  publisher={Nature Publishing Group}
}

@article{szymanski2023autonomous,
  title={An autonomous laboratory for the accelerated synthesis of novel materials},
  author={Szymanski, Nathan J and Rendy, Bernardus and Fei, Yuxing and Kumar, Rishi E and He, Tanjin and Milsted, David and McDermott, Matthew J and Gallant, Max and Cubuk, Ekin Dogus and Merchant, Amil and Kim, Haegyeom and Jain, Anubhav and Bartel, Christopher J and Persson, Kristin and Zeng, Yan and Ceder, Gerbrand},
  journal={Nature},
  volume={624},
  number={},
  pages={86--91},
  year={2023},
  publisher={Nature Publishing Group UK London}
}

@article{merchant2023scaling,
  title={Scaling deep learning for materials discovery},
  author={Merchant, Amil and Batzner, Simon and Schoenholz, Samuel S and Aykol, Muratahan and Cheon, Gowoon and Cubuk, Ekin Dogus},
  journal={Nature},
  volume={624},
  number={},
  pages={80--85},
  year={2023},
  publisher={Nature Publishing Group UK London}
}

@inproceedings{bogdoll2022anomaly,
  title={Anomaly detection in autonomous driving: A survey},
  author={Bogdoll, Daniel and Nitsche, Maximilian and Z{\"o}llner, J Marius},
  booktitle={Proceedings of the IEEE/CVF Conference on Computer Vision and Pattern Recognition},
  pages={June, 2022},
  year={2022},
  month={June},
  address={New Orleans, USA}
}

@article{lian2022anomaly,
  title={Anomaly detection and correction of optimizing autonomous systems with inverse reinforcement learning},
  author={Lian, Bosen and Kartal, Yusuf and Lewis, Frank L and Mikulski, Dariusz G and Hudas, Gregory R and Wan, Yan and Davoudi, Ali},
  journal={IEEE Transactions on Cybernetics},
  volume={53},
  number={7},
  pages={4555--4566},
  year={2022},
  publisher={IEEE}
}

@inproceedings{zhou2016continuous,
  title={Continuous anomaly detection in satellite image time series based on z-scores of season-trend model residuals},
  author={Zhou, Zeng-Guang and Tang, Ping},
  booktitle={Proceedings of the 2016 IEEE International Geoscience and Remote Sensing Symposium},
  pages={July, 2016},
  year={2016},
  month={July},
  address={Beijing, China}
}

@article{kamenik2023null,
  title={Null hypothesis test for anomaly detection},
  author={Kamenik, Jernej F and Szewc, Manuel},
  journal={Physics Letters B},
  volume={840},
  pages={137836},
  year={2023},
  publisher={Elsevier}
}

@article{schlegl2019f,
  title={F-Ano{GAN}: Fast unsupervised anomaly detection with generative adversarial networks},
  author={Schlegl, Thomas and Seeb{\"o}ck, Philipp and Waldstein, Sebastian M and Langs, Georg and Schmidt-Erfurth, Ursula},
  journal={Medical Image Analysis},
  volume={54},
  pages={30--44},
  year={2019},
  publisher={Elsevier}
}

@article{cohen2015active,
  title={Active hypothesis testing for anomaly detection},
  author={Cohen, Kobi and Zhao, Qing},
  journal={IEEE Transactions on Information Theory},
  volume={61},
  number={3},
  pages={1432--1450},
  year={2015},
  publisher={IEEE}
}

@article{weller2014survey,
  title={A survey of distance and similarity measures used within network intrusion anomaly detection},
  author={Weller-Fahy, David J and Borghetti, Brett J and Sodemann, Angela A},
  journal={IEEE Communications Surveys \& Tutorials},
  volume={17},
  number={1},
  pages={70--91},
  year={2014},
  publisher={IEEE}
}

@article{montechiesi2016artificial,
  title={Artificial immune system via {E}uclidean Distance Minimization for anomaly detection in bearings},
  author={Montechiesi, Luca and Cocconcelli, Marco and Rubini, Riccardo},
  journal={Mechanical Systems and Signal Processing},
  volume={76},
  pages={380--393},
  year={2016},
  publisher={Elsevier}
}

@article{wang2013online,
  title={Online anomaly detection for hard disk drives based on {M}ahalanobis distance},
  author={Wang, Yu and Miao, Qiang and Ma, Eden WM and Tsui, Kwok-Leung and Pecht, Michael G},
  journal={IEEE Transactions on Reliability},
  volume={62},
  number={1},
  pages={136--145},
  year={2013},
  publisher={IEEE}
}

@inproceedings{hou2020mahalanobis,
  title={Mahalanobis distance based adversarial network for anomaly detection},
  author={Hou, Yubo and Chen, Zhenghua and Wu, Min and Foo, Chuan-Sheng and Li, Xiaoli and Shubair, Raed M},
  booktitle={Proceedings of the 2020 IEEE International Conference on Acoustics, Speech and Signal Processing},
  pages={May, 2020},
  year={2020},
  month={May},
  address={Virtual}
}

@inproceedings{baldi2002computational,
  title={A computational theory of surprise},
  author={Baldi, Pierre},
  booktitle={Information, Coding and Mathematics: Proceedings of Workshop Honoring Prof. Bob Mceliece on his 60th Birthday},
  pages={1--25},
  year={2002}
}

@article{modirshanechi2022taxonomy,
  title={A taxonomy of surprise definitions},
  author={Modirshanechi, Alireza and Brea, Johanni and Gerstner, Wulfram},
  journal={Journal of Mathematical Psychology},
  volume={110},
  pages={102712},
  year={2022},
  publisher={Elsevier}
}

@article{prat2021human,
  title={Human inference in changing environments with temporal structure},
  author={Prat-Carrabin, Arthur and Wilson, Robert C and Cohen, Jonathan D and Azeredo da Silveira, Rava},
  journal={Psychological Review},
  volume={128},
  number={5},
  pages={879–912},
  year={2021},
  publisher={American Psychological Association}
}

@article{kolossa2015computational,
  title={A computational analysis of the neural bases of {B}ayesian inference},
  author={Kolossa, Antonio and Kopp, Bruno and Fingscheidt, Tim},
  journal={Neuroimage},
  volume={106},
  pages={222--237},
  year={2015},
  publisher={Elsevier}
}

@article{barto2013novelty,
  title={Novelty or surprise?},
  author={Barto, Andrew and Mirolli, Marco and Baldassarre, Gianluca},
  journal={Frontiers in Psychology},
  volume={4},
  pages={907},
  year={2013},
  publisher={Frontiers Media SA}
}

@article{itti2009bayesian,
  title={Bayesian surprise attracts human attention},
  author={Itti, Laurent and Baldi, Pierre},
  journal={Vision Research},
  volume={49},
  number={10},
  pages={1295--1306},
  year={2009},
  publisher={Elsevier}
}

@article{faraji2018balancing,
  title={Balancing new against old information: The role of puzzlement surprise in learning},
  author={Faraji, Mohammadjavad and Preuschoff, Kerstin and Gerstner, Wulfram},
  journal={Neural Computation},
  volume={30},
  number={1},
  pages={34--83},
  year={2018},
  publisher={MIT Press}
}

@article{ahmed2024toward,
  title={Toward futuristic autonomous experimentation—A surprise-reacting sequential experiment policy},
  author={Ahmed, Imtiaz and Bukkapatnam, Satish TS and Botcha, Bhaskar and Ding, Yu},
  journal={IEEE Transactions on Automation Science and Engineering},
  year={2025},
  volume={22},
  pages={7912--7926},
  publisher={IEEE}
}

@article{liakoni2021learning,
  title={Learning in volatile environments with the {B}ayes factor surprise},
  author={Liakoni, Vasiliki and Modirshanechi, Alireza and Gerstner, Wulfram and Brea, Johanni},
  journal={Neural Computation},
  volume={33},
  number={2},
  pages={269--340},
  year={2021},
  publisher={MIT Press}
}

@article{zamiri2022bayesian,
  title={A {B}ayesian surprise approach in designing cognitive radar for autonomous driving},
  author={Zamiri-Jafarian, Yeganeh and Plataniotis, Konstantinos N},
  journal={Entropy},
  volume={24},
  number={5},
  pages={672},
  year={2022},
  publisher={MDPI}
}

@article{dinparastdjadid2023measuring,
  title={Measuring surprise in the wild},
  author={Dinparastdjadid, Azadeh and Supeene, Isaac and Engstrom, Johan},
  journal={arXiv preprint arXiv:2305.07733},
  year={2023}
}

@inproceedings{ccatal2020anomaly,
  title={Anomaly detection for autonomous guided vehicles using {B}ayesian surprise},
  author={{\c{C}}atal, Ozan and Leroux, Sam and De Boom, Cedric and Verbelen, Tim and Dhoedt, Bart},
  booktitle={Proceedings of the 2020 IEEE/RSJ International Conference on Intelligent Robots and Systems},
  pages={October, 2020},
  year={2020},
  month={October},
  address={Las Vegas, USA}
}

@article{raihan2024augmented,
  title={An augmented surprise-guided sequential learning framework for predicting the melt pool geometry},
  author={Raihan, Ahmed Shoyeb and Khosravi, Hamed and Bhuiyan, Tanveer Hossain and Ahmed, Imtiaz},
  journal={Journal of Manufacturing Systems},
  volume={75},
  pages={56--77},
  year={2024},
  publisher={Elsevier}
}

@inproceedings{jin2022autonomous,
  title={Autonomous experimentation systems and benefit of surprise-based {B}ayesian optimization},
  author={Jin, Shilan and Deneault, James R and Maruyama, Benji and Ding, Yu},
  booktitle={Proceedings of the 2022 International Symposium on Flexible Automation},
  pages={July, 2022},
  year={2022},
  month={July},
  address={Yokohama, Japan}
}

@inproceedings{franccois2006permutation,
  title={The permutation test for feature selection by mutual information},
  author={Fran{\c{c}}ois, Damien and Wertz, Vincent and Verleysen, Michel},
  booktitle={Proceedings of the 14th European Symposium on Artificial Neural Networks},
  pages={April, 2006},
  year={2006},
  month={April},
  address={Bruges, Belgium}
}

@article{doquire2013mutual,
  title={Mutual information-based feature selection for multilabel classification},
  author={Doquire, Gauthier and Verleysen, Michel},
  journal={Neurocomputing},
  volume={122},
  pages={148--155},
  year={2013},
  publisher={Elsevier}
}

@article{moreno2012unifying,
  title={A unifying view on dataset shift in classification},
  author={Moreno-Torres, Jose G and Raeder, Troy and Alaiz-Rodr{\'\i}guez, Roc{\'\i}o and Chawla, Nitesh V and Herrera, Francisco},
  journal={Pattern Recognition},
  volume={45},
  number={1},
  pages={521--530},
  year={2012},
  publisher={Elsevier}
}

@article{vzliobaite2016overview,
  title={An overview of concept drift applications},
  author={{\v{Z}}liobait{\.e}, Indr{\.e} and Pechenizkiy, Mykola and Gama, Joao},
  journal={Big Data Analysis: New Algorithms for a New Society},
  volume={16},
  pages={91--114},
  year={2016},
  publisher={Springer}
}

@article{sugiyama2007covariate,
  title={Covariate shift adaptation by importance weighted cross validation},
  author={Sugiyama, Masashi and Krauledat, Matthias and M{\"u}ller, Klaus-Robert},
  journal={Journal of Machine Learning Research},
  volume={8},
  number={5},
  pages={985--1005},
  year={2007},
  publisher={Microtome Publishing}
}

@article{bickel2009discriminative,
  title={Discriminative learning under covariate shift},
  author={Bickel, Steffen and Br{\"u}ckner, Michael and Scheffer, Tobias},
  journal={Journal of Machine Learning Research},
  volume={10},
  number={9},
  pages={2137-2155},
  year={2009},
  publisher={Microtome Publishing}
}

@article{zhang2023concept,
  title={Concept drift monitoring and diagnostics of supervised learning models via score vectors},
  author={Zhang, Kungang and Bui, Anh T and Apley, Daniel W},
  journal={Technometrics},
  volume={65},
  number={2},
  pages={137--149},
  year={2023},
  publisher={Taylor \& Francis}
}

@article{joseph2016space,
  title={Space-filling designs for computer experiments: A review},
  author={Joseph, V Roshan},
  journal={Quality Engineering},
  volume={28},
  number={1},
  pages={28--35},
  year={2016},
  publisher={Taylor \& Francis}
}

@article{shannon1948mathematical,
  title={A mathematical theory of communication},
  author={Shannon, Claude E},
  journal={The Bell System Technical Journal},
  volume={27},
  number={3},
  pages={379--423},
  year={1948},
  publisher={Nokia Bell Labs}
}

@inproceedings{aytekin2018clustering,
  title={Clustering and unsupervised anomaly detection with L-2 normalized deep auto-encoder representations},
  author={Aytekin, Caglar and Ni, Xingyang and Cricri, Francesco and Aksu, Emre},
  booktitle={Proceedings of the 2018 International Joint Conference on Neural Networks},
  pages={October, 2018},
  year={2018},
  month={October},
  address={Rio de Janeiro, Brazil}
}

@InProceedings{nguyen2019anomaly,
  title = 	 {Anomaly detection with multiple-hypotheses predictions},
  author =       {Nguyen, Duc Tam and Lou, Zhongyu and Klar, Michael and Brox, Thomas},
  booktitle = 	 {Proceedings of the 36th International Conference on Machine Learning},
  pages = 	 {June, 2019},
  year = 	 {2019},
  month = 	 {June},
  address={Long Beach, USA},
}

@article{rousseeuw1993alternatives,
  title={Alternatives to the median absolute deviation},
  author={Rousseeuw, Peter J and Croux, Christophe},
  journal={Journal of the American Statistical Association},
  volume={88},
  number={424},
  pages={1273--1283},
  year={1993},
  publisher={Taylor \& Francis}
}

@inproceedings{chai2009generalization,
  title={Generalization errors and learning curves for regression with multi-task {G}aussian processes},
  author={Chai, Kian},
  booktitle={Proceedings of the 23rd Advances in Neural Information Processing Systems},
  pages={December, 2009},
  year={2009},
  month={December},
  address={Vancouver, Canada}
}

@inproceedings{bondu2010exploration,
  title={Exploration vs. exploitation in active learning: A {B}ayesian approach},
  author={Bondu, Alexis and Lemaire, Vincent and Boull{\'e}, Marc},
  booktitle={Proceedings of the 2010 International Joint Conference on Neural Networks},
  pages={July, 2010},
  year={2010},
  month={July},
  address={Barcelona, Spain}
}

@article{paninski2003estimation,
  title={Estimation of entropy and mutual information},
  author={Paninski, Liam},
  journal={Neural Computation},
  volume={15},
  number={6},
  pages={1191--1253},
  year={2003},
  publisher={MIT Press}
}

@article{cebron2009active,
  title={Active learning for object classification: From exploration to exploitation},
  author={Cebron, Nicolas and Berthold, Michael R},
  journal={Data Mining and Knowledge Discovery},
  volume={18},
  pages={283--299},
  year={2009},
  publisher={Springer}
}

@article{islam2025dynamic,
  title={Dynamic exploration--exploitation trade-off in active learning regression with {B}ayesian hierarchical modeling},
  author={Islam, Upala Junaida and Paynabar, Kamran and Runger, George and Iquebal, Ashif Sikandar},
  journal={IISE Transactions},
  volume={57},
  number={4},
  pages={393--407},
  year={2025},
  publisher={Taylor \& Francis}
}

@article{strong1998entropy,
  title={Entropy and information in neural spike trains},
  author={Strong, Steven P and Koberle, Roland and Van Steveninck, Rob R De Ruyter and Bialek, William},
  journal={Physical Review Letters},
  volume={80},
  pages={197},
  year={1998},
  publisher={APS}
}

@article{mcdiarmid1989method,
  title={On the method of bounded differences},
  author={McDiarmid, Colin},
  journal={Surveys in Combinatorics},
  volume={141},
  number={1},
  pages={148--188},
  year={1989},
  publisher={Norwich}
}

@book{cover1999elements,
  title={Elements of {I}nformation {T}heory},
  author={Cover, Thomas M},
  year={1999},
  publisher={John Wiley \& Sons}
}

@article{verdu1994generalizing,
  title={Generalizing the {F}ano inequality},
  author={Verd{\'u}, Sergio},
  journal={IEEE Transactions on Information Theory},
  volume={40},
  number={4},
  pages={1247--1251},
  year={1994},
  publisher={IEEE}
}

@article{dawid1984present,
  title={Present position and potential developments: {S}ome personal views statistical theory the prequential approach},
  author={Dawid, A Philip},
  journal={Journal of the Royal Statistical Society: Series A (General)},
  volume={147},
  number={2},
  pages={278--290},
  year={1984},
  publisher={Wiley Online Library}
}

@inproceedings{lee2020repad,
  title={RePAD: {R}eal-time proactive anomaly detection for time series},
  author={Lee, Ming-Chang and Lin, Jia-Chun and Gran, Ernst Gunnar},
  booktitle={The 2020 International Conference on Advanced Information Networking and Applications},
  pages={1291--1302},
  month={Apr},
  year={2020},
  address={Caserta, Italy},
  organization={Springer}
}

@inproceedings{ayed2020anomaly,
  title={Anomaly detection at scale: {T}he case for deep distributional time series models},
  author={Ayed, Fadhel and Stella, Lorenzo and Januschowski, Tim and Gasthaus, Jan},
  booktitle={The 2020 International Conference on Service-Oriented Computing},
  pages={97--109},
  month={Nov},
  year={2020},
  address={Dubai, UAE},
  organization={Springer}
}

@article{bayram2022concept,
  title={From concept drift to model degradation: {A}n overview on performance-aware drift detectors},
  author={Bayram, Firas and Ahmed, Bestoun S and Kassler, Andreas},
  journal={Knowledge-Based Systems},
  volume={245},
  pages={108632},
  year={2022},
  publisher={Elsevier}
}

@article{friston2015active,
  title={Active inference and epistemic value},
  author={Friston, Karl and Rigoli, Francesco and Ognibene, Dimitri and Mathys, Christoph and Fitzgerald, Thomas and Pezzulo, Giovanni},
  journal={Cognitive Neuroscience},
  volume={6},
  number={4},
  pages={187--214},
  year={2015},
  publisher={Taylor \& Francis}
}
\bibliographystyle{ieeetr}

\end{document}